\documentclass[twoside,11pt, final]{article}

\usepackage[abbrvbib, preprint]{jmlr2e}

\usepackage{custom}

\usepackage{lastpage}
\jmlrheading{21}{2020}{1-\pageref{LastPage}}{2/20}{9/20}{20-163}{Shuxiao Chen, Edgar Dobriban, and Jane H. Lee}
\ShortHeadings{Theory for Data Augmentation}{Chen, Dobriban and Lee}

\renewcommand{\hat}{\widehat}
\makeatletter
\newcommand*\bigcdot{\mathpalette\bigcdot@{.5}}
\newcommand*\bigcdot@[2]{\mathbin{\vcenter{\hbox{\scalebox{#2}{$\m@th#1\bullet$}}}}}
\makeatother

\firstpageno{1}

\begin{document}

\title{A Group-Theoretic Framework for Data Augmentation}

\author{\name Shuxiao Chen \email shuxiaoc@wharton.upenn.edu \\
       \name Edgar Dobriban \email dobriban@wharton.upenn.edu \\
       \addr Department of Statistics\\
       The Wharton School of the University of Pennsylvania\\
       Philadelphia, PA, 19104-6340, USA
       \AND
       \name Jane H.\ Lee \email janehlee@sas.upenn.edu \\
       \addr Department of Mathematics, and Computer and Information Science\\
       University of Pennsylvania\\
       Philadelphia, PA 19104-6309, USA
       }

\editor{Rob Fergus}

\maketitle

\begin{abstract}
Data augmentation is a widely used trick when training deep neural networks: in addition to the original data, properly transformed data are also added to the training set. However, to the best of our knowledge, a clear mathematical framework to explain the performance benefits of data augmentation is not available. In this paper, we develop such a theoretical framework. We show data augmentation is equivalent to an averaging operation over the orbits of a certain group that keeps the data distribution approximately invariant. We prove that it leads to variance reduction. We study empirical risk minimization, and the examples of exponential families, linear regression, and certain two-layer neural networks. We also discuss how data augmentation could be used in problems with symmetry where other approaches are prevalent, such as in cryo-electron microscopy (cryo-EM).
\end{abstract}

\begin{keywords}
Data Augmentation, Deep Learning, Empirical Risk Minimization, Invariance, Variance Reduction 
\end{keywords}

\section{Introduction}

Deep learning algorithms such as convolutional neural networks (CNNs) are successful in part because they exploit natural symmetry in the data. For instance, image identity is roughly invariant to translations and rotations: so a slightly translated cat is still a cat. Such invariances are present in many datasets, including image, text and speech data. Standard architectures are invariant to some, but not all transforms. For instance, CNNs induce an approximate equivariance to translations, but not to rotations. This is an \emph{inductive bias} of CNNs, and the idea dates back at least to the neocognitron \citep{fukushima1980neocognitron}.

To make models invariant to arbitrary transforms beyond the ones built into the architecture, \emph{data augmentation} is commonly used. Roughly speaking, the model is trained not just with the original data, but also with transformed data. Data augmentation is a crucial component of modern deep learning pipelines, and it is typically needed to achieve state of the art performance. It has been used, e.g., in AlexNet \citep{krizhevsky2012imagenet}, and other pioneering works \citep{cirecsan2010deep}. The state-of-the-art results on aggregator websites such as \url{https://paperswithcode.com/sota} often crucially rely on effective new ways of data augmentation. See Figure \ref{fig:exp} for a small experiment (see Appendix \ref{sec:exp} for details).

\begin{figure}[tb]
	\centering
	\includegraphics[width=0.45\textwidth]{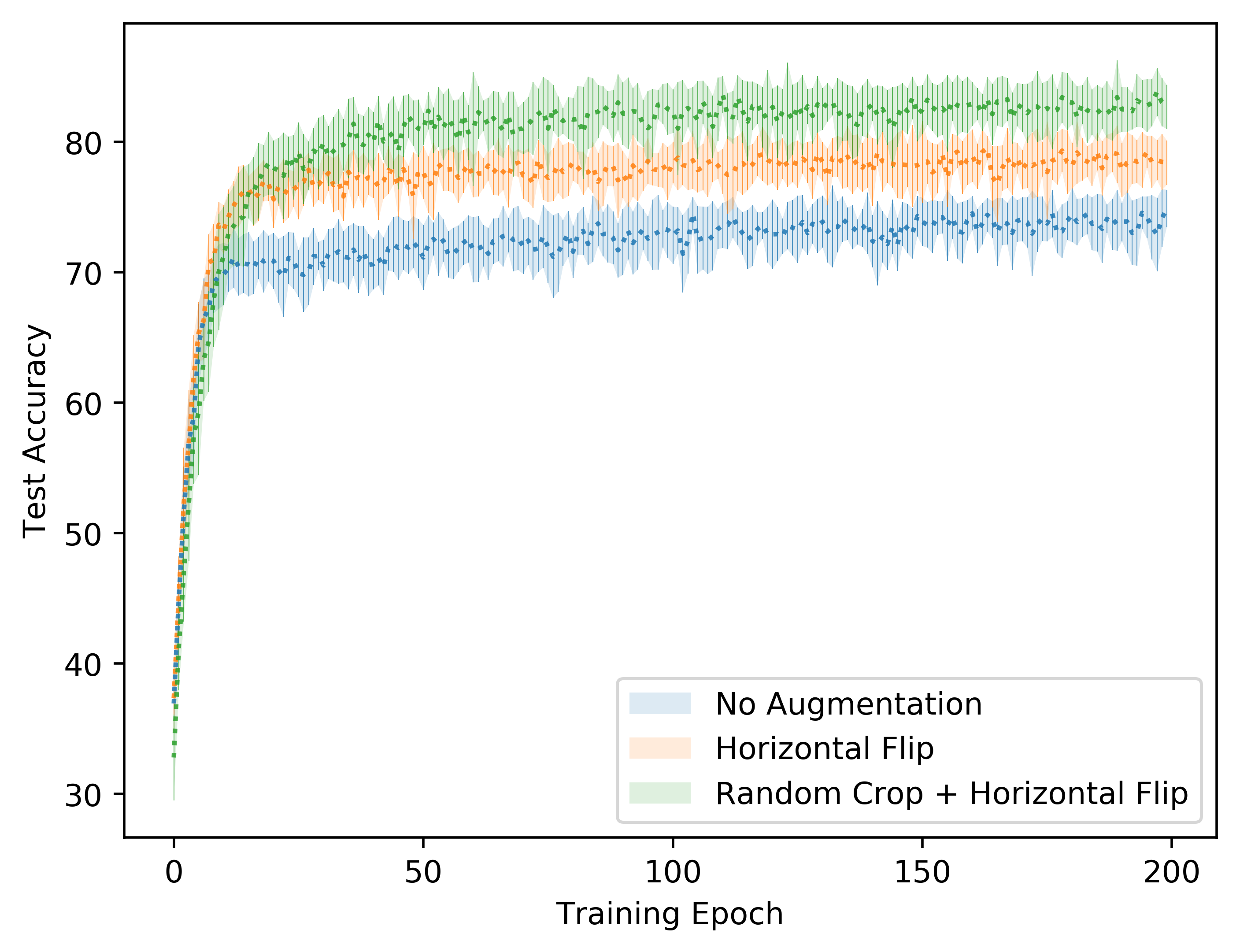}
	\includegraphics[width=0.45\textwidth]{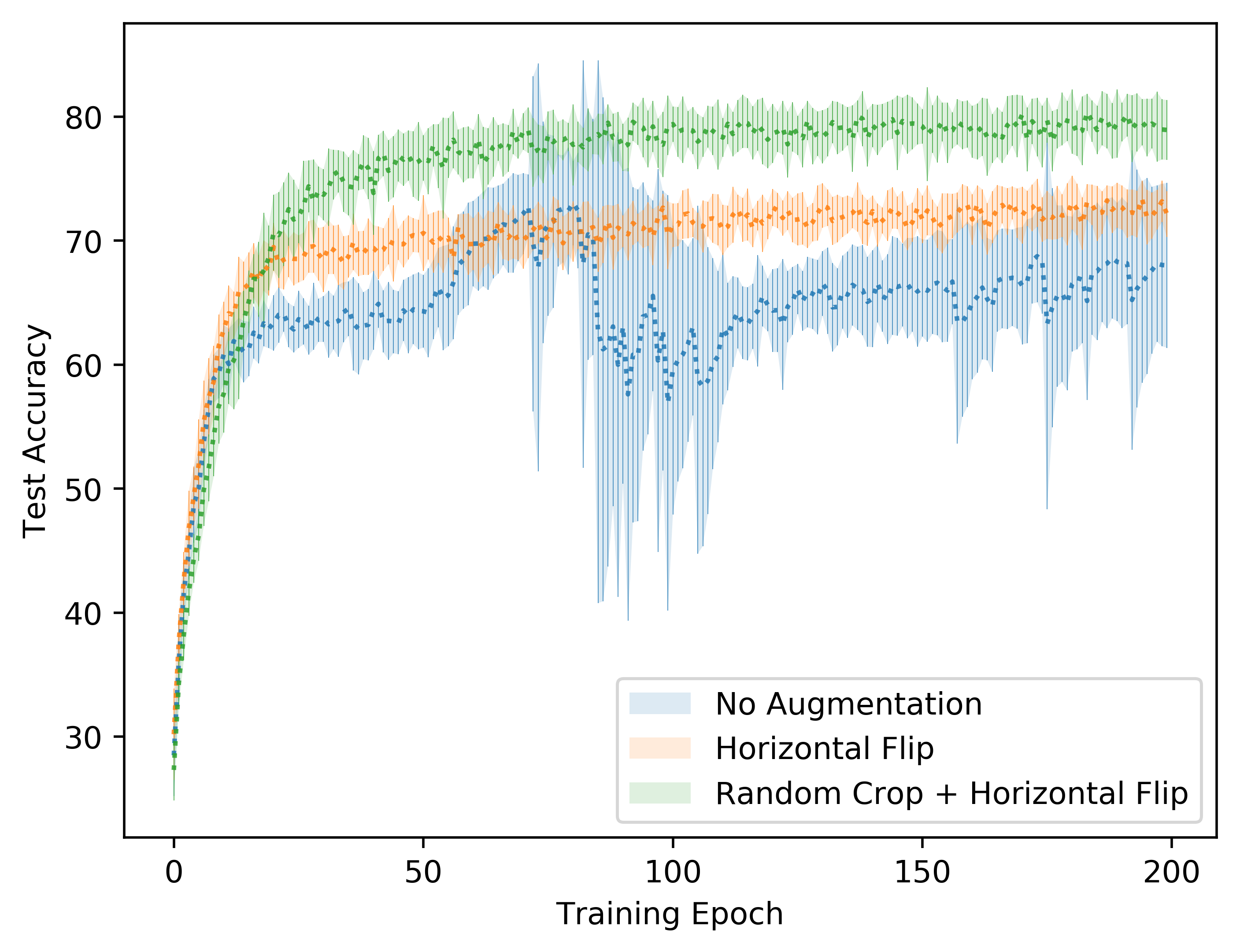}
	\caption{Benefits of data augmentation: A comparison of test accuracy across training epochs of ResNet18 \citep{resnet18} (1) without data augmentation, (2) horizontally flipping the image with $0.5$ probability, and (3) a composition of randomly cropping a $32 \times 32$ portion of the image and random horizontal flip. The experiment is repeated 15 times, with the dotted lines showing the average test accuracy and the shaded regions representing 1 standard deviation around the mean. The left graph shows results from training on the full CIFAR10 training data and the right uses half of the training data as that of the left. Data augmentation leads to increased performance, especially with limited data.}
	\label{fig:exp}
\end{figure}

Equivariant or invariant architectures (such as CNNs) are attractive for tackling invariance. However, in many cases, datasets have symmetries that are naturally described in a \emph{generative} form: we can specify generators of the group of symmetries (e.g., rotations and scalings). In contrast, computing the equivariant features requires designing new architectures. Thus, data augmentation is a universally applicable, generative, and algorithmic way to exploit invariances.
	
However, a general framework for understanding data augmentation is missing. Such a framework would enable us to reason clearly about the benefits offered by augmentation, in comparison to invariant features. Moreover, such a framework could also shed light on questions such as: How can we improve the performance of our models by simply adding transformed versions of the training data? Under what conditions can we see benefits? Developing such a framework is challenging for several reasons: first, it is unclear what mathematical approach to use, and second, it is unclear how to demonstrate that data augmentation ``helps". 

In this paper, we propose such a general framework. We use group theory as a mathematical language, and model invariances as ``approximate equality'' in distribution under a group action. We show that data augmentation can be viewed as invariant learning by averaging over the group action. We then demonstrate that data augmentation leads to sample efficient learning, both in the non-asymptotic setting (relying on results from stochastic convex optimization and Rademacher complexity), as well as in the asymptotic setting (by using asymptotic statistical theory for empirical risk minimizers/M-estimators). 

We show how to apply data augmentation beyond deep learning, to other problems in statistics and machine learning that have invariances. In addition, we explain the connections to several other important notions from statistics and machine learning, including sufficiency, invariant representations, equivariance, variance reduction in Monte Carlo methods, and regularization. 

We can summarize our main contributions as follows: 
\begin{enumerate}
	\item We study data augmentation in a group-theoretic formulation, where there is a group acting on the data, and the distribution of the data is equal (which we refer to as \emph{exact invariance}), or does not change too much (which we refer to as \emph{approximate invariance}) under the action. We explain that in empirical risk minimization (ERM), this leads to minimizing an augmented loss, which is the average of the original loss under the group action (Section \ref{sec:erm}). In the special case of maximum likelihood estimation (MLE), we discuss several variants of MLE that may potentially exploit invariance (Section \ref{subsec:variants_of_MLE}). We also propose to extend data augmentation beyond ERM, using the ``augmentation distribution'' (Section \ref{sec:bey_erm}).

	\item We provide several theoretical results to support the benefits of data augmentation. When the data is exactly invariant in distribution (exact invariance), we show that averaging over the group orbit (e.g., all rotations) reduces the variance of \emph{any function}. We can immediately conclude that estimators based on the ``augmentation distribution'' gain efficiency and augmentation reduces the mean squared error (MSE) of general estimators (Section \ref{subsec:gen_est}). 

	\item Specializing the variance reduction to the loss and the gradients of the loss, we show that the empirical risk minimizer with data augmentation enjoys favorable properties in a non-asymptotic setup. Specifically, ``loss-averaging'' implies that data augmentation reduces the Rademacher complexity of a loss class (Section \ref{subsubsec:effect_of_loss_avg}), which further suggests that the augmented model may generalize better. On the other hand, we show that ``gradient-averaging'' reduces the variance of the ERM when the loss is strongly-convex, using a recent result from stochastic convex optimization (Section \ref{subsubsec:effect_of_gradient_avg}).



	\item Moving to the asymptotic case, we characterize the \emph{precise} variance reduction obtained by data augmentation under exact invariance. We show that this depends on the covariance of the loss gradients along the group orbit (Section \ref{subsubsec:asymp_res_under_exact_invariance}). This implies that data augmentation can improve upon the Fisher information of the un-augmented maximum likelihood estimator (MLE) (Section \ref{subsubsec:implication_for_mle}). For MLE, we further study the special case when the subspace of parameters with invariance is a low-dimensional manifold. We connect this to geometry showing that the projection of the gradient onto the tangent space is always invariant; however, it does not always capture all invariance. As a result, the augmented MLE cannot always be as efficient as the ``constrained MLE'', where the invariance is achieved by constrained optimization.

	\item We work out several examples of our theory under exact invariance: exponential families (Section \ref{subsec:exponential_family}), least-squares regression (Section \ref{subsec:param_regression}) and least squares classification (Section \ref{subsec:param_classification}). As a notable example, we work out the efficiency gain for two-layer neural networks with circular shift data augmentation in the heavily underparametrized regime (where most of our results concern quadratic activations). We also provide an example of ``augmentation distribution'', in the context of linear regression with general linear group actions (Section \ref{subsec:lin_reg_aug_distr}).

	\item We extend most of the results to the approximate invariance case, where the distribution of the data is close, but not exactly equal to its transformed copy (Section \ref{sec:approx_inv}). Using optimal transport theory, we characterize an intriguing bias-variance tradeoff: while the orbit averaging operation reduces variance, a certain level of bias is created because of the non-exact-invariance. This tradeoff suggests that in practice where exact invariance does not hold, data augmentation may not always be beneficial, and its performance is governed by both the variability of the group and a specific Wasserstein distance between the data and its transformed copy.

	\item We illustrate the bias-variance tradeoff by studying the generalization error of over-parameterized two-layer networks trained by gradient descent (Section \ref{sec:overparam_2l_nets}), using recent results on Neural Tangent Kernel (see, e.g., \citealt{jacot2018neural,arora2019exact,ji2019polylogarithmic,chen2019much}).

	\item We also describe a few important problems where symmetries occur, but where other approaches---not data augmentation---are currently used (Section \ref{sec:potential_applications}): cryo-electron microscopy (cryo-EM), spherically invariant data, and random effects models. These problems may be especially promising for using data augmentation.
\end{enumerate}


\section{Some related work}\label{sec:related_work}
In this section, we discuss some related works, in addition to those that are mentioned in other places.\\

{\bf Data augmentation methodology in deep learning.} 
There is a great deal of work in developing efficient methods for data augmentation in deep learning. Here we briefly mention a few works. Data augmentation has a long history, and related ideas date back at least to \cite{baird1992document}, who built a ``pseudo-random image generator'', that ``\emph{given an image and a set of model parameters, computes a degraded image}''. This is recounted in \cite{schmidhuber2015deep}.

Conditional generative adversarial networks (cGAN) are a method for learning to generate from the class-conditional distributions in a classification problem \citep{mirza2014conditional}. This has direct applications to data augmentation. Data augmentation GANs (DAGAN) \citep{antoniou2017data} are a related approach that train a GAN to discriminate between $x,x_g$, and $x,x'$, where $x_g$ is generated as $x_g = f(x,z)$, and $x'$ is an independent copy. This learns the conditional distribution $x'|x$, where $x',x$ are sampled ``independently" from the training set. Here the training data is viewed as non-independent and they learn the dependence structure. 
	
\cite{hauberg2016dreaming} construct class-dependent distributions over diffeomorphisms for learned data augmentation. \cite{ratner2017learning} learn data augmentation policies using reinforcement learning, starting with a known set of valid transforms. \cite{tran2017bayesian} propose a Bayesian approach. RenderGAN \citep{sixt2018rendergan} combines a 3D model with GANs for image generation. \cite{devries2017dataset} propose to perform data augmentation in feature space. AutoAugment \citep{cubuk2018autoaugment} is another approach for learning augmentation policies based on reinforcement learning, which is one of the state of the art approaches. RandAugment \citep{cubuk2019randaugment} proposes a strategy to reduce the search space of augmentation policies, which also achieves state of the art performance on ImageNet classification tasks. \cite{hoffer2019augment} argues that replicating instances of samples within the same batch with different data augmentations can act as an accelerator and a regularizer when appropriately tuning the hyperparameters, increasing both the speed of training and the generalization performance.\\
	
{\bf Neural net architecture design.}
There is a parallel line of work designing invariant and equivariant neural net architectures. A key celebrated example is convolutions, dating back at least to the neocognitron \citep{fukushima1980neocognitron}, see also \cite{lecun1989backpropagation}. More recently, group equivariant Convolutional Neural Networks (G-CNNs), have been proposed, using G-convolutions to exploit symmetry \citep{cohen2016group}. That work designs concrete architectures for groups of translations, rotations by 90 degrees around any center of rotation in a square grid, and reflections.  \cite{dieleman2016exploiting} designed architectures for cyclic symmetry. 
	
\cite{worrall2017harmonic} introduces Harmonic Networks or H-Nets, which induce equivariance to patch-wise translation and 360 degree rotation. They rely on circular harmonics as invariant features. \cite{cohen2016steerable} propose steerable CNNs and \cite{cohen2018general} develop a more general approach. 
There are several works on SO(3) equivariance, see \cite{cohen2018spherical,Esteves_2018_ECCV,esteves2018cross,esteves2019equivariant}. 

\cite{gens2014deep} introduces deep symmetry networks (symnets), that form feature maps over arbitrary symmetry groups via  kernel-based interpolation to pool over symmetry spaces. See also \cite{ravanbakhsh2017equivariance,kondor2018generalization,weiler20183d,kondor2018clebsch}.
There are also many examples of data augmentation methods developed in various application areas, e.g., \cite{,jaitly2013vocal,xie2019unsupervised,park2019specaugment,ho2019population}, etc.\\

{\bf Data augmentation as a form of regularization.} 
There is also a line of work proposing to add random or adversarial noise to the data when training neural networks. The heuristic behind this approach is that the addition of noise-perturbed data should produce a classifier that is robust to random or adversarial corruption.

For example, \cite{devries2017improved} proposes to randomly mask out square regions of input images and fill the regions with pure gray color; \cite{zhong2017random} and \cite{lopes2019improving} propose to randomly select a patch within an image and replace its pixels with random values; \cite{bendory2018bispectrum} proposes to add Perlin noise \citep{perlin1985image} to medical images; \cite{zhang2017mixup} proposes to train with convex combinations of two images \emph{as well as their labels}. The experiments done by those papers show that augmenting with noise-perturbed data can lead to lower generalization error and better robustness against corruption. 

\cite{szegedy2013intriguing} and \cite{cohen2019certified} demonstrate that training with adversarial examples can lead to some form of regularization. \cite{hernandez2018data} has argued that data augmentation can sometimes even replace other regularization mechanisms such as weight decay, while \cite{hernandez2018further} has argued that this can be less sensitive to hyperparameter choices than other forms of regularization.
While this approach is also called data augmentation in the literature, it is \emph{fundamentally different} from what we consider. We study a way to exploit invariance in the data, while those works focus on smoothing effects (adding noise cannot possibly lead to exactly invariant distributions). 

More related to our approach, \cite{maaten2013learning} propose to train the model by minimizing the expected value of the loss function under the ``corrupting distribution'' and show such a strategy improves generalization. \cite{chao2017stochastic,mazaheri2019stochastic} aim to reduce the variance and enhance stability by ``functional integration'' inspired averaging over a large number of different representations of the same data. However, our quantitative approaches differ. \cite{lyle2016analysis} study the effect of invariance on generalization in neural networks. They study feature averaging, which agrees with our augmentation distribution defined later. They mention variance reduction, but we can, in fact, prove it.\\
	
{\bf Other works connected to data augmentation.}
There is a tremendous amount of other work connected to data augmentation. On the empirical end, \cite{bengio2011deep} has argued that the benefit of data augmentation goes up with depth. On the theoretical end, \cite{rajput2019does} investigate if gaussian data augmentation leads to a positive margin, with some negative results. Connecting to adversarial examples, \cite{engstrom2017rotation} shows that adversarially chosen group transforms such as rotations can already be enough to fool neural network classifiers. \cite{javadi2019hessian} shows that data augmentation can reduce a certain Hessian-based complexity of neural networks. \cite{liu2019bad} show that data augmentation can significantly improve the optimization landscape of neural networks, so that SGD avoids bad local minima and leads to much more accurate trained networks. \cite{hernandez2018deep} has shown that it also leads to better biological plausibility in some cases.

\cite{dao19b} also seek to establish a theoretical framework for understanding data augmentation, but focus on the connection between data augmentation and kernel classifiers. \cite{dao19b} study $k$-NN and kernel methods. They show how data augmentation with a kernel classifier yields approximations which look like feature averaging and variance regularization, but do not explicitly quantify how this \emph{improves} classification. One of the most related works by \cite{bloem2019probabilistic} use similar probabilistic models, but without focusing on data augmentation.

We also note that data augmentation has another meaning in Bayesian statistics, namely the introduction of auxiliary variables to help compute the posterior \citep[see e.g.,][]{tanner1987calculation}. The naming clash is unfortunate. However, since the term ``data augmentation'' is well established in deep learning, we decided to keep it in our current work.\\

{\bf Group invariance in statistical inference.}
There has been significant work on group invariance in statistical inference \citep[e.g.,][]{giri1996group}. However, the questions investigated there are different from the ones that we study. Among other contributions, \cite{helland2004statistical} argues that group invariance can form natural non‐informative priors for the parameters of the model, and introduces permissible sub‐parameters as those upon which group actions can be defined.\\

{\bf Physical invariances.}
There is a long history of studying invariance and symmetry in physics. Invariances lead to conservation laws, such as conservation of mass and momentum. In addition, invariances in Hamiltonians of physical systems lead to reductions in the number of parameters of probability distributions governing the systems. This has been among the explanations proposed of why deep learning works \citep{lin2017does}.



\section{Methodology for data augmentation}\label{sec:method}
	
\subsection{ERM}\label{sec:erm}

We start by explaining our framework in the context of empirical risk minimization (ERM). 
	
Consider observations $X_1, \ldots, X_n \in \xx$ (e.g., images along with their labels) sampled i.i.d. from a probability distribution $\P$ on the sample space $\xx$. Since data augmentation is a way of ``teaching invariance to the model'', we need to assume our data is invariant to certain transformations. Consider thus a group $G$ of transforms (e.g., the set of all rotations of images), which \emph{acts} on the sample space: there is a function $\phi: G\times \xx \to \xx, (g, x)\mapsto\phi(g, x)$, such that $\phi(e, x) = x$ for the identity element $e \in G$, and $\phi(gh, x) = \phi(g, \phi(h, x))$ for any $g, h \in G$. For notational simplicity, we write $\phi(g, x)\equiv gx$ then there is no ambiguity. To model invariance, we assume that for any group element $g \in G$ and almost any $X\sim \P$, we have an ``approximate equality'' in distribution:
\beq
\label{eq:inv}
X \approx_d gX.
\eeq
For ease of exposition, we start with \emph{exact invariance} $X =_d gX$ in Section \ref{sec:exact_inv}, and we handle \emph{approximate invariance} in Section \ref{sec:approx_inv}.

For supervised learning applications, $X_i=(Z_i,Y_i)$ contains both the features $Z_i$ and the outcome $Y_i$. The assumption \eqref{eq:inv} means that the probability of an image being a bird is (either exactly or approximately) the same as the probability for a rotated image.

As an aside, the invariance in distribution \eqref{eq:inv} arises naturally in many familiar statistical settings, including permutation tests, time series analysis. See Section \ref{subsec:otherconn}.
	
In the current context, data augmentation corresponds to ``adding all datapoints $gX_i$, $g\in G$, $i=1,\ldots, n$" to the dataset. When the group is finite, this can effectively be implemented by enlarging the data size. However, many important groups are infinite, and to understand data augmentation in that setting, it is most clear if we argue from first principles.

\begin{algorithm}[bt]
\small
\SetKwInOut{Input}{Input}
\SetKwInOut{Output}{Output}
\Input{Data $X_i$, $i=1,\ldots,n$; Method to compute gradients $\nabla L(\theta,X)$ of the loss; Method to sample augmentations $g\in G$, $g\sim \mbQ$; Learning rates $\eta_t$; Batch sizes $|S_t|$; Initial parameters $\theta_0$; Stopping criterion.}

\Output{Final parameters.}
Set $t=0$

While stopping criterion is not met 

\quad Sample random minibatch $S_t \subset \{1,\ldots,n\}$

\quad Sample random augmentation $g_{i, t}\sim \mbQ$ for each batch element

\quad Update parameters

$$ \theta_{t+1} \gets \theta_{t} - \frac{\eta_t}{|S_t|} \sum_{i\in S_t} \nabla L(\theta, g_{i, t} X_i)$$

\quad $t\gets t+1$

\Return $\theta$

\caption{Augmented SGD}
\label{alg:augsgd}
\end{algorithm}

In practice, data augmentation is performed via the following approach. To start, we consider a loss function $L(\theta, X)$, and attempt to minimize the empirical risk
\begin{equation}
    \label{eq:empirical_risk}
    R_n(\theta) := \frac{1}{n} \sum_{i=1}^n L(\theta, X_i). 
\end{equation}
We iterative over time $t = 1, 2, \hdots$ using stochastic gradient descent (SGD) or variants (see Algorithm \ref{alg:augsgd}). At each step $t$, a minibatch of $X_i$'s (say with indices $S_t$) is chosen. In data augmentation, a random transform $g_{i, t}\in G$ is sampled and applied to each data point in the minibatch. Then, the parameter is updated as
\begin{equation}
    \label{eq:aug_sgd}
    \theta_{t+1} = \theta_{t} - \frac{\eta_t}{|S_t|} \sum_{i\in S_t} \nabla L(\theta, g_{i, t} X_i).
\end{equation}

Here we need a probability distribution $\mathbb{Q}$ on the group $G$, from which $g_{i, t}$ is sampled. For a finite $G$, one usually takes $\mbQ$ to be the uniform distribution. However, care must be taken if $G$ is infinite. We assume $G$ is a compact topological group, and we take $\mbQ$ to be the \emph{Haar probability measure}.\footnote{Haar measures are used for convenience. Most of our results hold for more general measures with slightly more lengthy proofs.} Hence, for any $g\in G$ and measurable $S \subseteq G$, the following \emph{translation invariant} property holds: $\mbQ(gS) = \mbQ(S), \mbQ(Sg) = \mbQ(S)$.

A key observation is that the update rule \eqref{eq:aug_sgd} corresponds to SGD on an \emph{augmented} empirical risk, where we take an average over all augmentations according to the measure $\mbQ$:
\begin{equation}
    \label{eq:aug_empirical_risk}
    \min_{\theta} \bar{R}_{n} (\theta) := \frac{1}{n} \sum_{i = 1}^n \int_G L(\theta, gX_i) d\mbQ(g).
\end{equation}    
To be precise, $\nabla L(\theta, g_{i, t} X_i)$ is an unbiased stochastic gradient for the \emph{augmented} loss function 
\begin{equation}
    \label{eq:aug_loss}
    \bar{L}(\theta, X) := \int_G L(\theta, gX) d\mbQ(g),
\end{equation}
and hence we can view the resulting estimator as an empirical risk minimizer of $\bar{R}_n = n^{-1}\sum_{i=1}^n \bar{L}(\theta, X_i)$. 
	
This can be viewed as Rao-Blackwellizing the loss, meaning taking a conditional expectation of the loss over the conditional distribution of $x$ belonging to a certain group orbit. See Lemma \ref{lemma:exact_invariance_lemma} for a precise statement and Section \ref{subsec:otherconn} for more discussion. \\



\subsection{Variants of MLE under invariance} \label{subsec:variants_of_MLE}
The maximum likelihood estimation (MLE) is a special case of ERM when there is an underlying parametric model. The usual approach goes as follows. One starts with a parametric statistical model (i.e., a collection of probability measures) $\{\P_\theta: \theta\in\Theta\}$, where $\Theta$ is some parameter space. Let $p_\theta$ be the density of $\mbP_\theta$, and we define the log-likelihood function as $\ell_\theta(x) = \log p_\theta(x)$. We then define the MLE as any solution of the following problem:
\begin{equation}
\label{eq:vanilla_mle}
	\hat\theta_{\textnormal{MLE},n} \in \arg\max_{\theta \in \Theta} \frac{1}{n}\sum_{i\in[n]} \ell_\theta(X_i).
\end{equation}	
The invariance assumption \eqref{eq:inv} is only imposed on the ``true'' parameter $\thetanull$, so that for $\P_\thetanull$-a.e. $x$ and $\mbQ$-a.e. $g$, we have
\begin{equation}
	\label{eq:inv_assump_in_mle}
	p_\thetanull(gx) \cdot |\det Jac(x\to gx)| = p_\thetanull(x),
\end{equation}	
where $Jac(x\to gx)$ is the Jacobian of the transform $x \mapsto gx$.

There are multiple ways to adapt the likelihood function to the invariance structure (with data augmentation being one of them), and we detail some of them below. \\
	
{\bf Constrained MLE.} The invariance structure is a constraint on the density function. Hence we obtain a natural \emph{constrained} (or restricted) maximum likelihood estimation problem. Define the \emph{invariant subspace} as 
\begin{equation}
\label{eq:inv_subspace}
\Theta_G = \{\theta\in\Theta:X =_d gX,\forall g\in G\}.
\end{equation} 
Then the constrained MLE is 
\begin{equation}
\label{eq:constr_mle}
	\hat\theta_{\textnormal{cMLE},n} \in \arg\max_{\theta \in \Theta_G} \sum_{i\in[n]} \ell_\theta(X_i).
\end{equation}	
	
In general, this can be much more sample-efficient than the original MLE. For instance, suppose we are trying to estimate a normal mean based on one sample: $X\sim \N(\theta,1)$. Let the group be negation, i.e., $G = (\{\pm 1\},\cdot) = \mathbb{Z}_2$. Then $\Theta_G = \{0\}$, because the only normal density symmetric around zero is the one with zero mean. Hence, the invariance condition uniquely identifies the parameter, showing that the constrained MLE perfectly recovers the parameter.
	
However, in general, optimizing over the restricted parameter set may be computationally more difficult. This is indeed the case in the applications we have in mind. For instance, in deep learning, $G$  may correspond to the set of all translations, rotations, scalings, shearings, color shifts, etc. And it's not clear how to obtain information on $\Theta_G$ (for example, how to compute the projection operator onto $\Theta_G$).\\

{\bf Augmented MLE.} A particular example of the augmented ERM is augmented maximum likelihood estimation. Here the loss is the negative log-likelihood, $L(\theta,x)  = -\ell_\theta(x)$. Then the augmented ERM program \eqref{eq:aug_empirical_risk} becomes 
\begin{equation}
\label{eq:augmented_mle}
	\hat \theta_{\textnormal{aMLE},n} \in \arg \max_\theta \sum_{i\in[n]} \int_G \ell_\theta(gX_i) d\mbQ(g) .
\end{equation}

{\bf Invariant MLE.} Another perspective to exploit invariance is that of invariant representations, i.e., learning over representations $T(x)$ of the data such that $T(gx) = T(x)$  (see also Section \ref{subsec:otherconn}). However, it turns out that in some natural examples, the invariant MLE does not gain over the usual MLE (see Appendix \ref{invrep_thy} for a specific example). Thus we will not consider this in much detail.\\
	
{\bf Marginal MLE.}	There is a natural additional method to estimate the parameters, the marginal MLE. Under exact invariance, our original local invariance assumption \eqref{eq:inv_assump_in_mle} is equivalent to the following \emph{latent variable} model. Under the true parameter $\theta_0$, we sample a random group element $g\sim \mbQ$, and a random datapoint $\tilde X \sim  \mbP_\thetanull$, i.e., 
\begin{equation}
\label{eq:latet_var_model}
	\tilde X \sim \mbP_\thetanull,\,g\sim \mbQ.
\end{equation}
Then, we observe $X = g\tilde X$. We repeat this independently over all datapoints to obtain all $X_i$. Since $gX =_d X$ under $\theta_0$, this sampling process is exactly equivalent to the original model, under $\theta_0$. 
	
Suppose that instead of fitting the model \eqref{eq:latet_var_model}, we attempt to fit the relaxed model $\tilde X \sim \mbP_\theta,\, g \sim \mbQ$, observing  $X = g\tilde X$. The only change is that we assume that the invariance holds for all parameter values. This model is mis-specified, nonetheless, it may be easier to fit computationally. Moreover, its MLE may still retain consistency for the original true parameter.
Now the maximum marginal likelihood estimator (ignoring terms constant with respect to $\theta$) can be written as:
\begin{equation}
	\label{eq:def_of_marginal_mle}
	\hat \theta_{\textnormal{mMLE},n} = \arg \max_\theta \sum_{i\in[n]} 
	\log\left(\int_G p_\theta(gX_i) d\mbQ(g)\right).
\end{equation}
We emphasize that this is not the same as the augmented MLE estimator considered above. This estimator has the $\log(\cdot)$ terms outside the $G$-integral, while the augmented one effectively has the $\log(\cdot)$ terms inside.\\

\begin{table}[bt]
	\renewcommand{\arraystretch}{1.8}
	\small
	\centering
	\caption{Optimization objectives}
	\label{opt}
	\begin{tabular}{|l|l|l|l|l|l|}
		\hline
		Method  & Sample Objective & Population Objective\\ \hline
		ERM/MLE        
		& $\underset{\theta\in\Theta}\min\frac{1}{n}\sum_{i\in[n]} L(\theta,X_i)$ 
		& $\underset{\theta\in\Theta}{\min}\E_{\theta_0}L(\theta,X)$ \\ \hline
		Constrained ERM/MLE          
		& $\underset{\theta\in\Theta_G}{\min}\frac{1}{n}\sum_{i\in[n]} L(\theta,X_i)$ 
		& $\underset{\theta\in\Theta_G}\min\E_{\theta_0}L(\theta,X)$  \\ \hline
		Augmented ERM/MLE 
		& $\underset{{\theta\in\Theta}}{\min}\frac{1}{n}\sum_{i\in[n]} \int L(\theta,gX_i) d\mbQ(g)$ 
		& $\underset{{\theta\in\Theta}}{\min}\E_{\theta_0}\int L(\theta,gX) d\mbQ(g)$ \\ \hline
		Invariant ERM/MLE        
		& $\underset{{\theta\in\Theta}}{\min}\frac{1}{n}\sum_{i\in[n]} L(\theta,T(X_i))$,\ \ \  $T(x) = T(gx)$  
		& $\underset{{\theta\in\Theta}}{\min}\E_{\theta_0}L(\theta,T(X))$  \\ \hline
		Marginal MLE
		& $\underset{{\theta\in \Theta}}{\max} \frac{1}{n}\sum_{i\in[n]} \log\left(\int p_\theta(gX_i) d\mbQ(g)\right)$ 
		& $\underset{\theta\in\Theta}{\max}\E_{\theta_0} \log\left(\int p_\theta(gX) d\mbQ(g)\right)$ \\ \hline
	\end{tabular}
\end{table}
	
{\bf Summary of methods.} Thus we now have several estimators for the original problem: MLE, constrained MLE, augmented MLE, invariant MLE, and marginal MLE. We note that the former four methods are general in the ERM context, whereas the last one is specific to likelihood-based models. See Table \ref{opt} for a summary. \emph{Can we understand them?} In the next few sections, we will develop theoretical results to address this question.

\subsection{Beyond ERM}\label{sec:bey_erm}

The above ideas apply to empirical risk minimization. However, there are many popular algorithms and methods that are not most naturally expressed as plain ERM. For instance:
	
\begin{enumerate}
	\item Regularized ERM, e.g., Ridge regression and Lasso 
	\item Shrinkage estimators, e.g., Stein shrinkage
	\item Nearest neighbors, e.g., k-NN classification and regression
	\item Heuristic algorithms like Forward stepwise (stagewise, streamwise) regression, and backward stepwise. While these may be associated with an objective function, there may be no known computationally efficient methods for finding global solutions.
\end{enumerate}	
	
To apply data augmentation for those methods, let us consider a general estimator $\htheta(x)$ based on data $x$. The simplest idea would be to try to compute the estimator on \emph{all the data}, including the actual and transformed sets. Following the previous logic, if we have the invariance $X \approx_d gX$ for all $g\in G$, then after observing data $x$, we should ``augment" our data with $gx$, for all $g\in G$. Finally, we should run our method on this data. However, this idea can be impractical, as the entire data can be too large to work with directly. Therefore, we will take a more principled approach and work through the logic step by step, considering all possibilities. We will eventually recover the above estimator as well. \\
	
{\bf Augmentation distribution \& General augmentations.}  We define the \emph{augmentation distribution} as the set of values
\begin{equation}
\label{eq:aug_distr}
\htheta(gx),\,\, g\in G.
\end{equation}
We think of $x = (x_1,\ldots,x_n)$ as the entire dataset, and of $g\in G$ being a group element summarizing the action on every data point. The augmentation distribution is simply the collection of values of the estimator we would get if we were to apply all group transforms to the data. It has a special role, because we think of each transform as equally informative, and thus each value of the augmentation distribution is an equally valid estimator.
	
We can also make \eqref{eq:aug_distr} a proper probability distribution by taking a random $g\sim\mbQ$. Then we can construct a final estimator by computing a summary statistic on this distribution, for instance, the mean
\begin{equation}
	\label{eq: augmentation_via_aug_distr}
	\htheta_G(x) = \E_{g\sim\mbQ}\htheta(gx).
\end{equation}
	
It is worth noticing that this summary statistic is exactly invariant, so that $\htheta_G(x) = \htheta_G(gx)$. Moreover, this estimator can be approximated in the natural way in practice via sampling $g_i\sim\mbQ$ independently: 
\begin{equation}
	\label{eq: approx_aug_distr}
\htheta_k(x) = \frac1k \sum_{i=1}^k\htheta(g_ix).
\end{equation}

{\bf Connection to previous approach.} To see how this connects to the previous ideas, let us consider $X = (X_1,\ldots,X_n)$, and let $\htheta$ be the ERM with loss function $L(\theta,\cdot)$. Consider the group $G^n = G\times \ldots \times G$ acting on $X$ elementwise. Then, the augmentation distribution is the set of values
\beqs
	\htheta(X_1,\ldots,X_n; g_1,\ldots, g_n):=\arg\min_\theta\frac{1}{n}\sum_{i\in[n]} L(\theta,g_iX_i),
\eeqs
where we range $g\in G$.
Then, the final estimator, according to \eqref{eq: augmentation_via_aug_distr}, would be 
\beqs
	\htheta_G(X_1,\ldots,X_n):=\E_{g_1,\ldots, g_n\sim Q}\arg\min_\theta\frac{1}{n}\sum_{i\in[n]} L(\theta,g_iX_i).
\eeqs
	Compared to the previous augmented ERM estimator, the current one changes the order of averaging over $G$ and minimization. Specifically, the previous one is $\arg\min\E_g R_n(\theta; gX)$, while the current one is $\E_g \arg\min R_n(\theta;gX)$.  If we know that the estimator is obtained from minimizing a loss function, then we can average that loss function; but in general we can only average the estimator, which justifies the current approach. 
	
The two approaches above are closer than one may think. We can view the SGD iteration \eqref{eq:aug_sgd} as an online optimization approach to minimize a randomized objective of the form $\sum_{i\in[n]} L(\theta,g_iX_i)$. This holds exactly if we take one pass over the data in a deterministic way, which is known as a type of incremental gradient method. In this case, minimizing the augmented ERM has a resemblance to minimizing the mean of the augmentation distribution. However, in practice, people take multiple passes over the data, so this interpretation is not exact. \\

{\bf Augmentation in sequence of estimators.} In the above generalization of data augmentation beyond ERM, we assumed only the bare minimum, meaning that the estimator $\htheta(x)$ exists. Suppose now that we have slightly more structure, and the estimator is part of a sequence of estimators $\htheta_n$, defined for all $n$. This is a mild assumption, as in general estimators can be embedded into sequences defined for all data sizes.  Then we can directly augment our dataset $X_1,\ldots, X_n$ by adding new datapoints. We can define augmented estimators in several ways. For instance, for any fixed $m$, we can compute the estimator on a uniformly resampled set of size $m$ from the data, applying uniform random transforms:
	\begin{align*}
	&\htheta_m(g_1 X_{i_1},g_2 X_{i_2},\ldots, ,g_m X_{i_m})\\
	&i_k \sim \textnormal{Unif}([n]),\,\, g_k\sim \mbQ.
	\end{align*}
This implicitly assumes a form of symmetry of the estimator with respect to its arguments. There are many variations: e.g., we may insist that $m$ should be a multiple of $n$, or we can include all datapoints. This leads us to a ``completely augmented'' estimator which includes all data and all transforms (assuming $|G|$ is finite)
\begin{align*}
	&\htheta_{n|G|}(\{g_j X_{i}\}_{i\in [n],j \in [|G|]}).
\end{align*}
	
The advantage of the above reasoning is that it allows us to design augmented/invariant learning procedures extremely generally.

\subsection{Connections and perspectives} \label{subsec:otherconn}
	
Our approach has connections to many important and well-known concepts in statistics and machine learning. Here we elaborate on those connections, which should help deepen our understanding of the problems under study.\\
	
{\bf Sufficiency.} The notion of sufficiency, due to Ronald A Fisher, is a fundamental concept in statistics. Given an observation (datapoint) $X$ from a statistical model $X\sim \mbP$, a statistic $T:=T(X)$ is said to be sufficient for a parameter $\theta:=\theta(\P)$ if the conditional distribution of $X|T=t$ does not depend on $\theta$ for almost any $t$. Effectively, we can reduce the data $X$ to the statistic $T$ without any loss of information about $\theta$.
A statistic $T$ is said to be minimal sufficient if any other sufficient statistic is a function of it.
	
In our setup, assuming the invariance $X =_d gX$, on the invariant subspace $\Theta_G$, the \emph{orbits} $G\cdot x:=\{gx|g\in G\}$ are minimal sufficient for $\theta$. More generally, the local invariance condition (where invariance only holds for a subset of the parameter space) implies that the group orbits are a \emph{locally sufficient} statistic for our model. From the perspective of statistical theory, this suggests that we should work with the orbits. However, this can only be practical under the following conditions: 
\benum
	\item We can conveniently compute the orbits, or we can conveniently find representatives;
	\item We can compute the distribution induced by the model on the orbits;
	\item We can compute estimators/learning rules defined on the orbits in a convenient way.
\eenum
	
This is possible in many cases \citep{lehmann1998theory,lehmann2005testing}, but in complex cases such as deep learning, some or all of these steps can be impractical. For instance, the set of transforms may include translations, rotations, scalings, shearings, color shifts, etc. How can we compute the orbit of an image? It appears that an explicit description would be hard to find. \\

{\bf Invariant representations.} The notions of invariant representations and features are closely connected to our approach. Given an observation $x$, and a group of transforms $G$ acting on $\mathcal{X}$, a feature $F: \mathcal{X} \to  \mathcal{Y}$ is invariant if $F(x) = F(gx)$ for all $x,g$. This definition does not require a probabilistic model.  By design, convolutional filters are trained to look for spatially localized features, such as edges, whose pixel-wise representation is invariant to location. 
In our setup, we have a group acting on the data. In that case, it is again easy to see that the orbits $G\cdot x:=\{gx|g\in G\}$ are the maximal invariant representations. 
	
Related work by Mallat, B\"olcskei and others \citep[e.g.,][]{mallat2012group,bruna2013invariant,wiatowski2018mathematical,anselmi2019symmetry} tries to explain how CNNs extract features, using ideas from harmonic analysis. They show that the features of certain models of neural networks (Deep Scattering Networks for Mallat) are increasingly invariant with respect to depth. \\

{\bf Equivariance.} The notion of equivariance is also key in statistics \citep[e.g.,][]{lehmann1998theory}. A statistical model is called equivariant with respect to a group $G$ acting on the sample space if there is an \emph{induced group} $G^*$ acting on the parameter space $\Theta$ such that for any $X\sim \P_\theta$, and any $g\in G$, there is a $g^*\in G^*$ such that $gX\sim \P_{g^*\theta}$.  Under equivariance, it is customary to restrict to equivariant estimators, i.e., those that satisfy $\htheta(gx) = g^*\htheta(x)$. Under some conditions, there are Uniformly Minimum Risk Equivariant (UMRE) estimators. 
	
Our invariance condition can be viewed as having a ``trivial" induced group $G^*$, which always acts as the identity. Then the equivariant estimators are those for which $\htheta(gx) = \htheta(x)$. Thus, equivariant estimators are invariant on the orbits. 
	
The above mentioned UMRE results crucially use that the induced group has large orbits on the parameter space  (or in the extreme case, is transitive), so that many parameter values are equivalent. In contrast, we have the complete opposite setting, where the orbits are singletons. Thus our setting is very different from classical equivariance. \\

{\bf Exact invariance in some statistical models.} The exact invariance in distribution $X=_d gX$ arises naturally in many familiar statistical settings. In permutation tests, we are testing a null hypothesis $H_0$, under which the distribution of the observed data $X_1,\ldots,X_n$ is invariant to some permutations. For instance, in two-sample tests against a mean shift alternative, we are interested to test the null of no significant difference between two samples. Under the null, the data is invariant to all permutations, and this can be the basis of statistical inference \citep[e.g.,][]{lehmann2005testing}. In this case, the group is the symmetric group of all permutations, and this falls under our invariance assumption. However, the goals in permutation testing are very different from our ones.

In time series analysis, we have observations $X=(\ldots, X_1, X_2,\ldots, X_t, \ldots)$ measured over time. Here, the classical notion of strict stationarity is the same as invariance under shifts, i.e., $(\ldots, X_t,X_{t+1},\ldots) =_d (\ldots,X_{t+1},X_{t+2},\ldots)$ (e.g., \citealt{brockwell2009time}). Thus, our invariance in distribution can capture an important existing notion in time series analysis. Moreover, data augmentation corresponds to adding all shifts of the time series to the data. Hence, it is connected to auto-regressive methods.

We discuss other connections, e.g., to nonparametric density estimation under symmetry constraints, and U-statistics, in other places of this paper. \\

{\bf Variance reduction in Monte Carlo.} Variance reduction techniques are widely used in Monte Carlo methods \citep{robert2013monte}.  Data augmentation can be viewed as a type of variance reduction, and is connected to other known techniques.  For instance, $f(gX)$ can be viewed as \emph{control} variates for the random variable $f(X)$. The reason is that $f(gX)$ has the same marginal distribution, and hence the same mean as $f(X)$. Taking averages can be viewed as a suboptimal, but universal way to combine control variates.  
	
We briefly mention that under a reflection symmetry assumption, data augmentation can also be viewed as a special case of the method of \emph{antithetic variates}. \\

{\bf Connection to data augmentation.} We can summarize the connections to the areas mentioned above. Data augmentation is \emph{computationally feasible approach} to \emph{approximately} learn on the orbits (which are both the minimal sufficient statistics and maximal invariant features). The computational efficiency is partly because we never explicitly compute or store the orbits and invariant features. 


\section{Theoretical results under exact invariance}\label{sec:exact_inv}
	
In this section, we present our theoretical results for data augmentation under the assumption of \emph{exact invariance}: $g X =_d X$ for almost all $g \sim \mbQ$ and $x \sim \P$.

\subsection{General estimators} \label{subsec:gen_est}
We start with some general results on variance reduction. The following lemma characterizes the bias and variance of a general estimator under augmentation:
\begin{lemma}[Invariance lemma]
	\label{lemma:exact_invariance_lemma}
	Let $f$ be an arbitrary function s.t. the map $(X, g) \mapsto f(gX)$ is in $L^2(\P \times \mbQ)$. Assume exact invariance holds for $\mbQ$-almost all $g\in G$. Let $\bar f(x) : =  \E_{g\sim \mbQ} f(gx)$ be the ``orbit average'' of $f$. Then:
	\benum
		\item For any $x$, $\bar f(x)$ is the conditional expectation of $f(X)$, conditional on the orbit: $\bar f(x) = \E[f(X)  | X\in Gx]$, where $Gx:= \{gx: g\in G\}$;
		\item Therefore, by the law of total expectation, the mean of $\bar f(X)$ and $f(X)$ coincide: $\E_{X\sim \P} f(X) = \E_{X\sim \P} \bar f(X)$;
		\item By the law of total covariance, the covariance of $f(X)$ can be decomposed as
		$$
		\C_{X\sim \P} f(X) = \C_{X\sim \P} \bar f(X) + \E_{X\sim \P} \C_{g\sim \mbQ} f(gX);
		$$
		\item Let $\varphi$ be any real-valued convex function. Then $\E_{X\sim \P} [\varphi(f(X))] \geq \E_{X\sim \P} [\varphi(\bar f(X))]$.
	\eenum
\end{lemma}
\begin{proof}
	See Appendix \ref{subappend:proof_of_inv_lemma}.
\end{proof}
Figure \ref{fig:orbit_avg} gives an illustration of orbit averaging.\\

\begin{figure}[bt]
	\centering
	\includegraphics[scale=0.7]{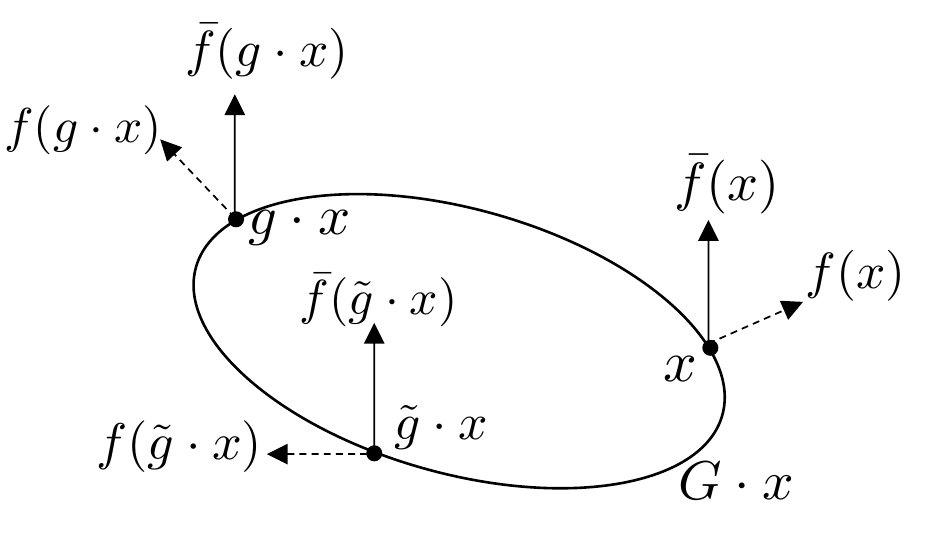}
	\caption{A pictorial illustration of orbit averaging. The circle represents the orbit $G\cdot x=\{gx: g\in G\}$, and $x, gx, \tilde gx$ are three different points on this orbit. Although $f(x), f(gx), f(\tilde gx)$ may be different, after orbit averaging, $\bar f(x), \bar f(gx), \bar f(\tilde gx)$ all take the same value.}
	\label{fig:orbit_avg}
\end{figure}

{\bf Augmentation leads to variance reduction.}  This lemma immediately implies that data augmentation has favorable variance reduction properties. For general estimators, we obtain a direct consequence: 
\begin{proposition}[Augmentation decreases MSE of general estimators] 
	\label{prop:cov_of_general_estimators}
	Assume exact invariance holds. Consider an estimator $\htheta(X)$ of $\thetanull$, and its augmented version $\htheta_G(X) = \E_{g\sim \mbQ} \htheta(gX)$. The bias of the augmented estimator is the same as the bias of the original estimator, and the covariance matrix decreases in the Loewner order:
		$$\Cov{\htheta_G(X) } \preceq \Cov{\htheta(X) }.$$
	Hence, the mean-squared erro (MSE) decreases by augmentation. Moreover, for any convex loss function $L(\thetanull, \cdot)$, we have 
	$$
	\E L(\thetanull, \htheta(X)) \geq \E L(\thetanull, \htheta_G(X)).
	$$
\end{proposition}
The estimator $\htheta_G(X)$ is an instance of data augmentation via \emph{augmentation distribution} \eqref{eq: augmentation_via_aug_distr}. We see that the augmented estimator is no worse than the original estimator according to any convex loss. In some cases, using such a strategy gives an unexpectedly large efficiency gain (See Section \ref{subsec:lin_reg_aug_distr} for an example). 

For ERM and MLE, the claim implies that the variance of the augmented loss, log-likelihood, and score functions all decreases.     
For other estimators based on the augmentation distribution, such as the median, one can show that other measures of error, such as the mean absolute error decrease. 
This shows how data augmentation can be viewed as a form of \emph{algorithmic regularization}. Indeed, the mean behavior of loss/log-likelihood, etc are all unchanged, but the variance decreases. This indeed shows that augmentation is a natural form of regularization. \\
	
{\bf Connection to Rao-Blackwell theorem.} Lemma \ref{lemma:exact_invariance_lemma} has the same flavor as the celebrated Rao-Blackwell theorem, where a ``better'' estimator is constructed from a preliminary estimator by conditioning on a sufficient statistic. In fact, the orbit $GX$ is a sufficient statistic in many cases and we provide a classical example below to illustrate this point:
\begin{example}[U-statistic as an augmented estimator]
	Consider data $X_1, ..., X_n$ sampled i.i.d. from some distribution $\P$. We are interested in estimating some functional $\theta$ of $\P$. Suppose we have a crude preliminary estimator $\htheta(X_1, ..., X_r)$, which takes $r \le n$ arguments. Let $G$ be the group acting on $(X_1, ..., X_n)$ by randomly selecting $r$ samples without replacement. Then the augmented estimator is
	$$
		\E_{g \sim \mbQ}[\htheta(g(X_1,..., X_n))] =\frac{1}{\binom{n}{r}} \sum_{i_1\neq i_2 \neq \cdots \neq i_r} \htheta(X_{i_1}, ..., X_{i_r}),
	$$
	where the summation is taken over the set of all unordered subsets $\{i_1, ..., i_r\}$ of $r$ different integers chosen from $[n]$.
	This is the U-statistic of order $r$ with kernel $\htheta$. The statistical properties of the U-statistic are better than its non-augmented counterpart, which does not use all the data. There are well-known explicit formulas for the variance reduction \citep[e.g.,][]{van1998asymptotic}.
\end{example}

{\bf Beyond groups.} Some of our conclusions hold without requiring a group structure on the set of transforms. Instead, it is enough to consider a set (i.e., a semigroup, because the identity always makes sense to include) of transforms $T:\xx\to\xx$, with a probability measure $\mbQ$ on them. This is more realistic in some applications, e.g., in deep learning where we also consider transforms such as cropping images, which are not invertible. Then we still get the variance reduction as in Lemma \ref{lemma:exact_invariance_lemma} (specifically, part 2, 3, and 4 still hold). Therefore, under appropriate regularity conditions, we also get improved performance for augmented ERM, as stated previously, as well as in the next sections. However, some of the results and interpretations do not hold in this more general setting. Specifically, the orbits are not necessarily defined anymore, and so we cannot view the augmented estimators as conditional expectations over orbits.
The semigroup extension may allow us to handle methods such as subsampling or the bootstrap in a unified framework.
	
\subsection{ERM / M-estimators}\label{subsec:theory_for_erm_under_exact_inv}
We now move on to present our results on the behavior of ERM. Recall our setup from Section \ref{sec:erm}, namely that 
\begin{align*}
	& \thetanull  \in \arg\min_{\theta\in \Theta} \E L(\theta, X) , \qquad  &\theta_G  \in \arg \min_{\theta \in \Theta} \E \bar L(\theta, X), \\
	& \htheta_n  \in \arg\min_{\theta \in \Theta} \frac{1}{n} \sum_{i = 1}^n L(\theta, X_i), \qquad  &\htheta_{n, G} \in \arg\min_{\theta\in \Theta}  \frac{1}{n} \sum_{i= 1}^n \bar L(\theta, X_i),
\end{align*}
where $\Theta$ is some parameter space.
Under exact invariance, it is easy to see that $\E L(\theta, X) = \E \bar L(\theta, X)$, and hence there is a one-to-one correspondence between $\thetanull$ and $\theta_G$. To evaluate an estimator $\htheta$, we may want
\benum
	\item small generalization error: we want $\E L(\htheta, X) - \E L(\thetanull, X)$ to be small;
	\item small parameter estimation error: we want $\|\htheta - \thetanull\|_2$ to be small.
	\eenum
Comparing $\htheta_n$ with $\htheta_{n, G}$, we will see in Section \ref{subsubsec:effect_of_loss_avg} that the reduction in generalization error can be quantified by averaging the \emph{loss function} over the orbit, whereas in Section \ref{subsubsec:effect_of_gradient_avg}, it is shown that the reduction in parameter estimation error can be quantified by averaging the \emph{gradient} over the orbit.
	
\subsubsection{Effect of loss-averaging} \label{subsubsec:effect_of_loss_avg}
To obtain a bound on the generalization error, we first present what can be deduced quite directly from known results on Rademacher complexity. The point of this section is mainly pedagogical, i.e., to remind readers of some basic definitions, and to set a baseline for the more sophisticated results to follow. We recall the classical approach of decomposing the generalization error into terms that can be bounded via concentration and Rademacher complexity \citep{bartlett2002rademacher,shalev2014understanding}:
\begin{align*}
	\E L(\htheta_n, X) - \E L(\thetanull, X) & =  \E L(\htheta_n, X) - \frac{1}{n}\sum_{i=1}^n L(\htheta_n, X_i) + \frac{1}{n}\sum_{i=1}^n L(\htheta_n, X_i) - \E L(\thetanull, X).
\end{align*}
The second half of the RHS above can be bounded as
\begin{align*}
	\frac{1}{n}\sum_{i=1}^n L(\htheta_n, X_i) - \E L(\thetanull, X) 
	& \leq \frac{1}{n}\sum_{i=1}^n L(\thetanull, X_i) - \E L(\thetanull, X),
	\end{align*}
because $\htheta_n$ is a minimizer of the empirical risk.
Hence we arrive at
\begin{align*}
	\E L(\htheta_n, X) - \E L(\thetanull, X) & \leq \E L(\htheta_n, X) - \frac{1}{n}\sum_{i=1}^n L(\htheta_n, X_i) + \frac{1}{n}\sum_{i=1}^n L(\thetanull, X_i) - \E L(\thetanull, X) \\ 
	& \leq  \sup_{\theta\in \Theta} \bigg|\frac{1}{n} \sum_{i=1}^n L(\theta, X_i) - \E L(\theta, X)\bigg| + \bigg(\frac{1}{n}\sum_{i=1}^n L(\thetanull, X_i) - \E L(\thetanull, X)\bigg).
\end{align*}
Using exact invariance, a similar computation gives
$$
	\E L(\htheta_{n, G}, X) - \E L(\thetanull, X) \leq  \sup_{\theta\in \Theta} \bigg|\frac{1}{n} \sum_{i=1}^n \bar L(\theta, X_i) - \E \bar L(\theta, X)\bigg| + \bigg(\frac{1}{n}\sum_{i=1}^n \bar L(\thetanull, X_i) - \E \bar L(\thetanull, X)\bigg).
$$
To control the two terms in the RHS, we need to show that both $n^{-1}\sum_{i=1}^n L(\theta, X_i)$ and $n^{-1}\sum_{i=1}^n \bar L(\theta, X_i)$ concentrate around their means, respectively, uniformly for all $\theta \in \Theta$. Intuitively, as $\bar L$ is an averaged version of $L$, we would expect that the concentration of $\bar L$ \emph{happens at a faster rate} than that of $L$, hence giving a tighter generalization bound. We will make this intuition rigorous in the following theorem:

\begin{theorem}[Rademacher bounds under exact invariance] \label{thm:rademacher_bound_under_exact_inv}
	Assume the loss $L(\cdot, \cdot) \in [0, 1]$ and assume exact invariance holds. Then with probability at least $1-\delta$ over the draw of $X_1, ..., X_n$, we have
	\begin{align*}
		\E L(\htheta_n, X) - \E L(\thetanull, X) & \leq 2 \rad_n(L \circ \Theta) +  \sqrt{\frac{2\log 2/\delta}{n}}
		\textnormal{   \,\,\,\,\,\, (Classical Rademacher bound)}
		\\
		\E L(\htheta_{n,G}, X) - \E L(\thetanull, X) & \leq 2 \rad_n(\bar L \circ \Theta) +  \sqrt{\frac{2\log 2/\delta}{n}} ,
		\textnormal{   \,\,\,\, (Implication for data augmentation)}
	\end{align*}
	where $\rad_n(L \circ \Theta)$ and $\rad_n(\bar L \circ \Theta)$ are the Rademacher complexity of the original loss class and the augmented loss class, repectively, defined as
	\begin{align*}
		&\rad_n(L\circ \Theta)  = \E \sup_{\theta \in \Theta} \bigg| \frac{1}{n} \sum_{i= 1}^n \ep_i L(\theta, X_i)  \bigg|,  \qquad \rad_n(\bar L\circ \Theta)  = \E \sup_{\theta \in \Theta} \bigg| \frac{1}{n} \sum_{i= 1}^n \ep_i \bar L(\theta, X_i)  \bigg|, \\
		& \ep_i  \overset{\textnormal{i.i.d.}}{\sim} \textnormal{Rademacher},  \ \{\ep_i\}_{i=1}^n \indep \{X_i\}_{i=1}^n,
	\end{align*}
	where the expectation is taken over both $\{\ep_i\}_1^n$ and $\{X_i\}_1^n$.
	Moreover, augmentation \emph{decreases} the Rademacher complexity of the loss class:
	$$
		\rad_n(\bar L \circ \Theta) \leq \rad_n(L\circ \Theta).  
	$$
\end{theorem}
\begin{proof}
	This is a special case of Theorem \ref{thm:rademacher_bound_under_approx_inv}. We refer the readers to Appendix \ref{subappend:proof_of_rademacher_bound_under_approx_inv} for a proof.
\end{proof}
The first generalization bound is a standard result. It can be viewed as an intermediate between Theorems 26.3 and Theorem 26.5 of \cite{shalev2014understanding}, part 3, because it is a high-probability bound (like 26.5) for expected generalization error (like 26.3). The second bound simply applies the first one to data augmentation, and the point is that we obtain a sharper result.

We also note that the decrease of the Rademacher complexity under augmentation fits into the ``Rademacher calculus'' \cite{shalev2014understanding}, along with results such as bounds on the Rademacher complexity under convex combinations and Lipschitz transforms. The current claim holds by inspection.

	
	\subsubsection{Effect of gradient-averaging} \label{subsubsec:effect_of_gradient_avg}
	Just as the averaged loss function concentrates faster around its mean, the averaged gradient also concentrates faster. Specifically, an application of Lemma \ref{lemma:exact_invariance_lemma} gives
	$$
	\C \nabla L(\theta, X) = \C \nabla \bar L(\theta, X) + \E_{X\sim \P}\C_{g\sim\mbQ} \nabla L(\theta, gX). 
	$$
	If we assume the loss function is strongly convex w.r.t. the parameter, then this observation, along with a recent result from \cite{foster2019complexity}, gives a bound of the parameter estimation error for both $\htheta_n$ and $\htheta_{n, G}$:
\begin{theorem}[Data augmentation reduces variance in ERM] 
\label{thm:bound_by_stoch_opt_under_exact_inv}
	Assume the loss $L(\cdot, x)$ is $\lambda$-strongly convex and assume exact invariance holds. Then
	\begin{align*}
		\E\|\htheta_{n}-\thetanull\|_2^2 & \le \frac{4}{\lambda^2 n} \tr\bigg(\C \nabla L(\thetanull, X)\bigg) \qquad\qquad\qquad\qquad \ \ \ \ 
		\textnormal{(Classical ERM bound)}
		\\
		\E\|\htheta_{n, G} - \thetanull\|_2^2 & \le \frac{4}{\lambda^2 n}  \tr\bigg(\C \nabla L(\thetanull, X) - \E_X\C_g \nabla L(\theta, gX) \bigg) \\
		& \qquad\qquad\qquad\qquad\qquad\qquad\qquad\qquad \textnormal{(Implication for data augmentation)}.
	\end{align*}
\end{theorem}
\begin{proof}
	The proof is simply a matter of checking the conditions required by Theorem 7 of \cite{foster2019complexity}. It is not hard to see that the strong convexity is preserved under averaging, and so the strong convexity constant of $\theta \to \bar L(\theta,x)$ is at least as large as that of $\theta \to L(\theta,x)$. Then invoking part 3 of Lemma \ref{lemma:exact_invariance_lemma} gives the desired result.
\end{proof}
	
The above result is similar to Theorem \ref{thm:rademacher_bound_under_exact_inv} in that it gives a sharper upper bound on the estimation error.

\subsubsection{Asymptotic results} \label{subsubsec:asymp_res_under_exact_invariance}
The finite sample results in Theorem \ref{thm:rademacher_bound_under_exact_inv} and \ref{thm:bound_by_stoch_opt_under_exact_inv} do not precisely tell us how much we gain. This motivates us to consider an asymptotic regime where we can characterize precisely how much we gain by data augmentation.

We assume the sample space $\mathcal{X} \subseteq \RR^d$ and the parameter space $\Theta \subseteq \RR^p$. We will consider the classical under-parameterized regime, where both $d$ and $p$ are fixed and $n \to \infty$. The over-parameterized regime will be handled in Section \ref{sec:overparam_2l_nets}. Under weak regularity conditions, it is a well-known result that $\htheta_n$ is asymptotically normal with covariance given by the inverse Fisher information (see, e.g, \citealt{van1998asymptotic}). We will see that a similar asymptotic normality result holds for $\htheta_{n, G}$ and its asymptotic covariance is explicitly computable. This requires some standard assumptions, and we state them for completeness below.

\begin{assump}[Regularity of the population risk minimizer]
	\label{assump:regularity_of_thetanull}
	The minimizer $\thetanull$ of the population risk is well separated: for any $\ep > 0$, we have
	$$
		\sup_{\theta: \|\theta - \theta_0 \| \geq \ep} \E L(\theta,X) > \E L(\thetanull,X).
	$$
\end{assump}
\begin{assump}[Regularity of the loss function]
	\label{assump:regularity_of_loss}
	For the loss function $L(\theta, x)$, we assume that
	\benum
		\item uniform weak law of large number holds: 
		$$\sup_{\theta \in \Theta} |\frac{1}{n}\sum_{i=1}^n L(\theta,X_i) - \E L(\theta,X)|  \overset{p}{\to} 0;$$
		\item for each $\theta \in \Theta$, the map $x \mapsto L(\theta, x)$ is measurable;
		\item the map $\theta \mapsto L(\theta, x)$ is differentiable at $\thetanull$ for almost every $x$;
		\item there exists a $L^2(\P)$ function $\dot{L}$ s.t. for almost every $x$ and for every $\theta_1, \theta_2$ in a neighborhood of $\thetanull$, we have $$|L(\theta_1,x) - L(\theta_2,x) | \leq \dot L(x) \|\theta_1 - \theta_2 \|;$$
		\item the map $\theta \mapsto \E L(\theta, X)$ admits a second-order Taylor expansion at $\thetanull$ with non-singular second derivative matrix $V_{\thetanull}$. 
	\eenum
\end{assump}

Note that the two assumptions above can be defined in a similar fashion for the pair $(\theta_G, \bar L)$. The lemma below says that such a re-definition is unnecessary under exact invariance.

\begin{lemma}[Regularity of the augmented loss]
	\label{lemma:regularity_of_augmented_loss}
	For each $\theta \in \Theta$, assume the map $(X, g) \mapsto L(\theta, gX)$ is in $L^1(\P\times \mbQ)$. Assume exact invariance holds. If the pair $(\thetanull, L)$ satisfies Assumption \ref{assump:regularity_of_thetanull} and \ref{assump:regularity_of_loss}, then the two assumptions also hold for the pair $(\theta_G, \bar L)$.
	\end{lemma}
\begin{proof}
	See Appendix \ref{subappend:proof_of_regularity_of_augmented_loss}.
\end{proof}

By the above lemma, we conclude that $\htheta_{n, G}$ is also asymptotically normal, and a covariance calculation gives the following theorem. We present the classical results for M-estimators in parallel for clarity.

\begin{theorem}[Asymptotic normality for the augmented estimator]
	\label{thm:asymp_normality_augmented_estimator}
	Assume $\Theta$ is open and the conditions in Lemma \ref{lemma:regularity_of_augmented_loss} hold. Then $\htheta_n$ and $\htheta_{n, G}$ admit the following Bahadur representation:
	\begin{align*}
		\sqrt{n} (\htheta_{n} - \thetanull) 
		& = 
		\frac{1}{\sqrt{n}}
		V_\thetanull^{-1} \sum_{i=1}^n \nabla L(\thetanull,X_i) 
		+ o_p(1) 
		\textnormal{   \,\,\,\, (Classical representation for M-estimators)}
		\\
		\sqrt{n} (\htheta_{n,G} - \thetanull) 
		& = 
		\frac{1}{\sqrt{n}}
		V_\thetanull^{-1} \sum_{i=1}^n \nabla \bar L(\thetanull,X_i) 
		+ o_p(1).
		\textnormal{   \,\,\, (Representation for augmented M-estimators)}
	\end{align*}
	Therefore, both $\htheta_n$ and $\htheta_{n, G}$ are asymptotically normal:
	\begin{align*}
		\sqrt{n} (\htheta_{n} - \thetanull)  \Rightarrow \N (0, \Sigma_0), \qquad  
		\sqrt{n} (\htheta_{n, G} - \thetanull)  \Rightarrow \N (0, \Sigma_G ) , 
	\end{align*}
	where the covariance is given by
	\begin{align*}
		\Sigma_0 & = V_\thetanull^{-1} \E_X [\nabla L(\thetanull,X) \nabla L(\thetanull,X)^\top] V_{\thetanull}^{-1} 
		\textnormal{   \,\,\,\,\,\, (Classical covariance)}
		\\
		\Sigma_G & = \Sigma_0 -  V_\thetanull^{-1} \E_X [\C_g \nabla L(\thetanull,gX)] V_{\thetanull}^{-1}.
			\textnormal{   \,\,\,\,\,\, (Augmented covariance)}
	\end{align*}
	As a consequence, the asymptotic relative efficiency of $\htheta_{n, G}$ compared to $\htheta_n$ is $\textnormal{RE} = \frac{\tr(\Sigma_0)}{\tr(\Sigma_G)} \geq 1.$
\end{theorem}
\begin{proof}
	See Appendix \ref{subappend:proof_of_asymp_normality_of_augmented_estimator}.
\end{proof}

{\bf Interpretation.} The reduction in covariance is governed by $\E_X \C_g \nabla L(\thetanull,gX)$. This is the average covariance of the gradient $\nabla L$ along the orbits $Gx$. If the gradient varies a lot along the orbits, then augmentation gains a lot of efficiency. This makes sense, because this procedure effectively denoises the gradient, making it more stable. See Figure \ref{fig:cov_G} for an illustration.
    
\begin{figure}[bt]
    \centering
    \includegraphics[scale=0.75]{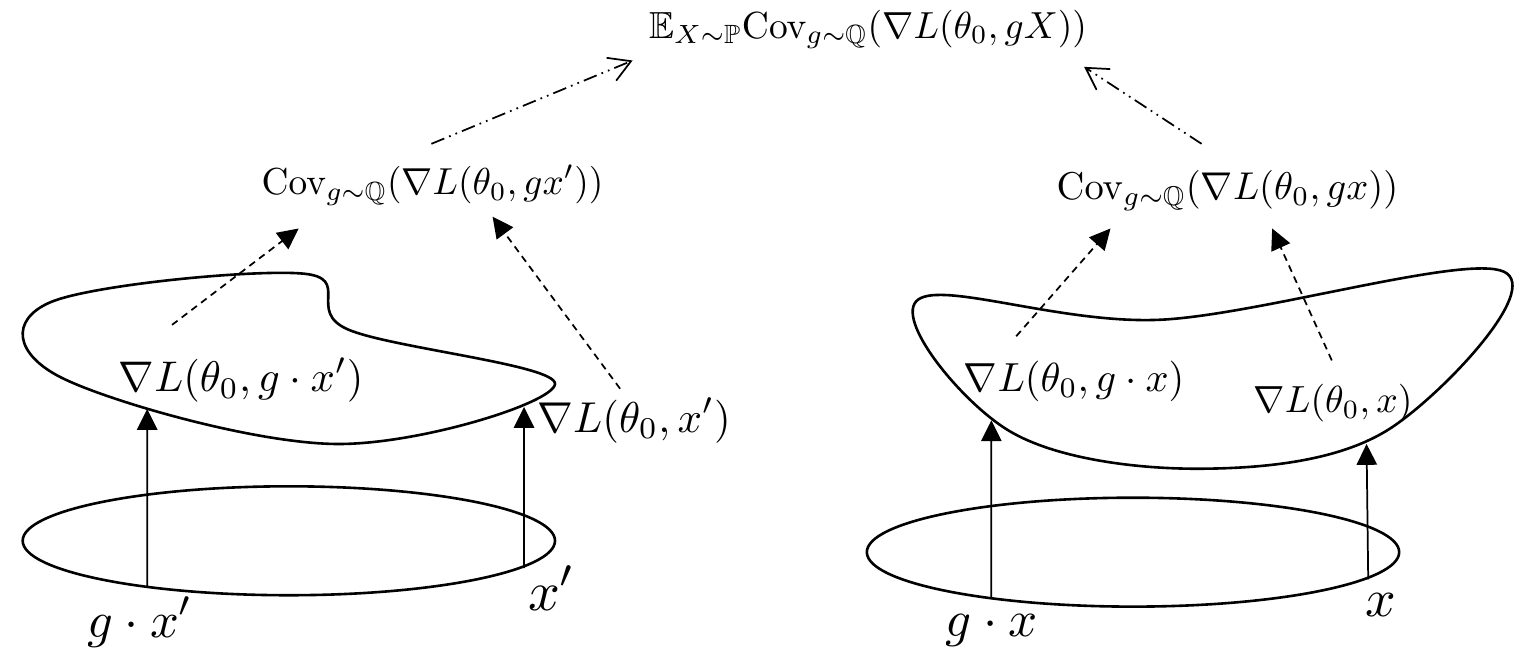}
    \caption{A pictorial illustration on computing the average covariance $\E_X \C_g \nabla L(\theta_0, gX)$ of the gradient of the loss over an orbit.}
    \label{fig:cov_G}
\end{figure}
    
Alternatively, the covariance of the augmented estimator can also be written as $V_\thetanull^{-1}$ $\C_X \E_g \nabla L(\thetanull,gX)$ $V_{\thetanull}^{-1}$.
The inner term is the covariance matrix of the orbit-average gradient, and similar to above, augmentation effectively denoises and reduces it. \\

{\bf Understanding the Bahadur representation.} It is worthwhile to understand the form of the Bahadur representations. For the ERM, we sum the terms $f(X_i):=\nabla L(\thetanull,X_i)$, while for the augmented ERM, we sum the terms $\bar f(X_i)=\nabla \bar L(\thetanull,gX_i)$. Thus, even in the Bahadur representation, we can see clearly that the effect of data augmentation is to \emph{average the gradients over the orbits}. \\
    
{\bf Statistical inference.} It follows automatically from our theory that statistical inference for $\theta$ can be performed in the usual way. Specifically, assuming that $\theta\mapsto L(\theta,x)$ is twice differentiable at $\theta_0$ on a set of full measure, we will have that $V_{\theta_0} = \E\nabla^2 L_{\theta_0}(X)$. We can then compute the plug-in estimator of $V_{\theta_0}$:
$$\hat V_{\theta_0} 
    = \frac{1}{n}\sum_{i=1}^n\nabla^2 L(\htheta_{n,G},X_i).
$$
Let us also define the plug-in estimator of the gradient covariance: 
$$
	\widehat{\bar I_{\thetanull}} 
    = \frac{1}{n}\sum_{i=1}^n
    \nabla L(\htheta_{n,G},X_i) 
    \nabla L(\htheta_{n,G},X_i)^\top.
$$
We can then define the plug-in covariance estimator
$$
	\hSigma_G
    = 
    \hat V_{\theta_0}^{-1} 
    \hat{\bar I_{\theta_0}} 
    \hat V_{\theta_0}^{-1}.
$$
This leads to the following corollary, whose proof is quite direct and we omit it for brevity.
\begin{corollary}[Statistical inference with the augmented estimator]
    \label{in_augmented_estimator}
    Assume the same conditions as Theorem \ref{thm:asymp_normality_augmented_estimator}. Assume in addition that $\theta\to L(\theta,x)$ is twice differentiable at $\theta_0$ on a set of full $x$-measure, and that each entry of the Hessian $\nabla^2 L(\theta_0,\cdot)$ is in $L^1(\P)$. Assume that for both of the functions $F_i$, $i=1,2$   $F_1(\theta,x) = \nabla^2 L(\theta,x)$ and $F_1(\theta,x) = \nabla L(\theta,x) \nabla L(\theta,x)^\top$ there exist $L^1(\P)$ functions $L_i$ such that for every $\theta_1, \theta_2$ in a neighborhood of $\thetanull$, we have   
    $$
        \|F_i(\theta_1,x) - F_i(\theta_2,x) \| 
        \leq 
        L_i(x) \|\theta_1 - \theta_2 \|.
    $$
    Then we have:
	$$
		\hSigma_G \overset{p}{\to} \Sigma_G.
	$$
	Therefore, we have the asymptotic normality
	$$
		\sqrt{n} \hSigma_G^{-1/2}(\htheta_{n, G} - \thetanull) \Rightarrow \N \bigg(0, I_p\bigg).
	$$
	Hence, statistical inference for $\theta_0$ can be performed in the usual way, constructing normal confidence intervals and tests based on the asymptotic pivot.
\end{corollary}

{\bf ERM in high dimensions.} We conclude this subsection by noting that asymptotic covariance formula can hold in some high-dimensional ERM problems, where the dimension of the parameter space can diverge as $n\to \infty$. Specifically, if we assume $\thetanull$ is sparse and we are able to perform variable selection in a consistent way, then the original (high-dimensional) problem can be reduced to a low-dimensional one, and the same asymptotic variance formulas hold (see, e.g., \citealt{wainwright2019high}). Then augmentation will lead to benefits as above.

\subsubsection{Implications for MLE}\label{subsubsec:implication_for_mle}

Now we consider the special case of maximum likelihood estimation in a model $\{\P_\theta: \theta \in \Theta \}$. Recall the setup from Section \ref{subsec:variants_of_MLE}:
\begin{align*}
	\hthetaMLE \in \underset{\theta\in\Theta}{\arg \max} \frac{1}{n}\sum_{i=1}^n \ell_\theta(X_i), \qquad \hthetaAMLE \in \underset{\theta\in\Theta}{\arg \max} \frac{1}{n}\sum_{i=1}^n \bar \ell_\theta(X_i),
\end{align*}
where $\bar \ell_\theta(x) = \E_{g\sim \mbQ} \ell_\theta(gx)$ is the averaged version of the log-likelihood. 

In the context of ERM, we have $L(\theta, x) = -\ell_\theta(x)$.
Hence, if the conditions in Theorem \ref{thm:asymp_normality_augmented_estimator} hold, we have
\begin{align*}
	\sqrt{n}({\hthetaMLE - \thetanull}) & \Rightarrow \N(0, I_\thetanull^{-1}) \\
	\sqrt{n}({\hthetaAMLE - \thetanull}) & \Rightarrow \N\bigg(0, I_{\thetanull}^{-1} \bigg(I_\thetanull - \E_\thetanull \C_g \nabla \ell_\thetanull(gX) \bigg) I_{\thetanull}^{-1}\bigg),
\end{align*}
where $I_\thetanull  = \E_\thetanull \nabla \ell_\thetanull \nabla \ell_\thetanull^\top$ is the Fisher information at $\thetanull$. \\

{\bf Invariance on a submanifold.} Note that the gain in efficiency is again determined by the term
$$
	\E_\thetanull \C_{g} \nabla \ell_\thetanull(gX).
$$
We will see that the magnitude of this quantity is closely related to the geometric structure of the parameter space. To be clear, let us recall the invariant subspace $\Theta_G$ defined in Equation \eqref{eq:inv_subspace}.
If this subspace is ``small'', we expect augmentation to be very efficient. Continuing to assume that the whole parameter set $\Theta$ is an open subset of $\R^p$, we additionally suppose that we can identify $\Theta_G \subseteq \Theta$ with a smooth submanifold of $\R^p$. 
If we assume $G$ \emph{acts linearly} on $\xx$, then every $g\in G$ can be represented as a matrix, and thus the Jacobian of $x\mapsto gx$ is simply the matrix representation of $g$. As a result, for every $\theta \in \Theta_G$, Equation \eqref{eq:inv_assump_in_mle} becomes
\begin{equation*}
	p_\theta(x) = p_\theta(g^{-1}x) / |\det(g)|.
\end{equation*}
As a result, we have
\begin{equation}
	\label{eq:relation_between_log_likelihood_of_x_and_gx}
	\ell_\theta(gx) +  \log |\det(g)| = \ell_\theta(x),\, \ \forall \theta\in\Theta_G, g\in G.
\end{equation}
It is temping to differentiate both sides at $\thetanull$ and obtain an identity. However, it is important to note that in general, for the score function, 
$$
	\nabla \ell_\thetanull(gx) \neq \nabla \ell_\thetanull(x).
$$
	To be clear, this is because the differentiation operation at $\thetanull$ may not be valid. Indeed, Equation \eqref{eq:relation_between_log_likelihood_of_x_and_gx} does not necessarily hold in an open neighborhood of $\thetanull$. If it holds in an open neighborhood, we can differentiate and get the identity for the score. However, it may only hold in a lower dimensional submanifold locally around $\thetanull$, in which case the desired equation for the score function only holds in the \emph{tangential direction} at $\thetanull$. Figure \ref{fig:tangent} provides an illustration of this intuition, and this intuition is made rigorous by the following proposition:

\begin{figure}[bt]
	\centering
	\includegraphics[scale=0.6]{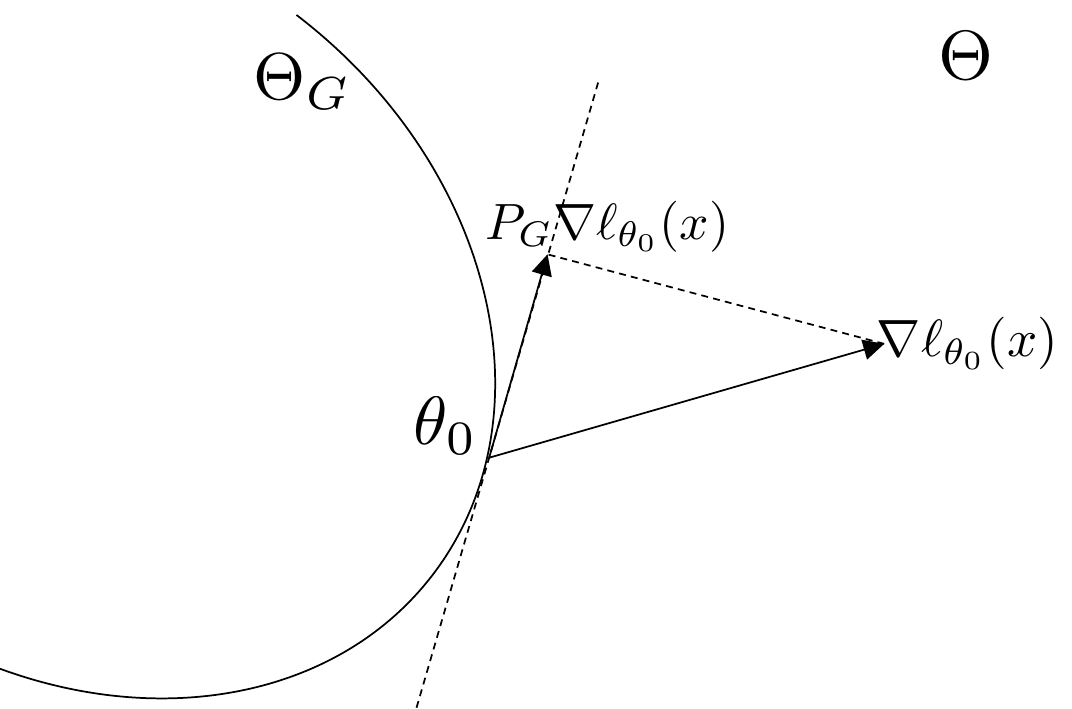}
	\caption{A pictorial illustration of projection onto local tangent space of $\Theta_G$ at the point $\thetanull$.}
	\label{fig:tangent}
\end{figure}

\begin{proposition}[MLE under invariance on a submanifold]
	\label{prop:tangential_decomp}
	Let the conditions of Theorem \ref{thm:asymp_normality_augmented_estimator} hold with $L(\theta, x) = - \ell_\theta(x)$. Assume $G$ acts linearly on the sample space $\xx$. Let $P_G$ be the orthogonal projection operator onto the tangent space of the manifold $\Theta_G$ at the point $\thetanull$, and let $P_G^\perp = Id - P_G$. Then we can decompose 
	$$
		\nabla \ell_\thetanull(gx) = P_G\nabla \ell_\thetanull(gx) + P_G^\perp \nabla\ell_\thetanull(gx),
	$$
	and the tangential part is invariant:
	$$
		P_G \nabla \ell_\thetanull(gx)  =  P_G \nabla \ell_\thetanull(x) ,\,\ \forall g\in G.
	$$
	Moreover, we have that the covariance of the gradient is the covariance of the projection ``out'' of the tangent space:
	$$
		\E_\thetanull \C_{g}(\nabla \ell_\thetanull(gX)) = \E_\thetanull \C_{g}(P_G^\perp \nabla \ell_\thetanull(gX)).
	$$
\end{proposition}
\begin{proof}
	See Appendix \ref{subappend:proof_of_tangential_decomp}.
\end{proof}
The above proposition sends a clear message on the magnitude of $\E_\thetanull \C_{g}(\nabla \ell_\thetanull (gX))$:
\begin{quote}
	\it The larger the tangential part is, the less we gain from augmentation.
\end{quote}
At one extreme, if $\Theta_G$ contains an open neighborhood of $\thetanull$, then $P_G = Id$, so that $\E_\thetanull \C_G(\nabla \ell_\thetanull (gX))  = 0$. This means augmentation does not gain us anything. At another extreme, if $\Theta_G$ is a singleton, then $P_G^\perp = Id$, and so augmentation leads to a great variance reduction.\\

{\bf Orbit averaging and tangential projection.} We now have two operators that ``capture the invariance''. The first one is the orbit averaging operator, defined by $\E_G: \nabla \ell_\thetanull(x) \mapsto \E_{g\sim \mbQ} \nabla\ell_\thetanull(gx)$. The second one is the tangential projection operator $P_G$, which projects $\nabla \ell_\thetanull (x)$ to the tangent space of the manifold $\Theta_G$ at the point $\thetanull$. How does $\E_G$ relate to $P_G$? In the current notation, recall that the inner term of the asymptotic covariance $\Sigma_G$ can be written as $\C_{\thetanull} \E_G \nabla \ell_\thetanull$. Proposition \ref{prop:tangential_decomp} allows us to do the following decomposition:
\begin{align}
\label{eq:decomp_of_cov_along_manifold}
	\C_{\thetanull} \E_G \nabla \ell_\thetanull & = \C_\thetanull P_G \nabla \ell_\thetanull(X) + \C_{\thetanull} \E_G P_G^\perp \nabla \ell_\thetanull,
\end{align}
which relates the two operators. 
An interesting question now is whether the two operators are equivalent. A careful analysis shows that this is not true in general. 

Let us consider a special case, where $\theta = (\theta_1,\theta_2)$, and the invariant set $\Theta_G$ is characterized by $\theta_2=0$. In this special case, the tangent space is exactly $\{(\theta_1,\theta_2): \theta_2 = 0\}$, and so
\begin{equation}
\label{eq: proj_of_gradient_onto_tangent_space}
	P_G\nabla \ell_\thetanull = 
	\begin{bmatrix}
	\nabla_1 \ell_\thetanull \\
	0
	\end{bmatrix},
	\qquad 
	P_G^\perp\nabla \ell_\thetanull = 
	\begin{bmatrix}
	0 \\
	\nabla_2 \ell_\thetanull	
	\end{bmatrix},
\end{equation}	
where $\nabla_1$ and $\nabla_2$ denote the gradient w.r.t. the first and the second component of the parameter, respectively. Now, since $\E_G$ is a conditional expectation (see Lemma \ref{lemma:exact_invariance_lemma}), we know $\E_G\nabla \ell_\theta$ and $\nabla \ell_\theta -\E_G\nabla \ell_\theta$ are uncorrelated. Hence, if $\E_G=P_G$, then $\nabla_1 \ell_\thetanull$ and $\nabla_2 \ell_\thetanull$ are uncorrelated. If we write the Fisher information in a block matrix:
$$
	I_\thetanull = 
	\E_\thetanull \nabla \ell_\thetanull \nabla \ell_\thetanull^\top=
	\E_\thetanull 
	\begin{bmatrix}
	\nabla_{1} \ell_\thetanull \cdot \nabla_{1} \ell_\thetanull^\top &
	\nabla_{1} \ell_\thetanull \cdot \nabla_{2} \ell_\thetanull^\top \\
	\nabla_{2} \ell_\thetanull \cdot \nabla_{1} \ell_\thetanull^\top &
	\nabla_{2} \ell_\thetanull \cdot \nabla_{2} \ell_\thetanull^\top
	\end{bmatrix}=
	\begin{bmatrix}
	I_{11}(\thetanull) &
	I_{12}(\thetanull) \\
	I_{21}(\thetanull) &
	I_{22}(\thetanull)
	\end{bmatrix},
$$
then we would have $I_{12}(\thetanull)=0$. This shows that if $\E_G = P_G$, then the two parameter blocks are orthogonal, i.e., $I_{12}(\thetanull)=0$, which does not hold in general. \\

{\bf aMLE vs cMLE.} Recall that another method that exploits the invariance structure is the constrained MLE, defined by
$$
	\hthetaCMLE \in \underset{\theta \in \Theta_G}{\arg\max} \frac{1}{n} \sum_{i=1}^n \ell_{\theta}(X_i),
$$
where we seek the minimizer over the invariant subspace $\Theta_G$. How does it compare to the augmented MLE? If $\Theta_G$ is open, and if the true parameter belongs to the interior of the invariant subspace, $\thetanull \in \textnormal{int} \Theta_G $, then by an application of Theorem \ref{thm:asymp_normality_augmented_estimator} (with $\Theta$ replaced by $\Theta_G$) and Proposition \ref{prop:tangential_decomp}, it is clear that all three estimators -- $\hthetaMLE, \hthetaAMLE$, and $\hthetaCMLE$ -- share the same asymptotic behavior. 
At the other extreme, if $\Theta_G$ is a singleton, then the constrained MLE, provided it can be solved, gives the exact answer $\hthetaCMLE = \thetanull$. In comparison, the augmented MLE gains some efficiency but in general will not recover $\thetanull$ exactly. 
	
	What happens to constrained MLE when the manifold dimension of $\Theta_G$ is somewhere between $0$ and $p$? We provide an answer to this question in an simplified setup, where we again assume that $\theta = (\theta_1, \theta_2)$ and that $\Theta_G$ can be characterized by $\{(\theta_1, \theta_2): \theta_2 = 0\}$. We first recall the known behavior of cMLE in this case, which has asymptotic covariance matrix $(I_{11}(\thetanull))^{-1}$ \citep[see e.g.,][]{van1998asymptotic}. By an application of the Schur complement formula, we have
	$$
	(I_{11}(\thetanull))^{-1} \preceq [I_\thetanull^{-1}]_{11} =  [I_{11}(\thetanull)-I_{12}(\thetanull) (I_{22}(\thetanull))^{-1}I_{21}(\thetanull)]^{-1},
	$$
	which says that the cMLE is more efficient than the MLE--which has has asymptotic covariance matrix $I_\thetanull^{-1}$. A further computation gives the following result: 
	
\begin{proposition}[Relation between aMLE and cMLE in parametric models that decompose]
	\label{prop:relation_between_aMLE_and_cMLE}
	Suppose that the parameter has two blocks, $\theta = (\theta_1,\theta_2)$, and the constrained set  $\Theta_G$ is characterized by $\theta_2=0$.  Denote by $\bar I(\theta):=\C{\E_g \nabla \ell_\theta(gX)}$ the covariance of the average gradient. Define
	$$ 
		M_\theta = 
		\begin{bmatrix}
		(I_{11}(\theta))^{-1}& 
		(I_\theta^{-1})_{1\bigcdot}& \\
		(I_\theta^{-1})_{\bigcdot1}&
		(\bar{I}(\theta))^{-1}
		\end{bmatrix}
	$$
	where the notation $(I^{-1})_{1\bigcdot}$ refers to the submatrix of $I^{-1}$ corresponding to the rows indexed by the coordinates of $\theta_1$.
	Then the aMLE is asymptotically more efficient than the cMLE in estimating $\theta_1$ if and only if $M_\thetanull$ is p.s.d.
\end{proposition}
\begin{proof}
	See Appendix \ref{subappend:proof_of_relation_between_aMLE_and_cMLE}.
\end{proof}
In general, it seems that the two are not easy to compare, and the required condition would have to be checked on a case by case basis.

\subsection{Finite augmentations}\label{subsec:fin_aug}
The augmented estimator considered in above sections involves integrals over $G$, which are in general intractable if $G$ is infinite. As we have seen in Section \ref{sec:method}, in practice, one usually randomly samples some finite number of elements in $G$ at each iteration. Therefore, it is of interest to understand the behavior of the augmented estimator in the presence of such a ``finite approximation''. It turns out that we can repeat the above arguments with minimal changes. Specifically, given a function $f(x)$, we can define $\bar f_k(x) = k^{-1} \sum_{j=1}^k f(g_j x)$, where $g_j$ are arbitrary elements of $G$. Similarly to before, we find that the mean is preserved, while the variance is reduced, in the following way:
\benum
	\item $\E f(X) = \E \bar f_k(X)$;
	\item $\C f(X) = \C \bar f_k(X) + \E_{X\sim\P} \V_k[f(g_1X),\ldots,f(g_kX)]$, where $\V_k(a_1,\ldots,a_k)$ is the variance of the $k$ numbers $a_i$.
\eenum	

The above arguments suggest that for ERM problems, the efficiency gain is governed by the expected variance of $k$ vectors $\nabla L(\thetanull,g_i X)$, where $i=1, ..., k$. The general principle for practitioners is that the augmented estimator performs better if we choose $g_i$ to ``vary a lot" in the appropriate variance metric.


\section{Examples} \label{sec:ex}
    
In this section, we give several examples of models where exact invariance occurs. We characterize how much efficiency we can gain by doing data augmentation and compare it with various other estimators. Some examples are simple enough to give a finite-sample characterization, whereas others are calculated according to the asymptotic theory developed in the previous section.

\subsection{Exponential families} \label{subsec:exponential_family}
We start with exponential families, which are a fundamental class of models in statistics \citep[e.g.,][]{lehmann1998theory,lehmann2005testing}. Suppose $X\sim \P_\theta$ is distributed according to an exponential family, so that the log-likelihood can be written as
$$
    \ell_\theta(X)  = \theta^\top T(X) - A(\theta),
$$
where $T(X)$ is the sufficient statistic, $\theta$ is the natural parameter, $A(\theta)$ is the log-partition function. The densities of $\P_\theta$ are assumed to exist with respect to some common dominating $\sigma$-finite measure. Then the score function and the Fisher information is given by 
$$
    \nabla\ell_\theta(X)  = T(X) -\nabla A(\theta), \qquad I_\theta = \Cov{T(X)} = \nabla^2 A(\theta).
$$
    
Given invariance with respect to a group $G$, by Theorem \ref{thm:asymp_normality_augmented_estimator}, the asymptotic covariance matrix of the aMLE equals $I_\theta^{-1} J_\theta I_\theta^{-1}$, where $J_\theta$ is the covariance of the orbit-averaged sufficient statistic $J_\theta  = \C_X \E_g T(gX).$
    
Assuming a linear action by the group $G$, by Equation \eqref{eq:relation_between_log_likelihood_of_x_and_gx}, the invariant parameter space $\Theta_G$ consists of those parameters $\theta$ for which 
$$
\theta^\top [T(gx)- T(x)] + v(g) = 0,\,\,\,\, \forall g,x,
$$
where $v(g)  = \log |\det g|$ is the log-determinant. This is a set of linear equations in $\theta$. Moreover, the log-likelihood is concave, and hence the constrained MLE estimator can in principle be computed as the solution to the following convex optimization problem:
\begin{align*}
    \hthetaCMLE &\in\underset{\theta}{\arg\max}\,\,\,\, \theta^\top T(X) - A(\theta)\\
    & s.t.\,\,\,\, \theta^\top [T(gx)- T(x)] + v(g) = 0,\,\,\,\, \forall g\in G, x\in \xx.
\end{align*}
Assume that $\Theta = \R^p$, so that the exponential family is well defined for all natural parameters, and that $\nabla A$ is invertible on the range of $\E_g T(gX)$. The KKT conditions of the above convex program is given by
\begin{align*}
    \hthetaCMLE &\in [\nabla  A]^{-1}(T(X)+ 
    \textnormal{span}\{T(gz)-T(z):z\in\R^d,g\in G\})\\
    & s.t.\,\,\,\, \theta^\top [T(gx)- T(x)] + v(g) = 0,\,\,\,\, \forall g,x.
\end{align*}

Meanwhile, augmented MLE is the solution of the optimization problem where we replace the sufficient statistic $T(x)$ by $\bar T(x) = \E_g T(gx)$:
\begin{align*}
    \hthetaAMLE \in\arg&\max_\theta\,\,\,\, \theta^\top \E_g T(gX) - A(\theta).
\end{align*}
We then have $\hthetaAMLE =[\nabla  A]^{-1}\E_g T(gX)$. Therefore, for exponential families we were able to give more concrete expressions for the augmented and constrained MLEs. \\
    
{\bf Gaussian mean.} Consider now the important special case of Gaussian mean estimation. Suppose that $X$ is a standard Gaussian random variable, so that $A(\theta) = \|\theta\|^2/2$, and $T(x) = x$.
Assume for simplicity that $G$ acts orthogonally. Then we have $\Theta_G = \{ v: g^\top v = v, \forall g\in G\}$. Recall that maximizing the Gaussian likelihood is equivalent to minimizing the distance $\|\theta - X\|_2^2$. Hence, the constrained MLE, by definition of the projection, takes the following form:
$$
    \hthetaCMLE = P_G X,
$$
where we recall that $P_G$ is the orthogonal projection operator onto the tangent space of $\Theta_G$ at $\theta$. However, since $\Theta_G$ is a linear space in our case, $P_G$ is simply the orthogonal projection operator onto $\Theta_G$. On the other hand, we have
$$
    \hthetaAMLE = \E_{g\sim \mbQ} [g] X.
$$
In fact, under the current setup, the augmented MLE equals the constrained MLE:
\begin{proposition}
    \label{prop:risk_of_gaussian_mean_estimation}
    Assume $G$ acts linearly and orthogonally.
    If $X$ is $d$-dimensional standard Gaussian, then $P_G = \E_{g\sim\mbQ}[g]$, so that both the aMLE and cMLE are equal to the projection onto the invariant subspace $\Theta_G$. In particular, their risk equals $\dim \Theta_G$.
\end{proposition}
\begin{proof}
    See Appendix \ref{subappend:proof_of_risk_of_gaussian_mean_estimation}.
\end{proof}

For instance, suppose $G= (\{\pm 1\}, \cdot)$ acting by flipping the signs. Then it is clear that $\Theta_G = \{0\}$, and so both the cMLE and aMLE are identically equal to zero. 
    
On the other hand, the marginal MLE \eqref{eq:def_of_marginal_mle} is a different object, even in the one-dimensional case. Suppose that $X\sim\N(\theta,1)$, and we consider the reflection group $G= \{1,-1\}$. The marginal distribution of the data is a Gaussian mixture model
    $$Z\sim \frac12[\N(\theta,1)+\N(-\theta,1)].$$
    So the mMLE fits a mixture model, solving
    \begin{align*}
    \hthetaMMLE    
    & \in \arg \max_\theta \sum_{i\in[n]} 
    \log\left(\int_G p_\theta(gX_i) d\mbQ(g)\right)\\
    & = \arg \max_\theta \sum_{i\in[n]} 
    \log\left(\frac12[p_\theta(X_i)+p_\theta(-X_i)]\right)\\
    & = \arg \max_\theta \sum_{i\in[n]} 
    \log\left(\exp[-(X_i-\theta)^2/2]+\exp[-(-X_i-\theta)^2/2]\right)
    \end{align*}
    The solution to this is not necessarily identically equal to zero, and in particular it does not agree with the cMLE and aMLE. \\
    
{\bf Numerical results.} We present some numerical results to support our theory. We consider $X\sim \N(\mu, I_d)$, and invariance to the reflection group \emph{that reverses the order of the vector}. This is a stylized model of invariance, which occurs for instance in objects like faces.
\begin{figure}[bt]
    \centering
    \begin{subfigure}
        \centering
        \includegraphics[scale=0.23]
        {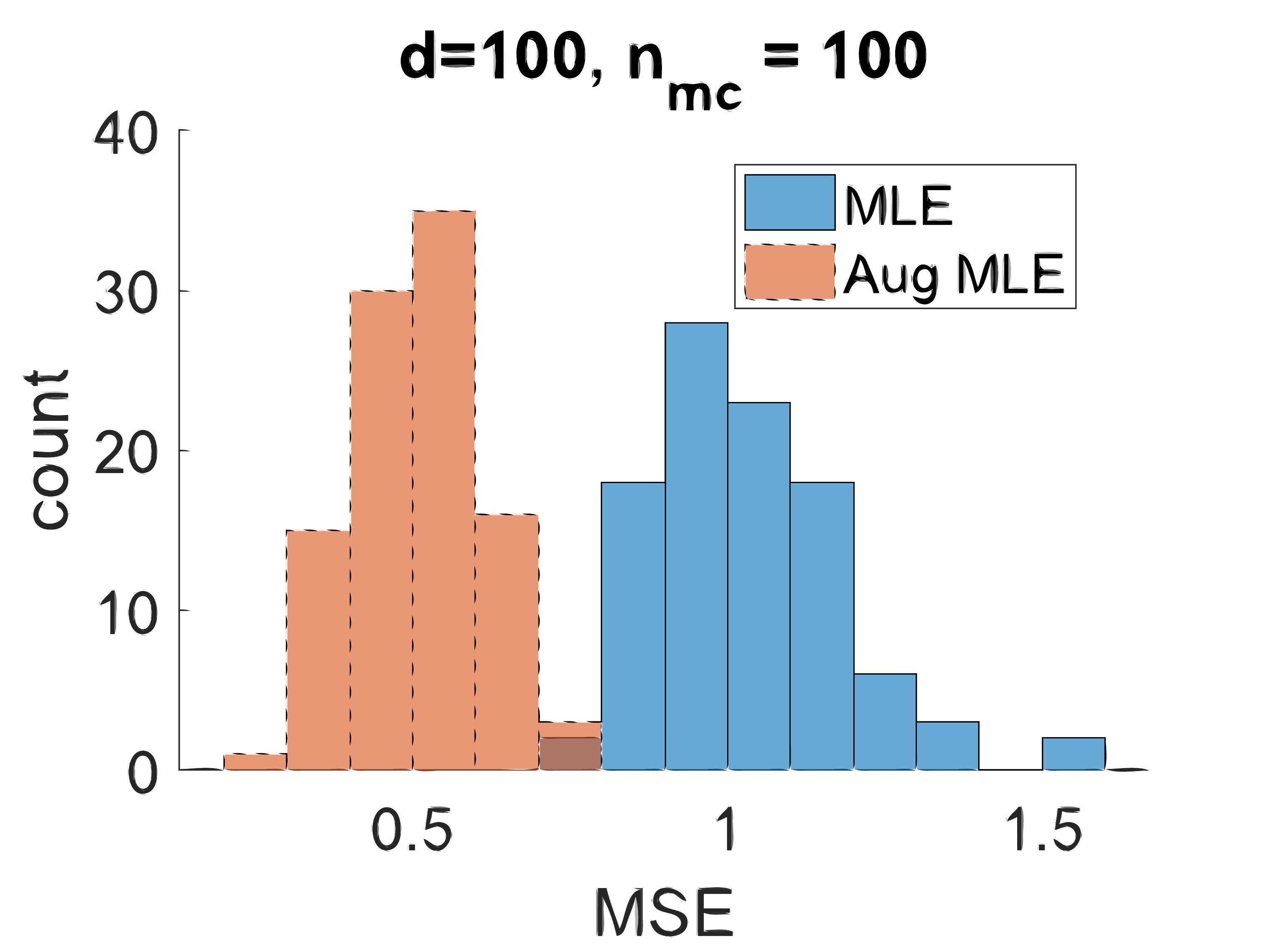}
    \end{subfigure}
    \begin{subfigure}
        \centering
        \includegraphics[scale=0.23]
        {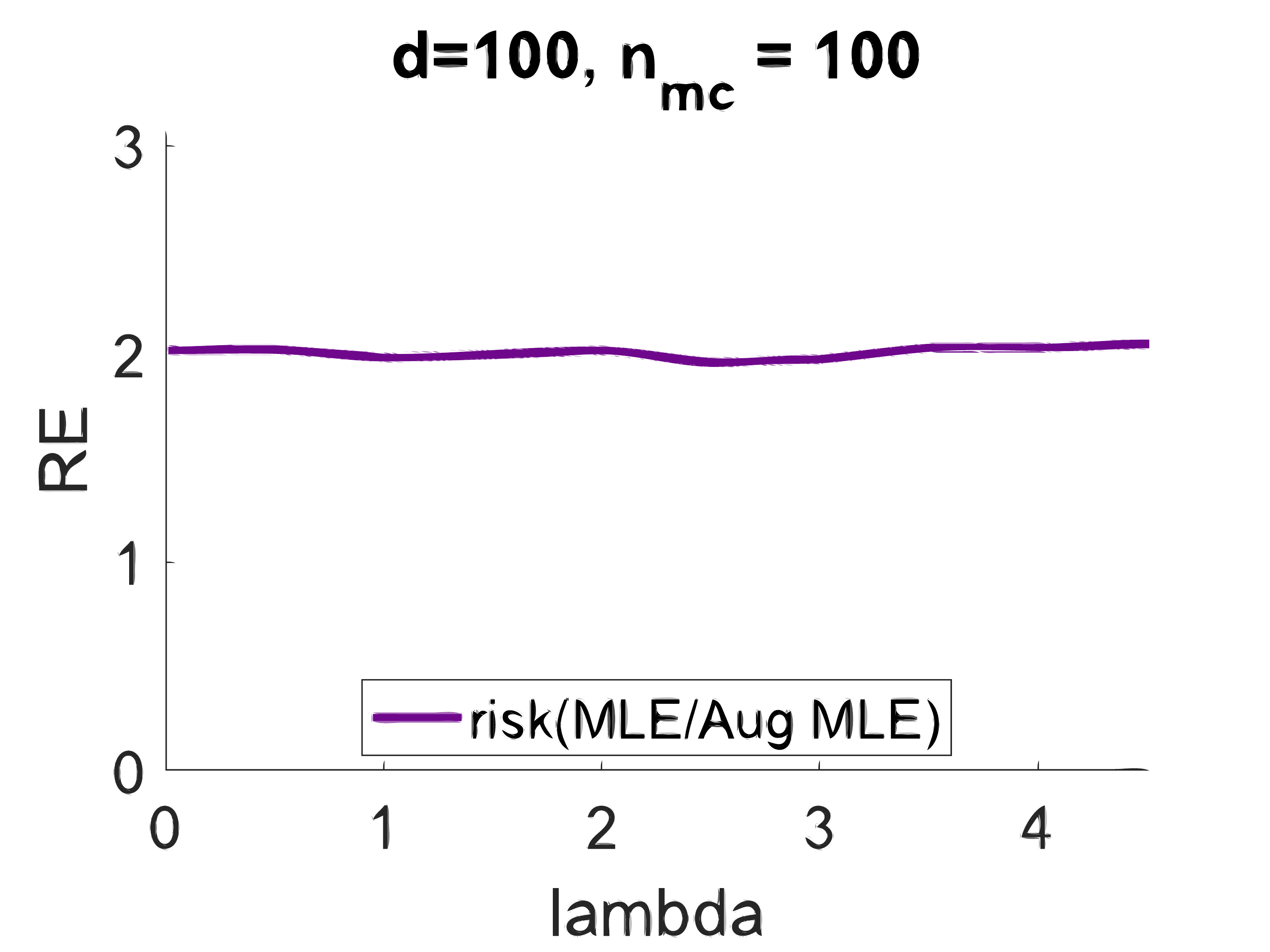}
    \end{subfigure}
    \caption{Plots of the increase in efficiency achieved by data augmentation in a \emph{flip symmetry} model.}
    \label{fig:flip}
\end{figure}
    
In Figure \ref{fig:flip}, we show the results of two experiments. On the left figure, we show the histograms of the mean squared errors (normalized by dimension) of the MLE and the augmented MLE on a $d=100$ dimensional Gaussian problem. We repeat the experiment $n_{MC}=100$ times. We see that the MLE has average MSE roughly equal to unity, while the augmented MLE has average MSE roughly equal to one half. Thus, data augmentation reduces the MSE two-fold. This confirms our theory. 
    
On the right figure, we change the model to each coordinate $X_i$ of $X$ being sampled independently as $X_i\sim Poisson(\lambda)$. We show that the relative efficiency (the relative decrease in MSE) of the MLE and the augmented MLE is roughly equal to two regardless of $\lambda$. This again confirms our theory.

\subsection{Parametric regression models}\label{subsec:param_regression}
We consider a regression problem where we observe an iid random sample $\{(X_1, Y_1), \ldots,$ $(X_n, Y_n)\}\subseteq \R^d \times \R$ from the law of a random vector $(X , Y)$. This follows the model:
$$
    Y = f(\thetanull,X) + \ep, \qquad \ep \indep X, \qquad \E \ep = 0,  
$$
where $\thetanull\in\R^p$.
We have a group $G$ acting on $\R^d \times \R$ \emph{only through} $X$: 
$$
    g(X, Y) = (gX, Y) ,
$$    
and the invariance is characterized by
$$
    (gX, Y) =_d (X, Y). 
$$
In regression or classification problems in deep learning, we typically apply the augmentations conditionally on the outcome or class label. This corresponds to the conditional invariance 
$$
	(gX|Y=y) =_d (X|Y=y).
$$
The conditional invariance can be deduced from $(gX, Y) =_d (X, Y)$ by conditioning on $Y=y$. Conversely, if it holds conditionally for each $y$, we deduce that it also holds jointly, i.e., $(gX,Y) =_d (X,Y)$. Thus, conditional and joint invariance are equivalent. Moreover, each of them clearly implies marginal invariance of the features, i.e., $gX=_dX$. 
By conditional invariance, $\E(Y|X=x) = \E(Y|X=gx)$, hence
\begin{equation}
	\label{eq:invariance_of_regression_function}
	f(\thetanull, g x)  = f(\thetanull, x)
\end{equation}
for every non-random $x$ and any $g\in G$.
This is what we would expect in the applications we have in mind: the label should be preserved provided there is no random error.

\subsubsection{Non-linear least squares regression}
We now focus on the least squares loss:
\begin{equation}
	\label{eq:least_square_loss_for_regression}
	L(\theta,X, Y) = (Y - f(\theta,X))^2.
\end{equation}
The population risk is
$$
	\E L_\theta (X, Y) = \E (Y - f(\thetanull,X)  + f(\thetanull,X) - f(\theta,X))^2 = \E (f(\thetanull,X) - f(\theta,X))^2 + \gamma^2,
$$
where $\gamma^2$ is the variance of the error $\ep$. Under standard assumptions, the minimizer $\htheta_{ERM}$ of $\theta \mapsto \sum_{i=1}^n L(\theta,X_i, Y_i)$ is consistent (see e.g., Example 5.27 of \citealt{van1998asymptotic}). Similarly, under standard smoothness conditions, we have 
$$
	\E L(\theta,X, Y) = \underbrace{\gamma^2}_{= \E L(\thetanull,X, Y)} + \frac{1}{2}\E\bigg[ (\theta - \thetanull)^\top  \bigg(2 \nabla f(\thetanull,X) \nabla f(\thetanull,X)^\top \bigg) (\theta-\thetanull)\bigg] + o(\|\theta - \thetanull\|^2),
$$
where $\nabla f(\theta,X)$ is the gradient w.r.t. $\theta$.
This suggests that we can apply Theorem \ref{thm:asymp_normality_augmented_estimator} with $V_\thetanull = 2 \E \nabla f(\thetanull,X) \nabla f(\thetanull,X)^\top $ and $\nabla L(\thetanull,X, Y) = -2(Y - f(\thetanull,X))\nabla f(\thetanull,X) = -2\ep \nabla f(\thetanull,X)$, which gives  (with the Fisher information $I_\theta = \E \nabla f(\theta,X)\nabla f(\theta,X)^\top $)
\begin{equation}
	\label{eq:asymp_normality_of_nonlinear_regression}
	\sqrt{n}(\hthetaERM - \thetanull) \Rightarrow \N\bigg(0, V_\thetanull^{-1}\E \bigg[4 \ep^2  \nabla f(\thetanull,X)\nabla f(\thetanull,X)^\top \bigg] V_\thetanull^{-1} \bigg) =_d \N\bigg(0 , \gamma^2 I_\thetanull^{-1}\bigg).
\end{equation}

On the other hand, the augmented ERM estimator is the minimizer $\hthetaAERM$ of $\theta \mapsto \sum_{i=1}^n$ $\E_g L(\theta,gX_i, Y_i)$. Now applying Theorem \ref{thm:asymp_normality_augmented_estimator} gives  
	\begin{equation}
	\label{eq:asymp_normality_of_augmented_nonlinear_regression}
	\sqrt{n}(\hthetaAERM - \thetanull) \Rightarrow \N(0, \Sigma_{\textnormal{aERM}}),
	\end{equation}
	with the asymptotic covariance being
	\begin{align*}
	\Sigma_{\textnormal{aERM}} & 
	= \gamma^2 I_\thetanull^{-1}
	- V_\thetanull^{-1}\E \bigg[\C_g \nabla L(\thetanull,gX, Y)\bigg] V_\thetanull^{-1} \\
	& 
	= \gamma^2 I_\thetanull^{-1}
	- I_\thetanull^{-1} \E \bigg[ \C_g (Y - f(\thetanull,gX)) \nabla f(\thetanull,gX) \bigg] I_\thetanull^{-1} \\
	& 
	= \gamma^2 I_\thetanull^{-1}
	-  I_\thetanull^{-1} \E \bigg[  \ep^2  \C_g \nabla f(\thetanull,gX) \bigg] I_\thetanull^{-1}\\
	& 
	= \gamma^2 \cdot \left( I_\thetanull^{-1}
	-  I_\thetanull^{-1} \E \bigg[  \C_g \nabla f(\thetanull,gX) \bigg] I_\thetanull^{-1}\right), 
	\end{align*}
	where we used $f(\thetanull,gx) = f(\thetanull,x)$ in the second to last line.
	Using Lemma \ref{lemma:exact_invariance_lemma}, we can also write
	\begin{align*}
	\Sigma_{\textnormal{aERM}} & = \gamma^2 I_\thetanull^{-1} \bigg(I_\thetanull - \E \bigg[  \C_G \nabla f(\thetanull,gX) \bigg]\bigg) I_\thetanull^{-1} \\
	& = \gamma^2 I_\thetanull^{-1}\bigg( \C_X \nabla f(\thetanull, X) - \E \bigg[  \C_G \nabla f(\thetanull,gX) \bigg]  \bigg)  I_\thetanull^{-1} \\
	& = \gamma^2 	I_\thetanull^{-1} \bar I_\thetanull  I_\thetanull^{-1},
	\end{align*}
	where $\bar{I}_\thetanull$ is the ``averaged Fisher information'', defined as
	$$
	\bar I_\thetanull =  \C_X[\E_g \nabla f(\thetanull, gX)] = \E_X \bigg[\bigg(\E_g \nabla f(\thetanull, gX)\bigg) \bigg(\E_g \nabla f(\thetanull, gX)\bigg)^\top\bigg]. 
	$$

\subsubsection{Two-layer neural network for regression tasks}
As an example, consider a two-layer neural network
$$
	f(\theta,x) = a^\top \sigma(Wx).
$$
Here $x$ is a $d$-dimensional input, $W$ is a $p\times d$ weight matrix, $\sigma$ is a nonlinearity applied elementwise to the preactivations $Wx$. 
The overall parameters are $\theta = (a, W)$. For simplicity, let us focus on the case where $a = 1_p$ is the all ones vector. This will simplify the expressions for the gradient. We can then write 
\begin{equation}
	\label{eq:simplified_expression_for_2lnn}
	f(W,x) = 1^\top \sigma(Wx).
\end{equation}

	Let $I_W = \E \nabla f(W,X)\nabla f(W,X)^\top $ be the fisher information, and denote its averaged version as 
	$$
	\bar I_W = \E_X \bigg[\bigg(\E_g \nabla f(W, gX)\bigg) \bigg(\E_g \nabla f(W, gX)\bigg)^\top\bigg].
	$$ Recall Equation \eqref{eq:asymp_normality_of_nonlinear_regression} and \eqref{eq:asymp_normality_of_augmented_nonlinear_regression}:
	\begin{align*}
	\sqrt{n}(\hat W_{\textnormal{ERM}} - W) &
	\Rightarrow \N\bigg(0 , \sigma^2 I_W^{-1}\bigg)\\
	\sqrt{n}(\hat W_{\textnormal{aERM}} - W) & 
	\Rightarrow \N\bigg(0 , \sigma^2  I_W^{-1} \bar I_W I_W^{-1}\bigg).
	\end{align*}
	Therefore, the efficiency of the ERM and aERM estimators is determined by the magnitude of $I_W$ and $\bar I_W$. In general, those are difficult to calculate. However, an exact calculation is possible under a natural example of \emph{translation invariance}. Let the group $G = \{g_0, g_1,..., g_{p-1}\}$, where $g_i$ acts by shifting a vector circularly by $i$ units:
	\begin{equation}
	\label{eq:circular_shift}
	(g_i x)_{j+i \textnormal{ mod } p} = x_{j}.
	\end{equation}
	We equip $G$ with a uniform distribution. Under such a setup, we have the following result:

\begin{theorem}[Circular shift data augmentation in two-layer networks]
		\label{thm:2lnn}
		Consider the two-layer neural network model \eqref{eq:simplified_expression_for_2lnn} trained using the least squared loss \eqref{eq:least_square_loss_for_regression}. Then:
		\begin{enumerate}
			\item The Fisher information matrix $I_W$ can be viewed as a tensor
			\begin{align*}
			I_W 
			& = \E (\sigma'(WX)\otimes \sigma'(WX)) \cdot (X \otimes X)^\top.
			\end{align*}
			If, furthermore, the activation function is quadratic, $\eta(x) = x^2/2$, then $I_W$ can be written as the product of a $p^2\times d^2$ tensor and a $d^2\times d^2$ tensor:
			\begin{align*}
			I_W
			& = (W \otimes W) \cdot  \E (XX^\top \otimes XX^\top).
			\end{align*}
			\item Assume the activation function is quadratic. 
			Let $C_v$ be the circulant matrix associated with the vector $v$, with entries $C_v(i,j) = v_{i-j+1}.$ Then $\bar I_W$ can be written as
			\begin{align*}
			\bar I_W  
			& = (W \otimes W) \cdot d^{-2} \E (C_X C_X^\top \otimes C_X C_X^\top).
			\end{align*}
			If furthermore the distribution of $X$ is normal, $X\sim \N(0,I_d)$, then we have
			\begin{align*}
			\bar I_W  
			& = (W \otimes W) \cdot F_2^*\cdot (F_2^2 \odot M) \cdot F_2^{*}.
			\end{align*}
			Here $F_2 = F\otimes F$, where $F$ is the $d\times d$ DFT matrix, $M$ is the $d^2\times d^2$ tensor with entries
			\begin{align*}
			M_{iji'j'}
			& = 
			F_i^\top F_j  \cdot F_{i'}^\top F_{j'}
			+
			F_i^\top  F_{j'} \cdot F_{i'}^\top F_j
			+
			F_i^\top  F_{i'} \cdot F_i^\top F_{j'},
			\end{align*}
			and $F_2^*$ is the complex conjugate of $F_2$.
		\end{enumerate}
	\end{theorem}
	\begin{proof}
		See Appedix \ref{subappend:proof_of_2lnn}.
	\end{proof}
	
	This theorem shows in a precise quantitative sense how much we gain from data augmentation in a two-layer neural network. To get a sense of the magnitude of improvement, we will attempt to understand how much ``smaller" $\bar I_W$ is compared to $I_W$ by calculating their MSEs. For simplicity, we will impose a prior $W \sim \N(0, I_p\otimes I_d)$, and the MSE is calculated by taking expectation w.r.t. this prior. First we have $\E\tr I_W = \E \|WXX^\top\|^2_{Fr}$. Let $S = XX^\top$. Now  $W \sim \N(0,I_p \otimes I_d)$. So conditional on $X$, we have $WS\sim  \N(0,I_p \otimes S^2)$. Hence, $\E\|WS\|^2_{Fr} = p \E \tr S^2 = p \E \tr (XX^\top)^2$. Similarly, we find $\E\tr \bar I_W = p \E \tr (C_XC_X^\top)^2/d^2$. \\
	
	\begin{figure}[bt]
		\centering
		\includegraphics[scale=0.25]
		{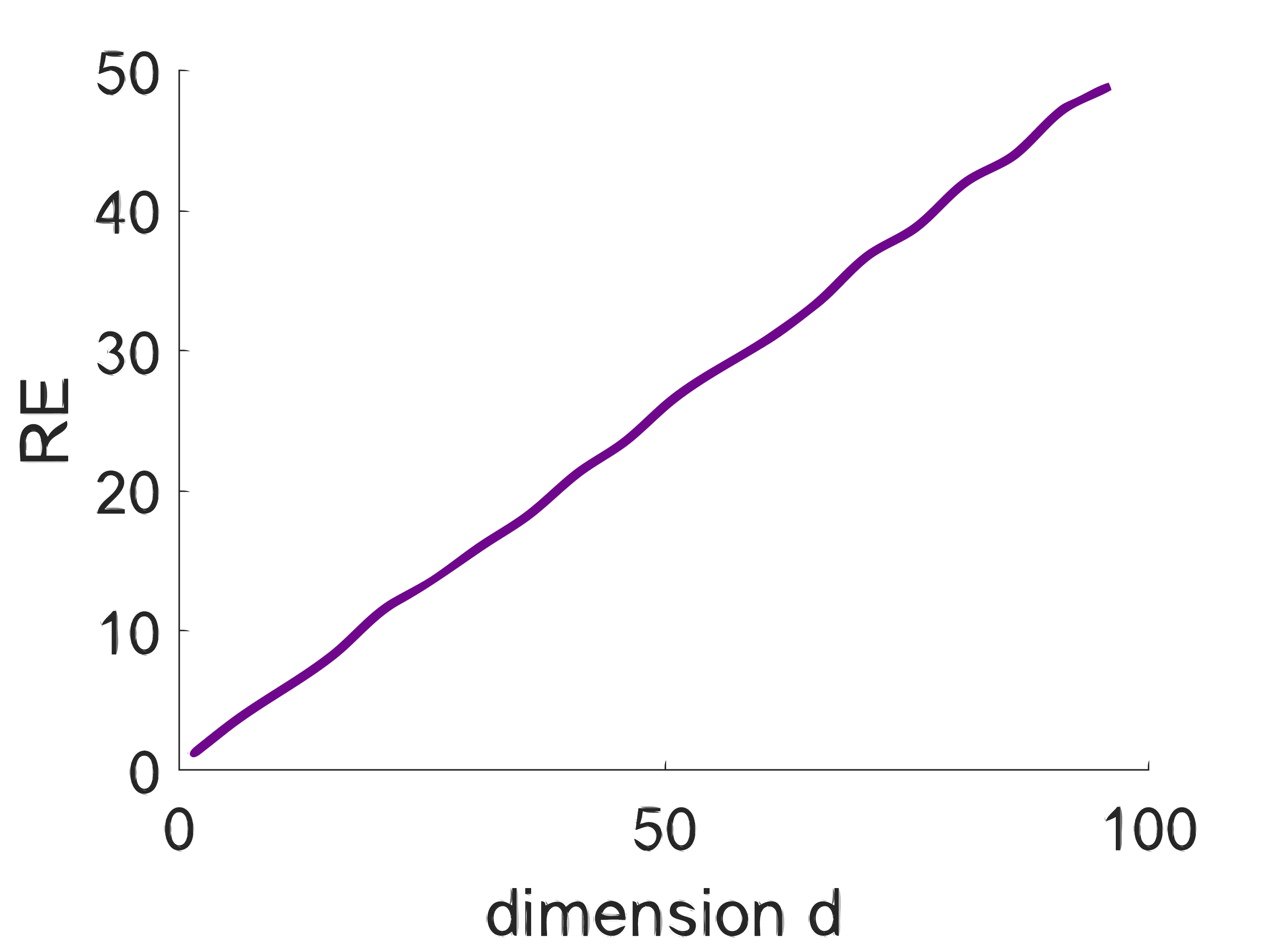}
		\caption{Plot of the increase in efficiency achieved by data augmentation in a \emph{circular symmetry} model.}
		\label{circ}
	\end{figure}
	
	{\bf Numerical results.}
	In Figure \ref{circ}, we show the results of an experiment where we randomly generate the input as $X \sim \N(0,I_d)$. For a fixed $X$, we compute the values of $\E\tr I_W = p\tr (XX^\top)^2$ and $\E\tr \bar I_W = p \tr (C_XC_X^\top)^2/d^2$, and record their ratio. We repeat the experiment $n_{MC}=100$ times. We then show the relative efficiency of aMLE with respect to MLE as a function of the input dimension $d$. We find that the relative efficiency scales as $RE(d) \sim d/2$. Thus, for the efficiency gain increases as a function of the input dimension. However, the efficiency gain does not depend on the output dimension $p$. This makes sense, as circular invariance affects and reduces only the input dimension. 
	

	\subsection{Parametric classification models} \label{subsec:param_classification}
	Similar calculations as in the previous subsection carry over to the classification problems. We now have a random sample $\{(X_1, Y_1),$ $..., (X_n, Y_n)\}$ $\subseteq \R^d \times \{0, 1\}$ from the law of a random vector $(X, Y)$, which follows the model: 
	$$
	\P(Y = 1 \ | \ X) = \eta(f(\thetanull,X)),
	$$
	where $\thetanull\in \R^p$, $\eta: \R \to [0, 1]$ is an increasing activation function, and $f(\thetanull,\cdot)$ is a real-valued function. For example, the sigmoid $\eta(x) = 1/(1+e^{-x})$ gives the logistic regression model, using features extracted by $f(\thetanull,\cdot)$. As in the regression case, we have a group $G$ acting on $\R^d \times \{0, 1\}$ via
	$$
	g(X, Y) = (gX, Y),
	$$
	and the invariance is  
	$$
	(gX, Y) =_d (X, Y).
	$$
	
	The interpretation of the invariance relation is again two-fold. On the one hand, we have $gX =_d X$. On the other hand, for almost every (w.r.t. the law of $X$) $x$, we have
	$$
	\P (Y=1 \ | \ gX = x) = \P (Y = 1 \ | \ X = x).
	$$
	The LHS is   $\eta(f(\thetanull,g^{-1}x))$, whereas the RHS is $\eta(f(\thetanull,x))$. This shows that for any (non-random) $g\in G$ and $x$, we have
	$$
	\eta(f(\thetanull,gx)) = \eta(f(\thetanull,x)).
	$$
	For image classification, the invariance relation says that the class probabilities stay the same if we transform the image by the group action. Moreover, since we assume $\eta$  is monotonically strictly increasing, applying its inverse gives $$f(\thetanull,gx) = f(\thetanull,x).$$

	\subsubsection{Non-linear least squares classification}
	We consider using the least square loss to train the classifier:
	\begin{equation}
	\label{eq:least_square_loss_for_classification}	
	L(\theta,X, Y) = (Y - \eta(f(\theta,X)))^2.
	\end{equation}
	Though this is not the most popular loss, in some cases it can be empirically superior to the default choices, e.g., logistic loss and hinge loss \citep{wu2007robust,nguyen2013algorithms}. The loss function has a bias-variance decomposition:
	\begin{align*}
	\E L(\theta,X, Y) & = \E [Y - \eta(f(\thetanull,X)) + \eta(f(\thetanull,X)) - \eta(f(\theta,X))]^2 \\	
	& = \underbrace{\E [Y - \eta(f(\thetanull,X))]^2}_{\E L(\thetanull,X, Y)} + \E [\eta(f(\thetanull,X)) - \eta(f(\theta,X)) ]^2,
	\end{align*}
	where the cross-term vanishes because $\eta(f(\thetanull,X)) = \E [ Y  |   X ]$. Note that 
	\begin{align*}
	& \E[Y - \eta(f(\thetanull,X))]^2 = \E [(Y - \E[Y|X])^2] 
	 = \E\bigg[\E[ (Y - \E[Y|X])^2\ | \ X]\bigg] \\
	& = \E \V(Y   |   X) 
	 = \E \V [\text{Bernoulli}(\eta(f(\thetanull,X)))] 
	 = \E \eta(f(\thetanull,X))(1 - \eta(f(\thetanull,X))).
	\end{align*}
	Meanwhile, since $\nabla \eta(f(\theta,X)) = \eta'(f(\theta,X)) \nabla f(\theta,X)$, for sufficiently smooth $\eta$, we have a second-order expansion of the population risk:
	$$
	\E L(\theta,X , Y) = \E L(\thetanull,X, Y) + \frac{1}{2} (\theta - \thetanull)^\top \E [2 \eta'(f(\thetanull,X))^2 \nabla f(\thetanull,X) \nabla f(\thetanull,X)^\top] (\theta - \thetanull) + o(\|\theta- \thetanull \|^2).
	$$
	This suggests that we can apply Theorem \ref{thm:asymp_normality_augmented_estimator} with $V_\thetanull = \E [2 \eta'(f(\thetanull,X))^2 \nabla f(\thetanull,X) \nabla f(\thetanull,X)^\top]$ and $\nabla L(\theta,X, Y) = -2(Y - \eta(f(\theta,X))) \eta'(f(\theta,X)) \nabla f(\theta,X) $, which gives
	\begin{equation}
	\label{eq:asymp_normality_of_nonlinear_classification}
	\sqrt{n}(\hthetaERM  - \thetanull) \Rightarrow \N(0, \SigmaERM),
	\end{equation}
	where the asymptotic covariance is
	\begin{align}
	\label{eq:asymp_cov_of_nonlinear_classification}
	\Sigma_{\textnormal{ERM}} & = \E [U_\thetanull(X)]^{-1}\E[v_\thetanull(X) U_\thetanull(X)] \E [U_\thetanull(X)]^{-1}  \\
	v_\thetanull(X) & = \eta(f(\thetanull,X))
	\cdot \left(1 - \eta(f(\thetanull,X))\right) \nonumber\\
	U_\thetanull(X) & = \eta'(f(\thetanull,X))^2 \nabla f(\thetanull,X) \nabla f(\thetanull,X)^\top. \nonumber
	\end{align}
	Here $v_\thetanull(X)$ can be viewed as the noise level, which corresponds $\E\ep^2$ in the regression case. Also, $U_\thetanull(X)$ is the information, which corresponds to $\E \nabla f(\thetanull,X)\nabla f(\thetanull,X)^\top$ in the regression case. The classification problem is a bit more involved, because the noise and the information do not decouple (they both depend on $X$). In a sense, the asymptotics of classification correspond to a regression problem with heteroskedastic noise, whose variance depends on the mean signal level.
	
	In contrast, applying Theorem \ref{thm:asymp_normality_augmented_estimator} for the augmented loss gives
	\begin{equation}
	\label{eq:asymp_normality_of_augmented_nonlinear_classification}
	\sqrt{n}(\hthetaAERM - \thetanull) \Rightarrow \N(0, \SigmaAERM),
	\end{equation}
	where
	\begin{equation*}
	\SigmaERM - \SigmaERM = V_\thetanull^{-1} \E \C_g \nabla L(\thetanull,gX) V_\thetanull^{-1}.
	\end{equation*}
	We now compute the gain in efficiency:
	\begin{align*}
	\E\C_g\nabla L(\thetanull,gX) & = \E \C_g \bigg(2(Y - \eta(f(\thetanull,gX))) \eta'(f(\thetanull,gX)) \nabla f(\thetanull,gX)\bigg) \\
	& = 4 \E \bigg[(Y - \eta(f(\thetanull,X)))^2 \C_g \bigg(\eta'(f(\thetanull,gX)) \nabla f(\thetanull,gX)\bigg) \bigg] \\
	& = 4 \E \bigg[ v_\thetanull(X) \C_g \bigg(\eta'(f(\thetanull,gX)) \nabla f(\thetanull,gX)\bigg)\bigg]. 
	\end{align*}
	In summary, the covariance of ERM is larger than the covariance of augmented ERM by
	\begin{equation}
	\label{eq:asymp_cov_of_augmented_nonlinear_classification}
	\SigmaERM - \SigmaAERM = \E[U_\thetanull(X)]^{-1} \E \bigg[ v_\thetanull(X) \C_g \bigg(\eta'(f(\thetanull,gX)) \nabla f(\thetanull,gX)\bigg)\bigg] \E[U_\thetanull(X)]^{-1}.
	\end{equation}

	\subsubsection{Two-layer neural network for classification tasks}
	We consider the two-layer neural network \eqref{eq:simplified_expression_for_2lnn} here. Most of the computations for the regression case carry over to the classification case. Recall that in the current setup, the model \eqref{eq:simplified_expression_for_2lnn} becomes   
	\begin{equation}
	\label{eq:simplified_expression_for_2lnn_for_classification}
	\P(Y = 1 \ |  \ X)  = \eta(f(W,X)):= \eta(1^\top\sigma(WX)),
	\end{equation}
	where $W \in \R^{p\times d}$ and $\eta$ is a nonlinearity applied elementwise. We then have an analog of Theorem \ref{thm:2lnn}: 

	\begin{corollary}
	\label{cor:2lnn_for_classification}
	Consider the two-layer neural network model \eqref{eq:simplified_expression_for_2lnn} trained using the least squares loss \eqref{eq:least_square_loss_for_classification}. Assume the activation function is quadratic: $\sigma(x) = x^2/2$. Let $U_W, v_W$ be defined as in \eqref{eq:asymp_cov_of_nonlinear_classification}. Then
	$$
	\SigmaERM = \E[U_W(X)]^{-1} \E[v_W(X) U_X(X)] \E[U_W(X)]^{-1}, 
	$$
	where
	\begin{align*}
	\E U_W(X) & = (W\otimes W) \cdot \E [\eta'(f(W,X))^2  \cdot  XX^\top \otimes XX^\top] \\
	\E [v_W(X) U_W(X)] & = (W\otimes W) \E[v_W(X) \eta'(f(W,X))^2 \cdot  XX^\top \otimes XX^\top].
	\end{align*}
	If we further assume that the group $G$ acts by cirlular shift \eqref{eq:circular_shift} and $\mbQ$ is the uniform distribution on $G$, then
	$$
	\E U_W(X) \SigmaAERM \E U_W(X) = (W\otimes W) \cdot d^{-2} \E[v_W(X) \eta'(f(W, X))^2 \cdot C_X C_X^\top \otimes C_X C_X^\top].
	$$
	Thus, the gain by augmentation is characterized by
	\begin{align*}
	& \E[U_W(X)](\SigmaERM - \SigmaAERM) \E[U_W(X)] \\
	& = (W\otimes W) \E \bigg[ v_W(X)\eta'(f(W,X))^2  \bigg( XX^\top\otimes XX^\top - C_X C_X^\top\otimes C_X C_X^\top \bigg) \bigg].
	\end{align*}
	\end{corollary}
	\begin{proof}
		See Appendix \ref{subappend:proof_of_2lnn_for_classification}.
	\end{proof}

	\subsection{Improving linear regression by augmentation distribution} \label{subsec:lin_reg_aug_distr}
	
	In this section, we provide an example on how to use the ``augmentation distribution'', described in Section \ref{sec:bey_erm}, to improve the performance of ordinary least squares estimator in linear regression. We will see that, perhaps unexpectedly, the idea of augmentation distribution gives an estimator nearly as efficient as the constrained ERM in many cases.

	We consider the classical linear regression model
	$$
	Y = X^\top \beta + \ep, \qquad \beta \in \RR^p.
	$$
	We again let $\gamma^2 = \E \ep^2$.
	We will assume that the action is linear, so that $g$ can be represented as a $p\times p$ matrix. If we augment by a single fixed $g$, we get
	\begin{align*}
	 \arg\min  \|y - Xg^\top \beta \|_2^2  =  ((Xg^\top)^\top Xg^\top)^{-1} (Xg^\top)^\top y.
	\end{align*}
	Following the ideas on augmentation distribution, we can then average the above estimator over $g \sim \mbQ$ to obtain
	$$
		\hbeta_{\textnormal{aDIST}} = \E_{g\sim \mbQ} \bigg[ ((Xg^\top)^\top Xg^\top)^{-1} (Xg^\top)^\top y \bigg].
	$$

	On the other hand, let us consider the estimator arising from constrained ERM. 
	By Equation \eqref{eq:invariance_of_regression_function}, we have
	$$
	x^\top \beta = (gx)^\top \beta
	$$
	for $\mbP_X$-a.e. $x$ and $\mbQ$-a.e. $g$. This is a set of \emph{linear constraints} on the regression coefficient $\beta$. Formally, supposing that $x$ can take any value (i.e., $\P_X$ has mass on the entire $\RR^p$), we conclude that $\beta$ is constrained to be in the invariant parameter subspace $\Theta_G$, which is a linear subspace defined by 
	$$\Theta_G = \{ v: g^\top v = v, \forall g\in G\}.$$
	If $x$ can only take values in a smaller subset of $\R^p$, then we get fewer constraints. So the constrained ERM, defined as
	$$
		\hbetaCERM  =  \underset{\beta}{\arg \min} \| y - X\beta \|_2^2 \qquad s.t. \ (g^\top - I_p) \beta = 0 \ \forall g \in G,
	$$
	can in principle, be solved via convex optimization. 

	Intuitively, we expect both $\hbeta_{\textnormal{aDIST}}$ and $\hbetaCERM$ to be better than the vanilla ERM
	$$
		\hbetaERM =  \underset{\beta}{\arg \min} \| y - X\beta \|_2^2,
	$$
	Let $\rERM, r_{\textnormal{aDIST}}, \rCERM$ be the mean squared errors of the three estimators. We summarize the relationship between the three estimators in the following proposition:
	\begin{proposition}[Comparison between ERM, aDIST and cERM in linear regression]
	\label{prop:risk_comparison_in_linear_regression}
	Let the action of $G$ be linear. Then:
	\benum
		\item Denote $v_j\in\RR^p$ as the $j$-th eigenvector of $X^T X$ and $d_j^2$ as the corresponding eigenvalue. We have
		\begin{align*}
		\rERM  = \gamma^2 \tr [X^\top X]^{-1} = \gamma^2 \sum_{j=1}^p  d_j^{-2}, \qquad r_{\textnormal{aDIST}}  = \gamma^2 \sum_{j=1} ^p d_j^{-2}  \| \mcG^\top v_j \|_2^2,	
		\end{align*}
		where $\mcG = \E_{g\sim \mbQ} [g]$.
		\item If $G$ acts orthogonally, then $r_{\textnormal{aDIST}}\leq \rERM$.
		\item If $G$ is the permutation group over $\{1, ..., p\}$, then
		$$
		r_{\textnormal{aDIST}} = \gamma^2 p^{-1} 1_p^\top (X^\top X)^{-1} 1_p, \qquad \rCERM = \gamma^2 p (1_p ^\top X^\top X 1_p)^{-1}.
		$$
		Furthermore, if $X$ is an orthogonal design so that $X^\top X = I_p$, we have
		$$
		\rERM  = p \gamma^2 , \qquad r_{\textnormal{aDIST}} = \rCERM = \gamma^2.
		$$
	\eenum
	\end{proposition}
	\begin{proof}
	See Appendix \ref{subappend:proof_of_risk_comparison_in_linear_regression}.
	\end{proof}

	In general, constrained ERM can be even more efficient than the estimator obtained by the augmention distribution. However, by the third point in the above proposition, in the special case where $G$ is the permutation group, we have $r_{\textnormal{aDIST}}=\rCERM \ll \rERM$ when the dimension $p$ is large. A direct extension of the above proposition shows that such a phenomenon occurs when $G$ is the permutation group on a subset of $\{1, \ldots, p\}$. There are several other subgroups of interest of the permutation group, including the group of cyclic permutations and the group that contains the identity and the operation that ``flips" or reverses each vector.
	
	We note briefly that the above results apply \emph{mutatis mutandis} to logistic regression. There, the outcome $Y \in \{-1,1\}$ is binary, and $P(Y=1|X=x) = \sigma(x^\top \beta)$, where $\sigma(z) = 1/(1+\exp(-x))$ is the sigmoid function. The invariance condition reduces to the same as for linear regression. We omit the details.


\section{Extension to approximate invariance} \label{sec:approx_inv}
    In this section, we develop extensions of the results in Section \ref{sec:exact_inv} without assuming exact invariance. We only require \emph{approximate invariance} $gX \approx_d X$ in an appropriate sense. We start by recalling the notion of distance between probability distributions based on transporting the mass from one distribution to another (see, e.g., \citealt{villani2003topics} and references therein):
    \begin{definition}[Wasserstein metric]
        Let $\xx$ be a Polish space. Let $d$ be a lower semi-continuous metric on $\xx$. For two probability distributions $\mu, \nu$ on $\xx$, we define
        $$
        \ww_d(\mu, \nu) = \inf_{\pi \in \Pi(\mu, \nu)} \int_{\xx \times \xx} d(x, y) d\pi(x, y),
        $$
        where $\Pi(\mu, \nu)$ are all couplings whose marginals agree with $\mu$ and $\nu$. When $\xx$ is a Euclidean space and $d$ is the Euclidean distance, we denote $\ww_{\ell_2} \equiv \ww_1$ and refer to it as the Wasserstein-$1$ distance.
    \end{definition}
    
    \subsection{General estimators}

    Since we no longer have $gX=_d X$, the equation $\E f = \E \bar f$ cannot hold in general. However, we still expect some variance reduction by averaging a function over the orbit. Hence, we should see a bias-variance tradeoff, which is made clear in the following lemma. For symmetric matrices $A,B,C$, we will use the notation $A \in [B,C]$ to mean that $B \preceq A \preceq C$ in the Loewner order.

    \begin{lemma}[Approximate invariance lemma]
        \label{lemma:approx_inv_lem}
        Assume the conditions in Lemma \ref{lemma:exact_invariance_lemma} hold, but $gX \neq_d X$. Let $\|f\|_\infty = \sup \|f(x)\|_2$ (which can be $\infty$). Then:
        \benum
        \item The expectations satisfy $\| \E_X \bar f(X) - \E_X f(X) \|_2 \leq \E_g \ww_1(f(gX), f(X))$;
        \item The covariances satisfy $\C_X \bar f(X) = \C_{(X, g)} f(gX) - \E_X \C_g f(gX)$, and according to the Loewner order, we have 
        \begin{align*}
        \C_{X} \bar f(X)  - \C_X f(X) \in \bigg[- \E_X \C_g f(gX)  \pm 4\|f\|_\infty \E_g \ww_1(f(gX), f(X)) \cdot I \bigg],
        \end{align*}
        where $I$ is the identity matrix;
        \item Let $\varphi$ be any real-valued convex function, and let $\overline{\varphi \circ f} (x) = \E_g \varphi\circ f(gx) $. Then
        $$
        \E_X \varphi(\bar f(X)) - \E_X \varphi(f(X)) \in \bigg[\bigg(\E_X \varphi(\bar f(X))  - \E_X \overline{\varphi \circ f} (X)\bigg) \pm \|\varphi\|_{\textnormal{Lip}} \E_g \ww_1(f(gX), f(X))\bigg],
        $$
        where $\|\varphi\|_{\textnormal{Lip}}$ is the (possibly infinite) Lipschitz constant of $\varphi$. We recall from our prior results that $\E_X \varphi(\bar f(X))  - \E_X \overline{\varphi \circ f} (X) \leq 0$.
        \eenum
    \end{lemma}    
    \begin{proof}
        See Appendix \ref{subappend:proof_of_approx_inv_lemma}.
    \end{proof}
    The above lemma reduces to Lemma \ref{lemma:exact_invariance_lemma} under exact invariance. With this lemma at hand, we prove an analog of Proposition \ref{prop:cov_of_general_estimators} below, which characterizes the mean squared error if we think of $f$ as a general estimator:

	\begin{proposition}[MSE of general estimators under approximate invariance]
	\label{prop:mse_of_general_estimators_under_approx_inv}
	Under the setup of Lemma \ref{lemma:approx_inv_lem}, consider the estimator $\htheta(X)$ of $\thetanull$, and its augmented version $\htheta_G(X) = \E_{g\sim\mbQ} \htheta(gX)$. Then we have
	$$
	\textnormal{MSE}(\htheta_G) - \textnormal{MSE}(\htheta) \in \bigg[ -\E_X \tr(\C_g \htheta (gX)) \pm \Delta \bigg],
	$$
	where
	$$
	\Delta = \E_g\ww_1(\htheta(gX), \htheta(X))
	\cdot  \bigg[ \E_g\ww_1(\htheta(gX), \htheta(X)) + 2 \|\textnormal{Bias}(\htheta(X))\|_2 + 4\|\htheta\|_\infty \bigg].
	$$
	\end{proposition}
	\begin{proof}
		See Appendix \ref{subappend:proof_of_mse_of_general_estimators_under_approx_inv}.
	\end{proof}
	Compared with exact invariance , apart from the variance reduction term $\E_X\tr(\C_g \htheta (gX))$, we have an additional ``bias'' term  $\Delta$ as a result of approximate invariance. This bias has three components in the present form. We need the following to be small: (1) the Wasserstein-$1$ distance $\E_g\ww_1(\htheta(gX), \htheta(X))$ (which is small if the invariance is close to being exact), (2) the bias of the original estimator $\textnormal{Bias}(\htheta(X))$ is small (which is a minor requirement, because otherwise there is no point in using this estimator), and (3) the sup-norm $\|\htheta\|_\infty$ is small (which is reasonable, because we can always standardize the estimating target $\thetanull$). 

	\subsection{ERM / M-estimators}
	We now extend the results on the behavior of ERM. Recall that $\thetanull$ and $\theta_G$ are minimizers of the population risks $\E L(\theta, X)$ and $\E \bar L(\theta, X)$, respectively. Meanwhile, $\htheta_n$ and $\htheta_{n, G}$ are minimizers of the empirical risks $n^{-1} \sum_{i=1}^n L(\theta, X_i)$ and $n^{-1} \sum_{i=1}^n \bar L(\theta, X_i)$, respectively. Under regularity conditions, it is clear that $\htheta_{n, G}$ is an asymptotically unbiased estimator for $\theta_G$. However, under approximate invariance, some quantities, for example, $\thetanull$ and $\theta_G$, may not coincide exactly, introducing an additional bias. As a result, there is a bias-variance tradeoff, similar to that observed in Proposition \ref{prop:mse_of_general_estimators_under_approx_inv}. 

	We first illustrate this tradeoff for the generalization error, in an extension of Theorem \ref{thm:rademacher_bound_under_exact_inv}:

	\begin{theorem}[Rademacher bounds under approximate invariance]
	\label{thm:rademacher_bound_under_approx_inv}
	Let $L(\theta, \cdot)$ be Lipschitz uniformly over $\theta\in \Theta$, with a (potentially infinite) Lipschitz constant $\|L\|_{\textnormal{Lip}}$. Assume  $L(\cdot, \cdot) \in [0, 1]$. Then with probability at least $1-\delta$ over the draw of $X_1,\ldots,X_n$, we have
		\begin{align*}
		\E L(\htheta_n, X) - \E L(\thetanull, X) & \leq 2 \rad_n(L \circ \Theta) +  \sqrt{\frac{2\log 2/\delta}{n}}
		\textnormal{   \,\,\,\,\,\, (Classical Rademacher bound)}
		\\
		\E L(\htheta_{n, G}, X) - \E L(\thetanull, X) & \leq 2 \rad_n(\bar L \circ \Theta) +  \sqrt{\frac{2\log 2/\delta}{n}} + 2 \| L \|_{\textnormal{Lip}}\cdot \E_{g\sim \mbQ} \ww_1(X, gX).\\
		&\qquad \qquad \textnormal{   \,\,\,\,\,\, (Bound for augmentation under approx invariance)}
		\end{align*}
		Moreover, the Rademacher complexity of the augmented loss class can further be bounded as 
        \begin{equation*}
        \rad_n(\bar L \circ \Theta) -  \rad_n(L\circ \Theta) \leq \Delta + \| L \|_{\textnormal{Lip}}\cdot\E_{g\sim\mbQ} \ww_1(X, gX),
        \end{equation*}
        where
        \begin{align*}
        \Delta &= \E \sup_{\theta } | \frac{1}{n}\sum_{i=1}^n \ep_i \E_g L(\theta, gX_i) | - \E \E_g \sup_{\theta\in\Theta} | \frac{1}{n} \ep_i L(\theta, gX_i) | \leq 0
        \end{align*}
        is the ``variance reduction'' term.
	\end{theorem}
	\begin{proof}
		See Appendix \ref{subappend:proof_of_rademacher_bound_under_approx_inv}.
	\end{proof}
	Compared to Theorem \ref{thm:rademacher_bound_under_exact_inv}, we see that there is an additional bias term $2 \| L \|_{\textnormal{Lip}}\cdot \E_g \ww_d(X, gX)$ due to approximate invariance. The bound $\Delta\leq 0$ can be loose, and this variance reduction term can be much smaller than zero. Consequently, as long as $\E_g\ww_1(X, gX)$ is small, the above theorem can potentially yield a tighter bound for the augmented estimator $\htheta_{n, G}$. 

	We now illustrate the bias-variance tradeoff in the asymptotic regime. The next theorem is analogous to Theorem \ref{thm:asymp_normality_augmented_estimator}:
	\begin{theorem}[Asymptotic normality under approximate invariance]
		\label{thm:asymp_normality_augmented_estimator_under_approx_inv}
		Assume $\Theta$ is open. Let Assumption \ref{assump:regularity_of_thetanull} and \ref{assump:regularity_of_loss} hold for both pairs $(\thetanull, L)$ and $(\theta_G, \bar L)$. Let $V_0, V_{G}$ be the Hessian of $\theta\mapsto \E L(\theta, X)$ and $\theta\mapsto \E \bar L(\theta, X)$, respectively. Let  $M_0(X) = \nabla L(\thetanull, X)\nabla L(\thetanull, X)^\top$ and $M_G(X) = \nabla L(\theta_G, X)\nabla L(\theta_G, X)^\top$.
		Then we have
		\begin{align*}
        & n(\textnormal{MSE}(\htheta_{n, G}) - \textnormal{MSE}(\htheta_n) ) \to  n \|\theta_G - \thetanull\|_2^2 + \E_g \E_X \bigg\la M_G(gX) - M_G(X), V_G^{-2} \bigg\ra \\
        & \qquad\qquad \qquad \qquad \qquad \qquad  + \E_X \bigg\la M_G(X) - M_0(X), V_G^{-2}  \bigg\ra +\la \C_X\nabla L(\thetanull, X), V_G^{-2} - V_0^{-2} \ra \\
        & \qquad\qquad \qquad \qquad \qquad \qquad  - \la \E_X \C_g(\nabla L(\theta_G, gX)), V_{G}^{-2} \ra.
        \end{align*}
	\end{theorem}
	\begin{proof}
		See Appendix \ref{subappend:proof_of_asymp_normality_under_approx_inv}.
	\end{proof} 
	From the above theorem, we see that the the variance reduction is determined by the term
	$$
	\E_X \C_g(\nabla L(\theta_G, gX)),
	$$ 
	whereas we have three extra bias terms on the RHS due to approximate invariance. Indeed, if $gX$ and $X$ are close in the $\ww_1$ metric, one can check that the three additional bias terms are small. For example, by the dual representation of Wasserstein metric, we have
	\begin{align*}
		\E_g \E_X \bigg\la M_G(gX) - M_G(X), V_G^{-2} \bigg\ra & \leq \mathbb{E}_g \mathcal{W}_1 \bigg(\la M_G(gX), V_G^{-2}\ra, \la M_G(X) , V_G^{-2} \ra\bigg) \\
		\E_X \bigg\la M_G(X) - M_0(X), V_G^{-2}  \bigg\ra 
    & \leq \mathcal{W}_1 \bigg( \la M_G(X), V_G^{-2}\ra , \la  M_0(X), V_G^{-2}\ra \bigg),
	\end{align*}
	and the upper bound goes to zero as we approach exact invariance. Hence we can see an efficiency gain by data augmentation. 



\section{A Case Study on Over-Parameterized Two-Layer Nets}
\label{sec:overparam_2l_nets}
The result on two-layer nets in Section \ref{sec:ex} has three major limitations. First, we only considered quadratic activations, whereas in practice rectified linear units (ReLU) are usually used. Also, the entire result concerns the under-parameterized regime, whereas in practice, neural nets are often over-parameterized. Finally, we assumed that we can optimize the weights of the neural network to consistently estimate the true weights, which can be a hard nonconvex optimization problem with only partially known solutions. 

In this section, we present an example that addresses the above three limitations: we consider a ReLU network; we allow the input dimension and the number of neurons to scale with $n$; we use gradient descent for training.

Consider a binary classification problem, where the data points $\{X_i, Y_i\}_1^n \subseteq \mathbb{S}^{n-1}\times \{\pm 1\}$ are sampled i.i.d. from some data distribution. For technical convenience, we assume $|G|< \infty$, and $gX\in \mathbb{S}^{d-1}$ for all $g\in G$ and almost every $X$ from the feature distribution. We again consider a two-layer net $f(x; W, a) = \frac{1}{\sqrt{m}} a^\top \sigma(Wx)$, where $W\in \R^{m\times d}, a\in \R^m$ are the weights and $\sigma(x) = \max(x, 0)$ is the ReLU activation. We consider the following initialization: $W_{ij}\sim \N(0, 1), a_s \sim \textnormal{unif}(\{\pm 1\})$. Such a setup is common in recent literature on Neural Tangent Kernel.

In the training process, we \emph{fix $a$ and only train $W$}. We use the logistic loss $\ell(z) = \log(1+e^{-z})$. In our previous notation, we have $L(\theta, X, Y) = \ell(Y f(X; W, a))$. The weight is then trained by gradient descent: $W_{t+1} = W_t - \eta_t \nabla \bar{R}_n(W_t)$, where $\bar{R}_n$ is the augmented empirical risk \eqref{eq:aug_empirical_risk}.

To facilitate the analysis, we impose a margin condition below, similar to that in \cite{ji2019polylogarithmic}.
\begin{assump}[Margin condition]
\label{assump:margin_condition}
Let $\mathcal{H}$ be the space of functions $v: \R^d \to \R^d$ s.t. $\int \|v(z)\|_2^2 d\mu(z) < \infty$, where $\mu$ is the $d$-dimensional standard Gaussian probability measure. Assume there exists $\bar{v}\in \mathcal{H}$ and $\gamma > 0$, s.t. the Euclidean norm satisfies $\|\bar v(z)\|_2\leq 1$  for any $z \in \R^d$, and that
$$
    Y \int \la \bar v(z), gX \ra \1\{\la z, gX \ra > 0\} d \mu(z) \geq \gamma
$$
for all $g\in G$ and amost all $(X, Y)$ from the data distribution.
\end{assump}
This says that there is a classifier that can distinguish the (augmented) data with positive margin.

We need a few notations. For $\rho > 0$, we define $\mathscr{W}_\rho := \{W\in\RR^{m\times d}: \|w_{s}-w_{s, 0} \|_2\leq \rho \textnormal{ for any } s\in[m]\}$, where $w_s, w_{s, 0}$ is the $s$-th row of $W, W_0$, respectively. We let
$$
    \mathcal{R}_n := \E\sup_{W\in \mathscr{W}_\rho} | \frac{1}{n} \ep_i  [ -\ell'(y_i f(X_i; W, a))] |
$$
be the Rademacher complexity of the non-augmented gradient, where the expectation is taken over both $\{\ep_i\}_1^n$ and $\{X_i, Y_i\}_1^n$. Similarly, writing $f_{i, g}(W) = f(gX_i; W, a)$, we define
$$
    \bar{\mathcal{R}}_n := \E \sup_{W\in \mathscr{W}_\rho} | \frac{1}{n} \ep_i \E_g [ -\ell'(Y_i f_{i, g}(W))] |
$$
to be the Rademacher complexity of the augmented gradient. The following theorem characterizes the performance gain by data augmentation in this example:

\begin{theorem}[Benefits of data augmentation for two-layer ReLU nets]
    \label{thm:overparam_two_layer_net_gd}
    Under Assumption \ref{assump:margin_condition}, take any $\ep \in (0, 1)$ and  $\delta \in (0, 1/5)$. Let
    $$
        \lambda = \frac{\sqrt{2\log(4n|G|/\delta)} + \log (4/\ep)}{\gamma/4} , M = \frac{4096 \lambda^2}{\gamma^6}, 
    $$
    and let $\rho = 4\lambda/(\gamma \sqrt{m})$. Let $k$ be the best iteration (with the lowest empirical risk) in the first $\lceil 2\lambda^2/n\ep \rceil$ steps. Let $\alpha=16[\sqrt{2\log(4n|G|/\delta)} + \log (4/\ep)]/\gamma^2 + \sqrt{md} + \sqrt{2\log(1/\delta)}$.
    For any $m \geq M$ and any constant step size $\eta\leq 1$, with probability at least $1-5\delta$ over the random initialization and i.i.d. draws of the data points, we have
    \begin{align*}
        &\P(Y f(X; W_k, a)\leq 0) \leq {2\ep} + [{\sqrt{\frac{2\log 2/\delta}{n}} + 4\bar{\mathcal{R}}_n}]+ \frac{1}{2} \E_{Y} \E_{g} \mathcal{W}_1(X|Y, gX|Y) \cdot \alpha.
    \end{align*}
    The three terms bound the optimization error,  generalization error, and the bias due to approximate invariance.
    Moreover, with probability at least $1-\delta$ over the random initialization, we have
    \begin{align*}
        & \bar{\mathcal{R}}_n - \mathcal{R}_n \leq \Delta + \frac{1}{4} \E_Y \E_{g} \mathcal{W}_1(X|Y, gX|Y) \cdot  \alpha,
    \end{align*}    
    where
    \begin{align*}
        & \Delta = \E \sup_{W\in \mathscr{W}_\rho} | \frac{1}{n} \ep_i \E_g [ -\ell'(y_i f_{i, g}(W))] |-  \E\E_g \sup_{W\in \mathscr{W}_\rho} | \frac{1}{n} \ep_i  [ -\ell'(y_i f_{i, g}(W))] | \leq 0
    \end{align*}    
    is the ``variance reduction'' term.
\end{theorem}
\begin{proof}
	See Appendix \ref{subappend:proof_of_ntk}.
\end{proof}
The proof idea is largely based on results in \cite{ji2019polylogarithmic}. We decompose the overall error into optimization error and generalization error. The optimization error is taken care of by a corollary of Theorem 2.2 in \cite{ji2019polylogarithmic}. The generalization error is dealt with by adapting several arguments in Theorem 3.2 of \cite{ji2019polylogarithmic} and using some arguments in the proof of Theorem \ref{thm:rademacher_bound_under_approx_inv}.

Again, we see a bias-variance tradeoff in this example due to approximate invariance. We conclude this section by noting that the term $\rad_n$ can be further bounded by invoking Rademacher calculus and the Rademacher complexity bound for linear classifiers. We refer readers to Section 3 of \cite{ji2019polylogarithmic} for details.


	\section{Potential applications} \label{sec:potential_applications}

	\subsection{Cryo-EM and related problems}
	In this section, we describe several important problems in the biological and chemical sciences, and how data augmentation may be useful. Cryo-Electron Microscopy (Cryo-EM) is a revolutionary technique in structural biology, allowing us to determine the structure of molecules to an unprecedented resolution \citep[e.g.,][]{frank2006three}. The technique was awarded the Nobel Prize in Chemistry in 2017. 
	
	The data generated by Cryo-EM poses significant data analytic (mathematical, statistical, and computational) challenges \citep{singer2018mathematics}. In particular, the data possesses several invariance properties that can be exploited to improve the accuracy of molecular structure determination. However, exploiting these invariance properties is highly nontrivial, because of the massive volume of the data, and due to the high levels of noise. In particular, exploiting the invariance is an active area of research \citep{Kam1980,frank2006three,zhao2016fast,bandeira2017estimation,bendory2018bispectrum}. Classical and recent approaches involve mainly (1) latent variable models for the unknown symmetries, and (2) invariant feature approaches. Here we will explain the problem, and how data augmentation may help. 
	
	In the imaging process, many copies of the molecule of interest are frozen in a thin layer of ice, and then 2D images are taken via an electron beam. A 3D molecule is represented by an electron density map $\phi:\R^3\to \R$. Each molecule is randomly rotated, via a rotation that can be represented by a 3D orthogonal rotation matrix $R_i \in O(3)$. Then we observe the noisy line integral
	$$Y_i = \int_{z} \phi(R_i[x,y,z]^\top)dz+\ep_i.$$
	We observe several iid copies, and the goal is to estimate the density map $\phi$. Clearly the model is invariant under rotations of $\phi$. Existing approaches mainly work by fitting statistical methods for latent variable models, such as the expectation maximization (EM) algorithm. Data augmentation is a different approach, where we add the data transformed according to the symmetries. It is interesting, but beyond our scope, to investigate if this can improve the estimation accuracy. \\

	{\bf Invariant denoising.} A related problem is invariant denoising, where we want to denoise images subject to an invariance of their distribution, say according to rotations \citep[see e.g.,][]{vonesch2015steerable,zhao2016fast,zhao2018steerable}. This area is well studied, and popular approaches rely on invariant features. It is known how to do it for rotations. However, capturing translation-invariance poses complications to the invariant features approach. In principle, data augmentation could be used as a more general approach.\\
    
    {\bf XFEL.} Another related technique, X-ray free electron lasers (XFEL), is a rapidly developing and increasingly popular experimental method for understanding the three-dimensional structure of molecules \citep[e.g.,][]{ favre2015xtop, maia2016trickle,xfelbook}.  Single molecule XFEL imaging collects two-dimensional diffraction patterns of single particles at random orientations. A key advantage is that XFEL uses extremely short femtosecond X-ray pulses, during which the molecule does not change its structure. On the other hand, we only capture one diffraction pattern per particle and the particle orientations are unknown, so it is challenging to reconstruct the 3D structure at a low signal-to-noise ratio. The images obtained are very noisy due to the low number of photons that are typical for single particles \citep{pande2015simulations}. 
    
    A promising approach for 3-D structure reconstruction is Kam's method \citep{kam1977determination, Kam1980,Saldin2009}, which requires estimating the covariance matrix of the noiseless 2-D images. This is extremely difficult due to low photon counts, and motivated prior work to develop improved methods for PCA and covariance estimation such as $e$PCA \citep{liu2018pca}, as well as the steerable $e$PCA method \cite{zhao2018steerable} that builds in invariances.  As above, it would be interesting to investigate if we can use data augmentation as another approach for rotation-invariance.

    \subsection{Spherically invariant data}
    
    Here we discuss models for spherically invariant data, and how data augmentation may be used. See for instance \cite{fisher1993statistical} for more general models of spherical data. In the invariant model, the data $X\in \R^p$ is such that $X=_d OX$ for any orthogonal matrix $O$. One can see that the Euclidean norms $\|X\|$ are sufficient statistics. There are several problems of interest: 
    \bitem 
    \item Estimating the radial density. By taking the norms of the data, this reduces to estimating their 1D density.
    \item 
    Estimating the marginal density $f$ of a single coordinate. Here it is less obvious how to exploit spherical invariance. However, data augmentation provides an approach.
    \eitem 
    
    A naive estimator for the marginal density is any density estimator applied to the first coordinates of the data, $X_1(1), \ldots, X_n(1)$. Since $X_i(1) \sim_{iid} f$, we can use any estimator, $\hat f(X) = \hat f(X_1(1), \ldots, X_n(1))$ for instance a kernel density estimator. However, this is inefficient, because it does not use information in all coordinates.
    
    In data augmentation we rotate our data uniformly, leading to 
    
    $$\hat f_a(X) = \int\hat f(gX)d\mbQ{(g)}
    = \E_{O_1,\ldots,O_p\sim O(p)} f([O_1X_1](1), \ldots, [O_pX_p](1)).$$  
    
    Note that if $O\sim O(p)$, then for any vector $x$, $[Ox](1) =_d\|x\| Z(1)/\|Z\|$, where $Z\sim \N(0,I_p)$. Hence the expectation can be rewritten in terms of Gaussian integrals. It is also possible to write it as a 2-dimensional integral, in terms of $Z(1),\|Z(2:p)\|^2$, which have independent normal and Chi-squared distributions. However, in general, it may be hard to compute exactly.
    
    When the density estimator decouples into a sum of terms over the datapoints, then this expression simplifies. This is the case for kernel density estimators: $\hat f(x) = (nh^p)^{-1} \sum_{i=1}^n k([x-x_i(1)]/h)$. More generally, if $\hat f(x) = \sum_{i=1}^n T(x-x_i(1))$, then we only need to calculate
    \begin{align*}
    \tilde T(x) &=  \E_O T\left(x-[Oy](1)\right)\\
    &=\E_{Z\sim\N(0,I_p)} T\left(x-\|y\|\frac{Z(1)}{\|Z\|}\right).
    \end{align*}
    This is significantly simpler than the previous expression. It can also be viewed as a form of convolution of a kernel with $T$, which is already a kernel typically. Therefore, we have shown how data augmentation can be used to estimate the marginal density of coordinates for spherically uniform data.
    
    \subsection{Random effects models}
    
    Data augmentation may have applications to certain random effect models \citep{searle2009variance}. Consider the one-way layout $X_{ij} = \mu + A_i + B_{ij}$, $i=1\ldots,s$, $j=1,\ldots,n_i$, where $A_i \sim \N(0,\sigma_A^2)$,  $B_{ij} \sim \N(0,\sigma_B^2)$ independently. We want to estimate the global mean $\mu$ and the variance components $\sigma_A^2, \sigma_B^2$.  If the layout is unbalanced, that is the number of replications is not equal, this can be somewhat challenging. Two general approaches for estimation are the restricted maximum likelihood (REML), and minimum norm quadratic estimation (MINQUE) methods. 
    
    Here is how one may use data augmentation. Consider a simple estimator of $\sigma_B^2$ such as $\hsigma^2_B $ $= s^{-1}\sum_{i=1}^n$ $\E (X_{i1}- X_{i2})^2/2$ (we assume that $n_i \ge 2$ for all $i$). This is a heuristic plug-in estimator, which is convenient to write down and compute in a closed form. It is also unbiased. However, it clearly does not use all samples, and therefore, it should be possible to improve it. 
    
    Now let us denote by $X_i$ the block of $i$-th observations. These have a joint normal distribution $X_i \sim \N(\mu 1_{n_i}, \sigma_A^2 1_{n_i} 1_{n_i}^\top + \sigma_B^2 I_{n_i})$. The model is invariant under the operations
    $$X_i \to O_i X_i,$$
    for any orthogonal matrix $O_i$ of size $n_i$ for which $O_i 1_{n_i} = 1_{n_i}$, i.e., a matrix that has the vector of all ones $1_{n_i}$ as an eigenvector. Let $G_i$ be the group of such matrices. Then the overall model is invariant under the action of the direct product $G_1 \times G_2 \times \ldots \times G_s$. Therefore, any estimator that is not invariant with respect to this group can be improved by data augmentation. 
    
    Going back to the estimator $\hsigma_B^2$, to find its augmented version, we need to compute the quantity $\E ([Ox]_1 - [Ox]_2)^2$, where $x\in \R^k$ is fixed and $O$ is uniformly random from the group of orthogonal matrices such that $O1_k = 1_k$. Write $x = \bar x 1_k + r$, where $\bar x$ is the mean of the entries of $x$. Then $Ox = \bar x 1_k + Or$, and $[Ox]_j = \bar x  + [Or]_j$. Thus we need $\E ([Or]_1 - [Or]_2)^2$. This can be done by using that $Or$ is uniformly distributed on the $k-1$ dimensional orthocomplement of the $1_k$ vector, and we omit the details. 




\acks{We are grateful to Zongming Ma for proposing the topic, and for participating in our early meetings. We thank Jialin Mao for participating in several meetings and for helpful references. We thank numerous people for valuable discussions, including Peter Bickel, Kostas Daniilidis, Jason Lee, William Leeb, Po-Ling Loh, Tengyu Ma, Jonathan Rosenblatt, Yee Whye Teh, Tengyao Wang,  the members in Konrad Kording's lab, with special thanks to David Rolnick. We thank Dylan Foster for telling us about the results from \cite{foster2019complexity}. This work was supported in part by NSF BIGDATA grant IIS 1837992 and NSF TRIPODS award 1934960. We thank the Simons Institute for the Theory of Computing for providing AWS credits.}




\appendix
\section{Proofs for results under exact invariance}\label{append:proofs_under_exact_invariance}

\subsection{Proof of Lemma \ref{lemma:exact_invariance_lemma}}\label{subappend:proof_of_inv_lemma}
We first prove part 1.  Let $x$ be fixed. Let $A = \{X \in Gx\}$. It suffices to show
    $$
        \int_A \E_g f(gx) d \P(X) = \int_A f(X) d \P(X).
    $$
For an arbitrary $g \in G$, the RHS above is equal to
\begin{align*}
    \int_A f(X)  d\P(X) & =  \int f(gX) \mathbbm{1}\{gX \in Gx\} d\P(X) \\
            & = \int f(gX) \mathbbm{1}\{X \in Gx\} d\P(X),
\end{align*}
where the first equality is by the exact invariance, and the second equality is by the definition of the orbit. Taking expectation w.r.t. $\mbQ$, we get
\begin{align*}
    \int_A f(X)  d\P(X) = \int_G \int f(gX) \mathbbm{1}\{X \in Gx\} d\P(X) d\mbQ(g).
\end{align*}
On the event $A$, there exists $g^*_X$, potentially depending on $X$, s.t. $X = g^*_X x$. Hence, we have
\begin{align*}
     \int_A f(X)  d\P(X) 
    & = \int_G \int f(g \circ g^*_X x)  \mathbbm{1}\{X \in Gx\} d\P(X) d\mbQ(g) \\ 
    & =   \int \int_G f(g\circ g^*_X x)  d \mbQ(g) \mathbbm{1}\{X \in Gx\} d\P(X)\\
    & = \int \int_G f(gx) d\mbQ(g) \mathbbm{1}\{X \in Gx\} d\P(X) \\
    & = \int_A \E_g f(gx) d\P(X),
\end{align*}
where the second equality is by Fubini's theorem, and the third inequality is due to the translation invariant property of the Haar measure.

Part 2 follows by law of total expectation along with the above point.
        
Part 3 follows directly from part 1 and the law of the total covariance applied to the random variable $f(gX)$, where $g\sim \mbQ, X \sim \P$. 

Part 4 follows from Jensen's inequality.

\subsection{Proof of Lemma \ref{lemma:regularity_of_augmented_loss}}\label{subappend:proof_of_regularity_of_augmented_loss}
By exact invariance and Fubuni's theorem, it is clear that $\E \bar L(\theta, X) = \E L(\theta, X)$ for any $\theta \in \Theta$. Hence $\theta_G = \theta_0$ and Assumption $A$ is verified for $(\theta_G, \bar L)$. We now verify the five parts of Assumption B.

For part 1, we have
\begin{align*}
    \sup_{\theta\in\Theta} \bigg|\frac{1}{n} \sum_{i=1}^n \E \bar L(\theta, X_i) - \E \bar L(\theta, X)\bigg| & = \sup_{\theta\in\Theta} \bigg|\E_g [\frac{1}{n} \sum_{i=1}^n L(\theta,gX_i) - \E L(\theta,gX)]\bigg| \\
    & \leq \sup_{\theta \in \Theta} \E_g \bigg| \frac{1}{n} \sum_{i=1}^n  L(\theta,gX_i) - \E L(\theta,gX) \bigg| \\
    & \leq \E_g \sup_{\theta \in \Theta}  \bigg| \frac{1}{n} \sum_{i=1}^n  L(\theta,gX_i) - \E L(\theta,gX) \bigg| \\
&  = o_p(1),
\end{align*}
where the two inequalities is by Jensen's inequality, and the convergence statement is true because of the exact invariance and the fact that the original loss satisfies part 1 of Assumption B. 

Part 2 is true because we have assumed that the action $x \mapsto gx$ is continuous.

For part 3, since $(\thetanull, L)$ satisfies this assumption, we know that on an event with full probability, we have
$$
    \lim_{\delta\to 0} \frac{\bigg| L(\thetanull+\delta,gX) - L(\thetanull,gX) - \delta^\top \nabla L(\thetanull,gX)\bigg|}{\|\delta\|} = 0.
$$
Now we have
\begin{align*}
    &\frac{\bigg| \E_g L(\thetanull+\delta,gX)  - \E_g L(\thetanull,gX) - \delta^\top \E_g \nabla L(\thetanull,gX) \bigg|}{\|\delta\|} \\
    & \leq \frac{\E_g \bigg| L(\thetanull+\delta,gX) - L(\thetanull,gX) - \delta^\top \nabla L(\thetanull,gX)\bigg|}{\|\delta\|} \\
    & \leq \E_g [\dot L(gX) + \|\nabla L(\thetanull, gX) \|],
\end{align*}
where the first inequality is by Jensen's inequality, and the second inequality is by part 4 applied to $(\thetanull, L)$. We have assumed that $\E[\dot L(X)^2] < \infty$, and hence so does $\E[\dot L(X)]$. By exact invariance, we have 
$$
    \E_X \E_g [\dot L(gX)] = \E_g \E_X [\dot L(gX)] = \E [\dot L(X)] < \infty,
$$
which implies $\E_g[\dot L(gx)] < \infty$ for $\P$-a.e. $x$. On the other hand, part 5 applied to $(\thetanull, L)$ implies the existence of $\E \nabla L(\thetanull, X) \nabla L(\thetanull, X)^\top$, and hence $\E\| \nabla L(\thetanull, X) \|^2 < \infty$ and so does $\E \| \nabla L(\thetanull, X) \|$. Now a similar argument shows that under exact invariance, we have
$$
    \E_g \| \nabla L(\thetanull, gx) \| < \infty.
$$
for $\P$-a.e. $x$. Hence we can apply dominated convergence theorem to conclude that
\begin{align*}
    &\lim_{\delta\to 0} \frac{\bigg| \E_g L(\thetanull+\delta,gX)  - \E_g L(\thetanull,gX) - \delta^\top \E_g \nabla L(\thetanull,gX) \bigg|}{\|\delta\|}  \\
    & \leq 
    \E_g \lim_{\delta\to 0} \frac{\bigg| L(\thetanull+\delta,gX) - L(\thetanull,gX) - \delta^\top \nabla L(\thetanull,gX)\bigg|}{\|\delta\|} \\
    & = 0,
\end{align*}    
which implies that $\theta\mapsto \bar L(\theta, x)$ is indeed differentiable at $\thetanull$.

For part 4, by assumption, we have
$$
    |L(\theta_1, gx) - L(\theta_2, gx) | \leq \dot L(gx) \|\theta_1 - \theta_2 \|
$$
for almost every $x$ and every $\theta_1, \theta_2$ in a neighborhood of $\thetanull = \theta_G$. Taking expectation w.r.t. $g\sim \mbQ$, we get
\begin{align*}
    |\bar L(\theta_1,x) - \bar L(\theta_2,x)|
    & \leq 
    \E_g|L(\theta_1,gx) - L(\theta_2,gx)| \\
    & \leq
    \E_g \dot L(gx) \|\theta_1 - \theta_2\|.
\end{align*}    
Now it suffices to show $x\mapsto \E_g \dot L(gx)$ is in $L^2(\P)$. This is true by an application of Jensen's inequality and exact invariance:
\begin{align*}
    \E_X [(\E_g \dot L (gX))^2] & \leq \E_X \E_g [(\dot L(gX))^2] \\
    & = \E_g \E_X [(\dot L (gX))^2] \\
    & = \E_g \E \dot L^2 \\
    & = \E \dot L^2 \\
    & < \infty.
\end{align*}

Part 5 is true because under exact invariance, we have $\E L(\theta, X) = \E \bar L(\theta, X)$.  

\subsection{Proof of Theorem \ref{thm:asymp_normality_augmented_estimator}}\label{subappend:proof_of_asymp_normality_of_augmented_estimator}
The results concerning $\htheta_n$ is classical (see, e.g., Theorem 5.23 of \citealt{van1998asymptotic}). By Lemma 3.4, we can apply Theorem 5.23 of \cite{van1998asymptotic} to the pair $(\theta_G = \thetanull, \bar L)$ to conclude the Bahadur representation. Hence the asymptotic normality of $\htheta_{n, G}$ follows and we have
$$
    \Sigma_G = V_\thetanull^{-1} \E [\bar L(\thetanull, X) \bar L(\thetanull, X)^\top] V_\thetanull^{-1}.
$$
The final representation of $\Sigma_G$ follows from part 3 of Lemma \ref{lemma:exact_invariance_lemma}.
	
\subsection{Proof of Proposition \ref{prop:tangential_decomp}}\label{subappend:proof_of_tangential_decomp}
Let $T_pM$ be the tangent space of $M$ at the point $p$. The inclusion map from $\Theta_G$ to $\R^p$ is an immersion. So we can decompose $T_\thetanull\R^p = T_\thetanull\Theta_G\oplus (T_\thetanull \Theta_G)^\perp$, i.e., the direct sum of the tangent space of $\Theta_G$ at $\thetanull$ and its orthogonal complement. Hence the decomposition is valid. Now as $\ell_\theta(gx) + \log\det g= \ell_\theta(x)$ for any $\theta\in \Theta_G$, it is clear that the gradient of $\ell_\thetanull$ w.r.t. $\Theta_G$ is invariant. 

Using $P_G \nabla \ell_\thetanull(gX) = P_G \nabla \ell_\thetanull(X)$, we have
\begin{align*}
	\E_\thetanull \C_g(\nabla \ell_\thetanull(gX)) & = \E_\thetanull \C_g(P_G\nabla \ell_\thetanull(gX) + P_G^\perp \nabla\ell_\thetanull(gX))\\
	& = \E_\thetanull \C_g(P_G\nabla \ell_\thetanull(X) + P_G^\perp \nabla\ell_\thetanull(gX))\\
	& = \E_\thetanull \C_g(P_G^\perp \nabla\ell_\thetanull(gX)),
\end{align*}
which concludes the proof.

\subsection{Proof of Proposition \ref{prop:relation_between_aMLE_and_cMLE}}\label{subappend:proof_of_relation_between_aMLE_and_cMLE}
From our theory we know that the asymptotic covariance matrix of cMLE for estimating parameters $\theta = (\theta_1,\theta_2)$ is $I^{-1}\bar{I}I^{-1}$, while that of aMLE for estimating $\theta_1$ is $I_{11}^{-1}$. Thus aMLE is asymptotically more efficient than the cMLE if and only if
$$
	I_{11}^{-1} \succeq  [I^{-1}\bar{I}I^{-1}]_{11}.
$$
Denoting $K:=(I^{-1})_{1\bigcdot}$, this is equivalent to 
$$
	I_{11}^{-1} \succeq  K \bar{I} K^\top,
$$
by Equation \eqref{eq:decomp_of_cov_along_manifold} and \eqref{eq: proj_of_gradient_onto_tangent_space},
Using the Schur complement formula, this is equivalent to the statement that the Shur complement of the matrix $I_{G}^{-1}$ in the matrix $M_\theta$ is PSD, where $M_\theta$ equals
$$ 
	M_\theta = 
	\begin{bmatrix}
	I_{11}^{-1}& 
	K& \\
	K^\top&
	\bar{I}^{-1}%
	\end{bmatrix}.
$$
Now, the top left block of this matrix, $I_{11}^{-1}$, is PSD. Thus, from the properties of Schur complements, the entire matrix $M_\theta$ is also PSD. Therefore, the condition is equivalent to the matrix $M_\theta$ being PSD.

\subsection{Proof of Proposition \ref{prop:risk_of_gaussian_mean_estimation}}\label{subappend:proof_of_risk_of_gaussian_mean_estimation}
Let $C = \E_{g\sim\mbQ}[g]$. By orthogonality, for each $g$ we have that $g^\top  = g^{-1}$ is also in $G$. Along with the fact that $\mbQ$ is Haar, we conclude that the matrix $C$ is symmetric.  Then for any $v \in \Theta_G$, we have $Cv = \E_{g\sim\mbQ}[gv] = \E_{g\sim\mbQ}[g^\top v]= \E_{g\sim\mbQ}[v] = v$. 
Moreover, for any $w\in \Theta_G^\perp$, we have $Cw = \E_{g\sim\mbQ}[gw] =\E_{g\sim\mbQ}[g^\top w] = \E_{g\sim\mbQ}[0] = 0$. Hence, $C$ is exactly the orthogonal projection into the subspace $\Theta_G$, which finishes the proof.

\subsection{Proof of Theorem \ref{thm:2lnn}}\label{subappend:proof_of_2lnn}
In the following proof, we will use $\theta$ and $W$ interchangeably when there is no ambiguity.

For part 1 of the theorem, we have
		$$\nabla f(W,x) = \sigma'(Wx) \cdot x^\top \in \R^{p\times d}.$$
		We then have 
		$$\nabla f(W,x) = \sigma'(Wx) \cdot x^\top \in \R^{p\times d}.$$
		We can think of the Fisher information matrix $I_\theta = \E \nabla f(\theta,X)\nabla f(\theta,X)^\top $ as a tensor, i.e, 
		\begin{align*}
		I_W  
		& = \E (\sigma'(WX) \cdot X^\top)\otimes (\sigma'(WX) \cdot X^\top) \\ 
		& = \E (\sigma'(WX)\otimes \sigma'(WX)) \cdot (X \otimes X)^\top.
		\end{align*}
		The $i,j, i',j'$-th entries of this tensor are 
		$$I_W(i,j, i',j') = \E 
		\sigma'(W_i^\top X)\sigma'(W_{i'}^\top X) \cdot
		X_j X_{j'}.$$
		For a quadratic activation function, $\sigma(x) = x^2/2$, we can get more detailed results. We then have
		\begin{align*}
		I_W  
		& = \E (WXX^\top)\otimes (WXX^\top)\\ 
		& = (W \otimes W) \cdot  \E (XX^\top \otimes XX^\top).
		\end{align*}

The group acts by $gx = T_g x$, where $T_g$ is an operator that shifts a vector circularly by $g$ units. We can then write the neural network $f(W,x) = \sum_{i=1}^p h(W_i;x)$ as a sum, where $h(a,x) = \sigma(a^\top x)$. Therefore, the invariant function corresponding to $f_W$ can also be written in terms of the corresponding invariant functions corresponding to the $h$-s:
		$$\bar f(W,x) 
		= \frac1d \sum_{g=1}^d f(W,T_g x)
		= \sum_{i=1}^p \bar h(W_i;x).$$
		where $\bar h(a;x) =\frac1d \sum_{g=1}^d h(a;T_g x)$. We can use this representation to calculate the gradient. We first notice $\nabla h(a;x) = \sigma'(a^\top x) x$. Thus, 
		\begin{align*}
		\nabla \bar h(a;x) 
		& =  \frac1d \sum_{g=1}^d \nabla  h(a;T_g x)
		= \frac1d \sum_{g=1}^d \sigma'(a^\top T_g x) T_g x\\
		& = \frac1d C_x\cdot \sigma'(C_x^\top a).
		\end{align*}
		Here $C_x$ is the circulant matrix 
		$$C_x = [x,T_1x,\ldots, T_{d-1}x]
		= \begin{bmatrix}
		x_1,& x_d,&\ldots, &x_{d-1}\\
		x_2,& x_1,&\ldots, &x_{d}\\
		&&\ldots&\\
		x_d,& x_{d-1},&\ldots,& x_{1}
		\end{bmatrix}.$$
		Hence the gradient of the invariant neural network $\bar f(W,x)$ as a matrix-vector product
		\begin{align*}
		\nabla \bar f(W,x)
		&
		= \begin{bmatrix}
		\nabla \bar h(W_1;x)^\top \\
		\ldots\\
		\nabla \bar h(W_p;x)^\top
		\end{bmatrix}
		= \frac1d \begin{bmatrix}
		\sigma'(W_1^\top C_x) \cdot C_x^\top  \\
		\ldots\\
		\sigma'(W_p^\top C_x) \cdot C_x^\top
		\end{bmatrix}
		=\frac1d \sigma'(W C_x) \cdot C_x^\top.
		\end{align*}
		So the Fisher information can also expressed in terms of matrix products
		\begin{align*}
		\bar I_W  
		& = \E (\sigma'(W C_X) \cdot C_X^\top)\otimes (\sigma'(W C_X) \cdot C_X^\top) \\
		& = \E (\sigma'(W C_X)\otimes\sigma'(W C_X)) \cdot (C_X \otimes C_X)^\top.
		\end{align*}
		For quadratic activation functions, we have 
		\begin{align*}
		\bar I_W  
		& = \frac{1}{d^2} \E (W C_XC_X^\top)\otimes (W C_XC_X^\top) \\
		& = (W \otimes W) \cdot \frac{1}{d^2} \E (C_XC_X^\top \otimes C_XC_X^\top)\\
		& = (W \otimes W) \cdot \frac{1}{d^2} \E (C_X\otimes C_X) \cdot (C_X\otimes C_X)^\top.
		\end{align*}
		Therefore, the efficiency gain can be characterized by the move from the 4th moment tensor of $X$ to that of $\frac{1}{\sqrt{d}} C_X$. 

		Finally, if we assume $X\sim \N(0, I_d)$, then we can express our results in a simpler form using the discrete Fourier transform. Let $F$ be the $d\times d$ Discrete Fourier Transform (DFT) matrix, with entries $F_{j,k} = d^{-1/2} \exp(-2\pi i/d\cdot (j-1)(k-1))$. Then $Fx$ is called the DFT of the vector $x$, and $F^{-1}y = F^*y$ is called the inverse DFT. The DFT matrix is a unitary matrix with $FF^* = F^*F = I_d$. Thus $F^{-1} = F^* $. It is also a symmetric matrix with $F^\top = F$. 
		Then the circular matrix can be diagonalized as
		$$\frac{1}{\sqrt{d}}  C_x = F^*\diag(F x) F.$$
		The eigenvalues of $d^{-1/2} C_x$ are the entries of $F x$, with eigenvectors the corresponding columns of $F$. 
	
		So we can write, with $D:=\diag(F X)$,
		\begin{align*}
		d^{-1} C_X \otimes C_X & = F^* D F \otimes F^* D F \\
		& = (F\otimes F)^* \cdot (D \otimes D) \cdot (F\otimes F)= F_2^* D_2 F_2,
		\end{align*}
		where $F_2 = F\otimes F$, and $D_2=D \otimes D$ is a diagonal matrix. So
		\begin{align*}
		d^{-2} \E (C_X\otimes C_X) \cdot (C_X\otimes C_X)^\top
		& = 
		\E F_2^* D_2 F_2
		\cdot 
		(F_2^* D_2 F_2)^\top\\
		&=
		\E F_2^* D_2 F_2
		\cdot 
		F_2^\top  D_2 F_2^{*,T} \\
		&=
		F_2^*\cdot \E  D_2 F_2^2  D_2  \cdot F_2^{*}.
		\end{align*}
		Here we used that $F = F^\top$, hence $F_2^\top = (F\otimes F)^\top = F^\top\otimes F^\top = F_2$.

		Now, $D_2$ can be viewed as a $d^2\times d^2$ matrix, with diagonal entries $D_2(i,j,i,j)=D_i D_j = F_i^\top X \cdot F_j^\top X$, where $F_i$ are the rows (which are also equal to the columns) of the DFT. Thus the inner expectation can be written as an elementwise product (also known as Hadamard or odot product)
		$$ \E  D_2 F_2^2  D_2  = F_2^2 \odot \E D_2 D_2^\top.$$
		So we only need to calculate the 4th order moment tensor $M$ of the Fourier transform $FX$, 
		$$M_{iji'j'} = \E F_i^\top X \cdot F_j^\top X \cdot F_{i'}^\top X \cdot F_{j'}^\top X.$$
		Let us write $r:=FX$. Then by Wick's formula,
		\begin{align*}
		\E f_i f_j f_{i'}f_{j'}
		& = 
		\E f_i f_j \cdot \E  f_{i'}f_{j'}
		+
		\E f_i f_{j'} \cdot \E  f_{i'}f_j
		+
		\E f_i f_{i'} \cdot \E  f_if_{j'}.
		\end{align*}
		Now 
		$$\E f_i f_j 
		= 
		\E F_i^\top X \cdot F_j^\top X 
		=
		F_i^\top \cdot \E XX^\top  \cdot F_j 
		=
		F_i^\top F_j.
		$$
		Hence
		\begin{align*}
		M_{iji'j'}
		& = 
		F_i^\top F_j  \cdot F_{i'}^\top F_{j'}
		+
		F_i^\top  F_{j'} \cdot F_{i'}^\top F_j
		+
		F_i^\top  F_{i'} \cdot F_i^\top F_{j'}.
		\end{align*}
		This leads to a completely explicit expression for the average information. Recall $F_2 = F\otimes F$, and $M$ is the $d^2\times d^2$ tensor with entries given above. Then
		\begin{align*}
		\bar I_W  
		& = (W \otimes W) \cdot F_2^*\cdot (F_2^2 \odot M) \cdot F_2^{*}.
		\end{align*}

\subsection{Proof of Corollary \ref{cor:2lnn_for_classification}}\label{subappend:proof_of_2lnn_for_classification}
For notational simplicity, for a generic matrix $A$, we will use $A^{\otimes 2}$ to denote the tensor product $A \otimes A$.
Recalling the definitions in \eqref{eq:asymp_cov_of_nonlinear_classification}, we have 
\begin{align*}
	\E U_W(X) & = \E [\eta'(f(W,X))^2  (\sigma'(WX) \otimes \sigma'(WX))\cdot (X \otimes X)^\top ],
\end{align*}
and similarly,
\begin{align*}
	\E [U_W(X)] \Sigma_{\textnormal{ERM}} \E [U_W(X)] & = \E [v_W(X) U_W(X)] \\
	& =  \E[v_W(X) \eta'(f(W,X))^2  (\sigma'(WX) \otimes \sigma'(WX))\cdot (X \otimes X)^\top].
\end{align*}	
On the other hand, under the circular symmetry model \eqref{eq:circular_shift}, for the augmented estimator, by Theorem \ref{thm:asymp_normality_augmented_estimator}, we can directly compute
	\begin{align*}
	\E [U_W(X)] \Sigma_{\textnormal{aERM}} \E [U_W(X)] & = \E \bigg[ \bigg(\E_G[(Y - \eta(f(W,gX))) \eta' (f(W,gX)) \nabla f(W,gX)] \bigg)^{\otimes 2}  \bigg] \\
	& = \E \bigg[ (Y - \eta(f(W,X)))^2 \eta'(f(W,X))^2  (\E_G \nabla f(W,gX))^{\otimes 2} \bigg] \\
	& = \E \bigg[ v_W(X) \eta'(f(W,X))^2 (\sigma'(WC_X))^{\otimes 2} (C_X^{\otimes 2})^\top  \bigg],
	\end{align*}
	where in the second line we used $f(W,x) = f(W,gx)$, and the last equality is by a similar computation as in the proof of Theorem \ref{thm:2lnn}.
	Then it is clear that
	\begin{align*}
	&\E[U_W(X)](\SigmaERM - \SigmaAERM) \E[U_W(X)] \\
	& = \E \bigg[ v_W(X)\eta'(f(W,X))^2 \times   \bigg( ( \sigma'(WX))^{\otimes 2} (X^{\otimes 2})^\top - (\sigma'(WC_X))^{\otimes 2} (C_X^{\otimes 2})^\top  \bigg) \bigg].
	\end{align*}
	Assuming $\sigma(x) = x^2/2$, we have 
	\begin{align*}
	& \E[U_W(X)](\SigmaERM - \SigmaAERM) \E[U_W(X)] \\
	& = W^{\otimes 2} \E \bigg[ v_W(X)\eta'(f(W,X))^2  \bigg( (XX^\top)^{\otimes 2} - (C_X C_X^\top)^{\otimes 2} \bigg) \bigg],
	\end{align*}
	which concludes the proof.

\subsection{Proof of Proposition \ref{prop:risk_comparison_in_linear_regression}}\label{subappend:proof_of_risk_comparison_in_linear_regression}
    For part 1, we have
    \begin{align*}
    \hat \beta_{\textnormal{aDIST}}  
    & = \mbE_g \bigg[((Xg^\top)^\top Xg^\top)^{-1} (Xg^\top)^\top y\bigg] \\
    & = \mbE_g \bigg[ (gX^\top Xg^\top)^{-1} gX^\top (Xg^\top \beta + \ep) \bigg] \\
    & = \beta + \mbE_g \bigg[(gX^\top Xg^\top)^{-1} gX^\top \ep\bigg] \\
    & = \beta + \mbE_g \bigg[ g^{-1} (X^\top X)^{-1} g^{-1} g X^\top \ep\bigg] \\
    & = \beta + \mbE_g [g^{-1}](X^\top X)^{-1} X^\top \ep \\
    & = \beta + \mathcal{G}^\top (X^\top X)^{-1} X^\top \ep.
    \end{align*}
    Let $X = UDV^\top$ be a SVD of $X$, where $V\in\mbR^{p\times p}$ is unitary. Note that $\hat \beta_{\textnormal{aDIST}}$ is unbiased, so its $\ell_2$ risk is  
    \begin{align*}
    r_{\textnormal{aDIST}} & = \gamma^2\tr(\text{Var}(\hat \beta_{\textnormal{aDIST}})) 
    = \gamma^2 \tr(\mcG^\top (X^\top X)^{-1} \mcG) \\
    & = \gamma^2 \tr(\mcG ^\top V D^{-2} V^\top \mcG) 
    = \gamma^2 \tr(D^{-2} V^\top \mcG \mcG^\top V ) \\
    & =\gamma^2 \sum_{j=1}^p d_j^{-2} e_j^\top V^\top \mcG \mcG^\top V e_j 
    = \gamma^2 \sum_{j=1} ^p d_j^{-2}  \| \mcG^\top v_j \|_2^2,
    \end{align*}
    where $v_j \in \mbR^{p}$  is $j$-th eigenvector of $X^\top X$ and $d_j^2$ is $j$-th eigenvalue of $X^\top X$.
    As a comparison, for the usual ERM, we have 
    $$
    \hbetaERM = (X^\top X)^{-1} X^\top y = \beta + (X^\top X)^{-1} X^\top \ep, 
    $$
    so its $\ell_2$ risk is
    \begin{align*}
    \rERM & = \gamma^2 \tr((X^\top X)^{-1}) 
    = \gamma^2 \sum_{j=1}^p  d_j^{-2}.
    \end{align*}
    So part 1 is proved.    

    We now prove part 2. For $r_{\textnormal{aDIST}}  \le \rERM$ we need to show
    \beqs
    \tr( (X^\top X)^{-1} \mcG \mcG^\top)  \le \tr( (X^\top X)^{-1}). 
    \eeqs
    A sufficient condition is that, in the partial ordering of positive semidefinite matrices,
    \beqs
    \mcG \mcG^\top  \le I_p. 
    \eeqs
    This is equivalent to the claim that for all $v$
    $\|\E_g g^\top v\|^2 \le \|v\|^2.$
    However, by Jensen's inequality,
    $\|\E_g g^\top v\|^2 \le \E_G \|g^\top v\|^2.$
    Since $G$ is a subgroup of the orthogonal group, we have $\|g^\top v\|^2  = \|v\|^2$, which finishes the proof for part 2.

    We finally prove part 3. We assume $G$ is the permutation group. This group is clearly a subgroup of the orthogonal group. Note that invariance w.r.t. $G$ implies that the true parameter is a multiple of the all ones vector:   
    $
    \beta = 1_p b.
    $
    So we have
    $$
    \hbetaCERM = 1_p \hat b, \qquad \hat b = \arg\min \|y - X 1_p b\|_2^2.
    $$
    Solving the least-squares equation gives
    $$
    \hat b = \frac{1^\top X^\top y}{1_p^\top X^\top X 1_p}.
    $$
    The risk of estimating $b$ is then
    $
    \gamma^2 (1_p^\top X^\top X 1_p)^{-1},
    $
    so that the risk of estimating $\beta$ by $1_p \hat b$ is
    $$
    \rCERM = \gamma^2 p (1_p^\top X^\top X 1_p)^{-1}.
    $$
    Finally, we have
    $$
    r_{\textnormal{aDIST}} = \frac{\gamma^2}{p^2} \tr(1_p1_p^\top (X^\top X)^{-1} 1_p 1_p^\top) = \frac{\gamma^2}{p} 1_p^\top (X^\top X)^{-1} 1_p,
    $$
    which is equal to $\rCERM$ if $X^\top X = I_p$.

\section{Proofs for results under approximate invariance}\label{append:proofs_under_approx_invariance}

\subsection{Proof of Lemma \ref{lemma:approx_inv_lem}}\label{subappend:proof_of_approx_inv_lemma}
For part 1, we have
\begin{align*}
    \| \E_X \E_g f(gX) - \E_X f(X) \|_2 & =  \sup_{\|v\|_2 \leq 1 } \E_g \E_X \la v, f(gX) - \E_X f(X) \ra \\
    & \leq \sup_{\|v\|_2 \leq 1} \E_g\|v\|_2  \ww_1(f(gX), f(X)) \\
    & =\E_g \ww_1(f(gX), f(X)),
\end{align*}
where the inequality is due to Kantorovich-Rubinstein theorem, i.e., the dual representation of the $W_1$ metric (see. e.g., \citealt{villani2003topics}).

For part 2, by law of total variance, we have
$$
    \C_X\bar f(X) - \C_X f(X) = -\E_X \C_g f(gX) + \Delta_1 + \Delta_2,
$$
where
\begin{align*}
    \Delta_1 & = \E_{(X, g)} f(gX) f(gX)^\top - \E_X f(X) f(X)^\top\\
    \Delta_2 & = \E_X f(X) \E_X f(X)^\top - \E_{(X, g)} f(gX) \E_{(X, g)} f(gX)^\top
\end{align*}
For any non-zero vector $v$, we have
\begin{align*}
    |v^\top \Delta_1 v| & = \bigg|\E_g \E_X \bigg[ \la v, f(gX) \ra^2 - \la v, f(X) \ra^2 \bigg]\bigg| \\
& \leq \E_g \bigg| \E_X \bigg[ \la v, f(gX) \ra^2 - \la v, f(X) \ra^2 \bigg] \bigg|.
\end{align*}
The Lipschitz constant for the function $w \mapsto \la v, w\ra^2$, where $w \in \textnormal{Range}(f)$, is bounded above by $2\|v\|_2^2 \|f\|_\infty$. Invoking Kantorovich-Rubinstein theorem again, we have
$$
    |v^\top \Delta_1 v| \leq 2 \|v\|_2^2 \|f\|_\infty \E_g \ww_1(f(gX), f(X)). 
$$
Similarly, we have
\begin{align*}    
    |v^\top \Delta_2 v | & = \bigg|  \bigg(\E_X \la v, f(X)\ra\bigg)^2 - \bigg(\E_{(X, g)} \la v, f(gX)\ra\bigg)^2  \bigg| \\
    & \leq 2\|f\|_\infty \bigg| \E_g\E_X \bigg[\la v, f(X)\ra -  \la v, f(gX) \ra\bigg] \bigg| \\
    & \leq  2\|f\|_\infty \E_g \bigg| \E_X \bigg[\la v, f(X)\ra -  \la v, f(gX) \ra\bigg] \bigg| \\
    & \leq 2 \|v\|_2^2 \|f\|_\infty \E_g \ww_1(f(gX), f(X)).
\end{align*}
Part 2 is then proved by recalling the definition of the Loewner order.    

For part 3, we have
\begin{align*}
     \E_X \varphi(\bar f(X)) - \E_X \varphi(f(X)) & =  \E_X \varphi(\bar f(X))  - \E_X \bar{\varphi \circ f} (X)  + \E_g \E_X \varphi \circ f (gX) - \E_X \varphi(f(X)).
\end{align*}
We finish the proof by noting that 
\begin{align*}
     \bigg|\E_g \E_X \varphi \circ f (gX) - \E_X \varphi(f(X))\bigg| 
    & \leq  \|\varphi \|_\textnormal{Lip} \ww_1(f(gX), f(X)). 
\end{align*}

\subsection{Proof of Proposition \ref{prop:mse_of_general_estimators_under_approx_inv}}	\label{subappend:proof_of_mse_of_general_estimators_under_approx_inv}
For notational simplicity, we let $\bar f = \htheta_G$ and $f = \htheta$. Then using bias-variance decomposition, we have
$$
    \textnormal{MSE}(\bar f) - \textnormal{MSE}(f) = B + V, 
$$
where
\begin{align*}
    B & = \|\textnormal{Bias}(\bar f) \|_2^2 - \|\textnormal{Bias}(f) \|_2^2 \\
    V & =  {\tr(\C_X \bar f(X)) - \tr(\C_X f(X))} .
\end{align*}    

We first analyze the bias term. Note that by Lemma \ref{lemma:approx_inv_lem}, we have
\begin{align*}
    \bigg|\|\textnormal{Bias}(\bar f) \|_2 - \|\textnormal{Bias}(f) \|_2 \bigg| & \le \| \E_X \bar f(X) - \E_X f(X) \|_2 \\
    & \leq \E_g \ww_1(f(gX) , f(X)).
\end{align*}    
Hence
\begin{align*}
     |B| 
    & \leq \bigg(\|\textnormal{Bias}(\bar f) \|_2 + \|\textnormal{Bias}(f) \|_2\bigg) \bigg|\|\textnormal{Bias}(\bar f) \|_2 - \|\textnormal{Bias}(f) \|_2 \bigg| \\
    & \leq \bigg(\|\E_X \bar f(X) - \E_X f(X) \|_2 + 2 \|\textnormal{Bias}(f)\|_2 \bigg) \cdot  \|\E_X \bar f(X) - \E_X f(X) \|_2 \\
    & \leq \bigg( \E_G \ww_1(f(gX) , f(X)) + 2 \|\textnormal{Bias}(f)\|_2 \bigg) \cdot \E_g \ww_1(f(gX) , f(X)) . 
    \end{align*}

The variance term $V$ can be bounded by the following arguments. We have
\begin{align*}
    |\tr(\Delta_1)| & = \bigg|\E_{g} \E_X \bigg[ \|f(gX) \|_2^2 - \|f(X)\|_2^2 \bigg]\bigg| \\
    & \leq 2 \|f\|_\infty  \E_g \bigg|\E_X\bigg[ \| f(gX) \|_2 - \|f(X)\|_2 \bigg]\bigg| \\
    & \leq 2 \|f\|_\infty \E_g \ww_1(f(gX), f(X)).
    \end{align*}
    Similarly, we have
    \begin{align*}
    |\tr(\Delta_2)| & = \bigg|\| \E_X f(X) \|_2^2 - \| \E_{X, g} f(gX) \|_2^2 \bigg|\\
    & \leq 2 \|f \|_\infty \bigg| \|\E_X f(X)\|_2 - \|\E_{X, g} f(gX)\|_2 \bigg| \\
    & \leq  2 \|f\|_\infty \| \E_X f(X) - \E_X \bar f(X)  \|_2 \\
    & \leq 2\|f\|_\infty \E_g \ww_1(f(gX), f(X)) ,
\end{align*}
where the last inequality is due to Lemma \ref{lemma:approx_inv_lem}.

Combining the bound for $B$ and $V$ gives the desired result.

\subsection{Proof of Theorem \ref{thm:rademacher_bound_under_approx_inv}}\label{subappend:proof_of_rademacher_bound_under_approx_inv}
We first prove a useful lemma.
\begin{lemma}[Triangle inequality/Tensorization]
    For two  random vectors $(X_1,..., X_n),$ $(Y_1,$ $..., Y_n) \in \xx^n$, we denote the joint laws as $\mu^n, \nu^n$ respectively, and the marginal laws as $\{\mu_i\}_1^n, \{\nu_i\}_1^n$ respectively. We have
    $$
        \ww_{d^n}(\mu^n, \nu^n) \leq \sum_i \ww_d(\mu_i, \nu_i).
    $$
\end{lemma}
\begin{proof}
    By Kantorovich duality (see, e.g., \citealt{villani2003topics}), for each coordinate, we can choose optimal couplings $(X^*_i, Y^*_i) \in \Pi(\mu_i, \nu_i)$ s.t. $\ww_d(X_i, Y_i) =  \E d(X_i^*, Y^*_i)$. We then conclude that proof by noting that $(\{X_i^*\}_1^n, \{Y_i^*\}_1^n) \in \Pi(\mu^n, \nu^n)$.
\end{proof}

Now we are ready to give the proof. We will do the proof for a general metric $d$. The desired result is a special case when $d$ is the Euclidean metric. 

The results concerning $\htheta_n$ is classical. We present a proof here for completeness. We recall the classical approach of decomposing the generalization error into terms that can be bounded via concentration and Rademacher complexity \citep{bartlett2002rademacher,shalev2014understanding}:
\begin{align*}
         \E L(\htheta_n, X) - \E L(\thetanull, X) 
        & =  \E L(\htheta_n, X) - \frac{1}{n}\sum_{i=1}^n L(\htheta_n, X_i)  + \frac{1}{n}\sum_{i=1}^n L(\htheta_n, X_i) - \E L(\thetanull, X).
    \end{align*}
Hence we arrive at
\begin{align*}
     \E L(\htheta_n, X) - \E L(\thetanull, X) 
    & \leq \E L(\htheta_n, X) - \frac{1}{n}\sum_{i=1}^n L(\htheta_n, X_i)  + \frac{1}{n}\sum_{i=1}^n L(\thetanull, X_i) - \E L(\thetanull, X) \\ 
    & \leq  \sup_{\theta\in \Theta} \bigg|\frac{1}{n} \sum_{i=1}^n L(\theta, X_i) - \E L(\theta, X)\bigg|  + \bigg(\frac{1}{n}\sum_{i=1}^n L(\thetanull, X_i) - \E L(\thetanull, X)\bigg),
\end{align*}
where the first inequality is because $\htheta_n$ is a minimizer of the empirical risk.
By McDiarmid's inequality, we have
$$
    \P(\frac{1}{n}\sum_{i=1}^n L(\thetanull, X_i) - \E L(\thetanull, X) > t) \leq \exp\{-2nt^2\}.
$$
So w.p. at least $1-\delta/2$, we have
$$
    \frac{1}{n}\sum_{i=1}^n L(\thetanull, X_i) - \E L(\thetanull, X)  \leq \sqrt{\frac{\log 2/\delta}{2n}}.
$$
It remains to control
$$
    \sup_{\theta\in \Theta} \bigg|\frac{1}{n} \sum_{i=1}^n L(\theta, X_i) - \E L(\theta, X)\bigg|.
$$
We bound the above quantity using Rademacher complexity. The arguments are standard and can be found in many textbooks (see, e.g., \citealt{shalev2014understanding}). Since we've assumed $L(\theta, x)\in[0, 1]$, for two data sets $\{X_i\}_1^n$ and $\{\tilde X_i\}_{1}^n$ which only differ in the $i$-th coordinate, we have
\begin{align*}
     &\sup_{\theta\in \Theta} \bigg| \frac{1}{n}\sum_{i= 1}^n L(\theta, X_i) - \E L(\theta, X)\bigg|
     - \sup_{\theta\in \Theta} \bigg| \frac{1}{n}\sum_{i= 1}^n L(\theta, \tilde X_i) - \E L(\theta, X)\bigg| \\
    & \leq \frac{1}{n}\sup_{\theta \in \Theta} |L(\theta, X_i) - L(\theta, \tilde X_i)| \leq \frac{1}{n}. 
\end{align*}    
By McDiarmid's inequality, we have 
\begin{align*}
    & \P \bigg(\sup_{\theta\in \Theta} \bigg|\frac{1}{n} \sum_{i=1}^n L(\theta, X_i)  - \E L(\theta, X)\bigg|- \E \bigg[ \sup_{\theta\in \Theta} \bigg|\frac{1}{n} \sum_{i=1}^n L(\theta, X_i) - \E L(\theta, X)\bigg| \bigg] \geq t \bigg) \\
    &  \leq \exp\{-2nt^2\}.
\end{align*}    
It follows that w.p. $1-\delta/2$, we have
\begin{align*}
     \sup_{\theta\in \Theta} \bigg|\frac{1}{n} \sum_{i=1}^n L(\theta, X_i) - \E L(\theta, X)\bigg| - \E \bigg[ \sup_{\theta\in \Theta} \bigg|\frac{1}{n} \sum_{i=1}^n L(\theta, X_i) - \E L(\theta, X)\bigg| \bigg]  \leq \sqrt{\frac{\log2/\delta}{2n}}.
\end{align*}    
A standard symmetrization argument then shows that
$$
    \E \bigg[ \sup_{\theta\in \Theta} \bigg|\frac{1}{n} \sum_{i=1}^n L(\theta, X_i) - \E L(\theta, X)\bigg| \bigg] \leq  2 \rad_n(L \circ \Theta),
$$
where the Rademacher complexity of the function class $L\circ \Theta = \{x \mapsto L(\theta, x): \theta \in \Theta\}$ is defined as
$$
    \rad_n(L\circ \Theta) = \E \sup_{\theta \in \Theta} \bigg| \frac{1}{n} \sum_{i= 1}^n \ep_i L(\theta, X_i)  \bigg|,
$$
where the expectation is taken over both the data and IID Rademacher random variables $\ep_i$, which are independent of the data.
Summarizing the above computations (along with a union bound) finishes the proof for $\htheta_n$.

Now we consider the results for $\htheta_{n, G}$ under approximate invariance.  We start by doing a similar decomposition
\begin{align*}
    \E L(\htheta_G, X) - \E L(\thetanull, X) = \RN{1} + \RN{2} + \RN{3} + \RN{4} + \RN{5},
    \end{align*}
where
\begin{align*}
    \RN{1} & = \E L(\htheta_G, X) - \E \E_G L(\htheta_G, gX) \\
    \RN{2} & = \E \E_G L(\htheta_G, gX) - \frac{1}{n} \sum_{i=1}^n \E_G L(\htheta_G, gX_i) \\
    \RN{3} & = \frac{1}{n} \sum_{i=1}^n \E_G L(\htheta_G, gX_i) - \frac{1}{n} \sum_{i=1}^n \E_G L(\thetanull, gX_i) \\
    \RN{4} & =  \frac{1}{n} \sum_{i=1}^n \E_G L(\thetanull, gX_i) - \E \E_G L(\thetanull, gX) \\
    \RN{5} & =  \E \E_G L(\thetanull, gX)  - \E L(\thetanull, X).
\end{align*}
By construction, we have $\RN{3} \leq 0$ and 
$$
    \RN{2} \leq  \sup_{\theta \in \Theta} \bigg| \frac{1}{n} \sum_{i=1}^n \E_G L(\theta, gX_i) - \E \E_G L(\theta, gX)   \bigg|.
$$
Moreover, we have
$$
    \RN{1} + \RN{5} \leq 2 \sup_{\theta\in \Theta} \bigg| \E L(\theta, X) - \E \E_G L(\theta, gX) \bigg|,
$$
which is equal to zero under exact invariance $gX =_d X$. 

The term $\RN{2} + \RN{4}$ is taken care of by essentially the same arguments as the proof for $\htheta_n$. One uses concentration to bound $\RN{4}$ and uses Rademacher complexity to bound $\RN{2}$. These arguments give
$$
    \RN{2} + \RN{4} \leq 2\rad_n(\bar L \circ \Theta) +  \sqrt{\frac{2 \log2/\delta}{n}}
$$
w.p. at least $1-\delta$, where
\begin{align*}
    \rad_n(\bar L \circ \Theta) & = \E \sup_{\theta\in \Theta} \bigg| \frac{1}{n} \sum_{i=1}^n \ep_i \E_G L(\theta, g X_i) \bigg|.
\end{align*}
Now we have
\begin{align*}
     \rad_n(\bar L\circ \Theta)     - \rad_n(L\circ \Theta)  \leq \Delta + \E_g \bigg[ \E \sup_{\theta\in \Theta} \bigg| \frac{1}{n} \sum_{i=1}^n \ep_i L(\theta, g X_i) \bigg| -\E \sup_{\theta\in \Theta} \bigg| \frac{1}{n} \sum_{i=1}^n \ep_i L(\theta,  X_i) \bigg| \bigg],
\end{align*}
where we recall
\begin{align*}
    \Delta &= \E \sup_{\theta } | \frac{1}{n}\sum_{i=1}^n \ep_i \E_g L(\theta, gX_i) | - \E \E_g \sup_{\theta\in\Theta} | \frac{1}{n} \ep_i L(\theta, gX_i) | \leq 0
\end{align*}
by Jensen's inequality.

By our assumption, for any $x, \tilde x \in \xx, \theta \in \Theta$, we have
$$
    L(\theta, x) - L(\theta, \tilde x) \leq \|L\|_{\text{Lip}} \cdot d(x, \tilde x)
$$
for some constant $\|L\|_\text{Lip}$.
For a fixed vector $(\ep_1, ..., \ep_n)$, consider the function 
$$
    h: (x_1, ..., x_n)\mapsto \sup_{\theta\in \Theta} \bigg| \frac{1}{n} \sum_{i=1}^n \ep_i L(\theta,  x_i)\bigg|.
$$
We have
\begin{align*}
     |h(x_1, ..., x_n) - h(y_1, ..., y_n)| 
    & \leq \frac{1}{n}\sup_{\theta\in \Theta} \bigg| \sum_{i= 1}^n \ep_i L(\theta, x_i) - \ep_i L(\theta, y_i)    \bigg| \\
    & \leq \frac{1}{n} \|L\|_\text{Lip}\cdot \sum_{i} d(x_i, y_i).
\end{align*}
That is, the function $h : \xx^n \to \R$ is $(\|L\|_\text{Lip}/n)$-Lipschitz w.r.t. the l.s.c. metric $d_n$, defined by $d_n(\{x_i\}_1^n, \{y_i\}_1^n) = \sum_i d(x_i, y_i)$. 
Applying the tensorization lemma and Kantorovich-Rubinstein theorem, for arbitrary random vectors $(X_1, ..., X_n)$ and $(Y_1, ..., Y_n)$, we have 
\begin{align*}
    |\E h(X_1, ..., X_n) - h(Y_1, ..., Y_n)|  
    & \leq  \frac{1}{n} \|L\|_\text{Lip}\cdot \ww_{d^n}(\mu^n, \nu^n) \\
    & \leq \frac{1}{n}\|L\|_\text{Lip}\cdot \sum_{i=1}^n \ww_d(X_i, Y_i).
\end{align*}    
Hence we arrive at
\begin{align*}
    \rad_n(\bar L \circ \Theta) - \rad_n(L \circ \Theta) 
    & \leq \Delta +  \|L\|_\text{Lip}\cdot \frac{1}{n} \sum_i \E_g \ww_d(X_i, g X_i) \\
    & = \|L\|_\text{Lip}\cdot \E_g \ww_d(X, gX). 
\end{align*}    
Summarizing the above computations, we have
\begin{align*}
    \RN{2} + \RN{4} & \leq 2 \rad_n(L\circ \Theta) + 2 \|L\|_\text{Lip}\cdot \E_G \ww_d(X, gX)+  \sqrt{\frac{2\log 2/\delta}{n}}
\end{align*}    
    w.p. at least $1-\delta$.

We now bound $\RN{1} + \RN{5}$. We have
\begin{align*}
    \RN{1} + \RN{5}& \leq 2 \sup_{\theta\in \Theta} \bigg| \E L(\theta, X) - \E \E_G L(\theta, gX) \bigg| \\
     & \leq 2\sup_{\theta\in \Theta} \E_G \bigg| \E L(\theta, X)  - \E L(\theta, gX) \bigg| \\
    & \leq 2\|L\|_\text{Lip}\cdot \ww_d(X, gX).
\end{align*}

Combining the bounds for the five terms gives the desired result.

\subsection{Proof of Theorem \ref{thm:asymp_normality_augmented_estimator_under_approx_inv}} \label{subappend:proof_of_asymp_normality_under_approx_inv}
Recall that 
$$
    \theta_G = \underset{\theta\in\Theta}{\arg \min} \E \E_g L(\theta, gX).
$$
By our assumptions, we can apply Theorem 5.23 of \cite{van1998asymptotic} to obtain the Bahadur representation: 
$$
    \sqrt{n}({\htheta_G -\theta_G}) = \frac{1}{\sqrt{n}} V_{G}^{-1} \sum_{i=1}^n \nabla \E_G L(\thetanull, gX_i) + o_p(1),
$$
so that we get
\begin{align*}
     \sqrt{n}(\htheta_G - \theta_G)& \Rightarrow \N\bigg( 0, V_{G}^{-1}  \E \bigg[ \nabla \E_G L(\theta_G, gX)  (\E_G \nabla L(\theta_G, gX))^\top \bigg]  V_{G}^{-1}\bigg).
\end{align*}    

To simplify notations, we let 
$$
    C_0 = \C_X(\nabla L(\thetanull, X)), \ \ \  C_G = \C_X (\nabla \E_g L(\theta_G, gX)).
$$ 
By bias-variance decomposition, we have
\begin{align*}
    \textnormal{MSE}_0 & = n^{-1}\tr(V_0^{-1} C_0 V_0^{-1}) \\
    \textnormal{MSE}_G & = n^{-1} \tr(V_G^{-1} C_G V_G^{-1}) + \|\theta_G - \thetanull\|^2. 
\end{align*}        
We have
\begin{align*}
     \tr(V_G^{-1} C_G V_G^{-1}) - \tr(V_0^{-1} C_0 V_0^{-1}) 
    &= \la C_G, V_G^{-2} \ra - \la C_0, V_0^{-2} \ra \\
    & = \la C_G - C_0 + C_0 , V_G^{-2}\ra - \la C_0 , V_0^{-2} \ra \\
    & = \la C_G - C_0 , V_G^{-2} \ra + \la C_0, V_G^{-2} - V_0^{-2} \ra.
\end{align*}
We let $M_0(X) = \nabla L(\thetanull, X)\nabla L(\thetanull, X)^\top$ and $M_G(X) = \nabla L(\theta_G, X)\nabla L(\theta_G, X)^\top$. Then we have
\begin{align*}
      C_G - C_0 
     & = C_G - \E_X M_G(X) + \E_X M_G(X) - C_0 \\
     & = \bigg(C_G - \E_X \E_g M_G(gX)\bigg)  + \E_g \E_X \bigg[M_G(gX) -  M_G(X)\bigg]   +  \E_X \bigg[M_G(X) - M_0(X)\bigg] \\
     & = -\E_X \C_g(\nabla L(\theta_G, gX))  + \E_g \E_X \bigg[M_G(gX) -  M_G(X)\bigg]  + \E_X \bigg[M_G(X) - M_0(X)\bigg].
\end{align*}
Hence we arrive at
\begin{align*}
     n(\textnormal{MSE}_G - \textnormal{MSE}_0) 
    & =  - \bigg\la \E_X \C_G(\nabla L(\theta_G, gX)), V_G^{-2} \bigg\ra + \RN{1} + \RN{2} + \RN{3} + \RN{4},
\end{align*}
where
\begin{align*}
    \RN{1} &= n\| \theta_G - \thetanull \|^2 \\
    \RN{2} & = \E_g \E_X \bigg\la M_G(gX) - M_G(X), V_G^{-2} \bigg\ra \\
    \RN{3}& = \E_X \bigg\la M_G(X) - M_0(X), V_G^{-2}  \bigg\ra \\
    \RN{4} & = \la C_0, V_G^{-2} - V_0^{-2} \ra,
\end{align*}
and this is the desired result.

\subsection{Proof of Theorem \ref{thm:overparam_two_layer_net_gd}}\label{subappend:proof_of_ntk}
We seek to control the misclassification error of the two-layer net at step $k$. By Markov's inequality, for a new sample $(X, Y)$ from the data distribution, we have
\begin{align*}
    \P (Y f(x; W_k, a)\leq 0) & = \PP\bigg(\frac{1}{1+ e^{Y f(X; W_k, a)}} \geq \frac{1}{2}\bigg) \\
    & \leq 2 \EE \bigg[- \ell'(Y f(X; W_k , a))\bigg].
\end{align*}    
The population quantity in the RHS is decomposed by
$$
    \EE \bigg[- \ell'(Y f(X; W_k , a))\bigg] = \RN{1} + \RN{2}, 
$$
where
$$
    \RN{1} = \frac{1}{n} \sum_{i\in[n]} \EE_g \bigg[-\ell'(Y_i f_{i, g}(W_k))\bigg]
$$
and
\begin{align*}
    \RN{2} & = \EE \bigg[- \ell'(Y f(X; W_k , a))\bigg] -\frac{1}{n} \sum_{i\in[n]} \EE_g \bigg[-\ell'(Y_i f_{i, g}(W_k))\bigg]
\end{align*}    

The first term (optimization error) is controlled by calculations based on the Neural Tangent Kernel. The second term (generalization error) is controlled via Rademacher complexity. 


We first control the first term (optimization error). In fact, everything is set up so that we can directly invoke Theorem 2.2 of \cite{ji2019polylogarithmic}. We note that their result holds for any fixed dataset and there is no independence assumption. This gives the following result:
\begin{proposition}
\label{prop: aug_opt_err}
Given $\ep \in (0, 1), \delta\in(0, 1/3)$. Let
$$
    \lambda = \frac{\sqrt{2\log(4n|G|/\delta)} + \log (4/\ep)}{\gamma/4} , \qquad M = \frac{4096 \lambda^2}{\gamma^6}. 
$$
For any $m\geq M$ and any constant step size $\eta\leq 1$, w.p. $1-3\delta$ over the random initialization, we have
$$
    \frac{1}{T} \sum_{t< T} \bar R_n(W_t) \leq \ep, \qquad T = \lceil 2\lambda^2/(n\ep) \rceil.
$$
Moreover, for any $0\leq t < T$ and any $1\leq s \leq m$, we have
$$
    \| w_{s, t} -w_{s, 0} \|_2 \leq \frac{4\lambda}{\gamma \sqrt{m}},
$$    
where $w_{s, t}$ is the $s$-th row of the weight matrix at step $t$.
\end{proposition}
\begin{proof}
    This is a direct corollary of Theorem 2.2 in \cite{ji2019polylogarithmic}.
\end{proof}
Assume the above event happens. Since we've chosen $k$ to be the best iteration (with the lowest empirical loss) in the first $T$ steps. Then with the same probability as above, we have $\bar R_n(W_k)\leq \ep$. Now, let us note that the logistic loss satisfies the following fundamental self-consistency bound: $-\ell' \leq \ell$. This shows that if the loss is small, then the magnitude of the derivative is also small. Thus on the same event, we have that the term $\RN{1}$ is also bounded, 
$$
    \RN{1} \leq \bar R_n(W_k) \leq \ep.
$$

We then control the second term (generalization error). The calculations below are similar to the proof of Theorem 4.4. We begin by decomposing
$$
    \RN{2} = \RN{2.1} + \RN{2.2},
$$
where
\begin{align*}
\RN{2.1} & = \EE \bigg[- \ell'(Y f(X; W_k , a))\bigg]- \EE \EE_g \bigg[ - \ell'(Y f(gX; W_k , a)) \bigg]
\end{align*}
and
\begin{align*}
    \RN{2.2} & =   \EE \EE_g \bigg[ - \ell'(Y f(gX; W_k , a)) \bigg]   - \frac{1}{n} \sum_{i\in[n]} \EE_g \bigg[-\ell'(Y_i f(gX_i; W_k, a))\bigg].
\end{align*}

We control term $\RN{2.1}$ by exploiting the closedness between the distribution of $(X, Y)$ and that of $(gX, Y)$. Note that the Lipschitz constant of the map $x\mapsto -\ell'(y f(x; W_k, a ))$ (w.r.t. the Euclidean metric on $\RR^d$) can be computed by:
\begin{align*}
    |\ell'(yf(x; W_k, a)) - \ell'(yf(\tilde x; W_k, a)) | 
    & \leq \frac{1}{4} | f(x; W_k, a) - f(\tilde x; W_k, a) | \\
    & = \frac{1}{4} \bigg|\frac{1}{\sqrt{m}} \sum_{s\in[m]} \sigma(w_{s, k}^\top x)  - \frac{1}{\sqrt{m}} \sum_{s\in[m]} \sigma(w_{s, k}^\top \tilde x)\bigg| \\
    & \leq \frac{1}{4} \frac{1}{\sqrt{m}} \sum_{s\in[m]} |w_{s, k}^\top (x-\tilde x)| \\
    & \leq \frac{1}{4} \frac{1}{\sqrt{m}} \sum_{s\in[m]} \| w_{s, k} \|_2 \| x-\tilde x\|_2 \\
    & \leq \frac{1}{4} \frac{1}{\sqrt{m}} \sum_{s\in[m]} (\| w_{s, 0} \|_2 + \|w_{s, 0}-w_{s, k} \|_2) \| x-\tilde x\|_2  \\
    & \leq \frac{1}{4} \bigg(\rho \sqrt{m} +  \frac{1}{\sqrt{m}}\sum_{s\in[m]} \|w_{s, 0} \|_2\bigg) \|x-\tilde x \|_2,
\end{align*}
where $\rho = \frac{4\lambda}{\gamma \sqrt{m}}$ and the last inequality is by Proposition \ref{prop: aug_opt_err}. Note that each $\| w_{s, 0}\|_2$ is $1$-subgaussian as a 1-Lipschitz function of a Gaussian random vector (for example, by Theorem 2.1.12 of \citealt{tao_2012}), so that
$$
    \PP\bigg(\frac{1}{\sqrt{m}} \sum_{s\in[m]}  \| w_{s, 0}\|_2 - \EE \bigg[\frac{1}{\sqrt{m}} \sum_{s\in[m]}  \| w_{s, 0}\|_2\bigg] \geq t\bigg)  \leq e^{-t^2/2}.
$$
Hence w.p. at least $1-\delta$, we have
\begin{align*}
    \frac{1}{\sqrt{m}} \sum_{s\in[m]}  \| w_{s, 0}\|_2 & \leq \EE \bigg[\frac{1}{\sqrt{m}} \sum_{s\in[m]}  \| w_{s, 0}\|_2\bigg] + \sqrt{2 \log \frac{1}{\delta}}      \\
    & \leq \frac{1}{\sqrt{m}} \sum_{s\in[m]} \sqrt{\EE \| w_{s, 0}\|_2^2} + \sqrt{2\log \frac{1}{\delta}} \\
    & = \sqrt{md} + \sqrt{2\log 1/\delta}.
\end{align*}    
So w.p. at least $1-\delta$, the Lipschitz constant of the map $x\mapsto -\ell'(yf(x; W_k, a))$ is bounded above by
$$
    \frac{1}{4} \bigg(\frac{4\lambda}{\gamma} + \sqrt{md} + \sqrt{2\log 1/\delta}\bigg).
$$
Assume the above event happens (along with the previous event, the overall event happens w.p. at least $1-4\delta$). This information allows us to exploit the closeness between $X|Y$ and $gX|Y$. We have
\begin{align*}
    \RN{2.1} & = \EE_{Y}\EE_{g\sim\mathbb{Q}} \bigg[ \EE_{X|Y}[ -\ell'(Y f(X; W_k, a))]  - \EE_{gX|Y} [-\ell'(Yf(gX; W_k, a))] \bigg] \\
    & \leq \EE_Y \EE_{g\sim\mathbb{Q}}\bigg[ \frac{1}{4} \bigg(\frac{4\lambda}{\gamma} + \sqrt{md} + \sqrt{2\log 1/\delta}\bigg) \cdot \mathcal{W}_1(X|Y, gX|Y)  \bigg] \\
    & = \frac{1}{4} \bigg(\frac{4\lambda}{\gamma} + \sqrt{md} + \sqrt{2\log 1/\delta}\bigg) \cdot \EE_Y \EE_{g\sim\mathbb{Q}} \mathcal{W}_1(X|Y, gX|Y),
\end{align*}
where we let $X|Y$ to denote the conditional distribution of $X$ given $Y$, and the inequality is by the dual representation of the Wasserstein distance. Note that under exact invariance, $\RN{2.1} = 0$.

The term $\RN{2.2}$ is controlled by standard results on Rademacher complexity. Indeed, by the same arguments as in the proof of Theorem 6.4 of the main manuscript, w.p. at least $1-\delta$, we have
$$
    \RN{2.2} \leq 2\bar{\mathcal{R}}_n + \sqrt{\frac{\log 2/\delta}{2n}}.
$$
Taking a union bound (now w.p. at least $1-5\delta$), we have proved the generalization error bound.

Finally, we prove the bound on $\bar{\mathcal{R}}_n - \mathcal{R}_n$. 
Under exact invariance, Jensen's inequality gives $\bar{\mathcal{R}}_n \leq \mathcal{R}_n$. However, under approximate invariance, we have an extra bias term. We have
\begin{align*}
    & \bar{\mathcal{R}}_n - \mathcal{R}_n  =  \Delta + \EE \EE_g \sup_{W\in \mathscr{W}_\rho} \bigg| \frac{1}{n} \ep_i  \bigg[ -\ell'(Y_i f_{i, g}(W))\bigg] \bigg| - \mathcal{R}_n.
\end{align*}
where
\begin{align*}
    \Delta & = \EE \sup_{W\in \mathscr{W}_\rho} \bigg| \frac{1}{n} \ep_i \EE_g \bigg[ -\ell'(Y_i f_{i, g}(W))\bigg] \bigg|-  \EE\EE_g \sup_{W\in \mathscr{W}_\rho} \bigg| \frac{1}{n} \ep_i  \bigg[ -\ell'(Y_i f_{i, g}(W))\bigg] \bigg| \leq 0
\end{align*}    
by Jensen's inequality.
Now by the computations when bounding term $\RN{2.1}$ and the arguments in the proof of Theorem 4.4, we have
\begin{align*}
     \EE\EE_g  \sup_{W\in \mathscr{W}_\rho} \bigg| \frac{1}{n} \ep_i  \bigg[ -\ell'(Y_i f_{i, g}(W))\bigg] \bigg| - \mathcal{R}_n 
    & \leq \frac{1}{4} \bigg(\frac{4\lambda}{\gamma} + \sqrt{md} + \sqrt{2\log 1/\delta}\bigg) \\
    & \qquad  \cdot \EE_Y \EE_{g\sim\mathbb{Q}} \mathcal{W}_1(X|Y, gX|Y)
\end{align*}    
w.p. at least $1-\delta$. Combining the above bounds finishes the proof.

\section{Invariant MLE}
	\label{invrep_thy}
	Another perspective to exploit invariance is that of invariant representations. The natural question is, how can we work with invariant representations, and what are the limits of information we can extract from them?
	
	Suppose therefore that in our model it is possible to choose a representation $T(x)$ such that $(T(x),0_m) \in G\cdot x$ for all $x$(where $0_m$ is the zero vector with $m$ entries). Thus, $T$ chooses a representative from each orbit. This is equivalent to $(T(x),0_m) = g_0(x)\cdot x$, for some specific $g_0(x)\in G$. Suppose $T(\cdot),g(\cdot)$ satisfy sufficient regularity conditions, such as smoothness. For example, when $G$ is the orthogonal rotation group $O(d)$, we can take $T(x):=\|x\|_2$, and $g$ any orthogonal rotation such that   $g_0(x)=(\|x\|_2, 0_{d-1})$.
	
	How can we estimate the parameters $\theta$ based on this representation? A natural approach is to construct the MLE based on the data $T(X_1), \ldots, T(X_n)$. We can also construct invariant ERM using the same principle, but we will focus on MLE first. Let therefore $Q_\theta$ be the induced distribution of $T(X)$, when $X\sim P_\theta$, and assume it has a density  $q_\theta$ with respect to Lebesgue measure on a potentially lower dimensional Euclidean subspace (say $d'$ dimensional, where $d$ is original dimension and $m = d-d'$). We can construct the invariant MLE (iMLE): 
	
	$$
	\hat \theta_{iMLE,n} = \arg \max_\theta \sum_{i\in[n]} 
	\log q_\theta(T(X_i)).
	$$
	
	How does this compare to the previous approaches? It turns out that in general this is not better than the un-augmented MLE. Suppose that the group $G$ is discrete. Then we have
	\begin{align*}
	q_T(t)= \sum_{g\in G} p_X(g\cdot (t,0)) = |G| \cdot p_X((t,0)). 
	\end{align*}
	Therefore, in this case the iMLE equals the MLE. Therefore, the invariant MLE does not actually gain anything over the usual MLE, and in particular augmented MLE is better.

	\section{Experiment details}
	\label{sec:exp}
	Our code is available at \url{https://github.com/dobriban/data_aug}.
	Our experiment to generate Figure \ref{fig:exp} is standard: We train ResNet18 \citep{resnet18} on CIFAR10 \citep{krizhevsky2009learningml} for 200 epochs, based on the code of \url{https://github.com/kuangliu/pytorch-cifar}. The CIFAR10 dataset is standard and can be downloaded from \url{https://www.cs.toronto.edu/~kriz/cifar.html}. We use the default settings from that code, including the SGD optimizer with a learning rate of 0.1, momentum 0.9, weight decay $5\cdot 10^{-4}$, and batch size of 128. We train three models: 
	(1) without data augmentation, (2) horizontally flipping the image with 0.5 probability, and (3) a composition of randomly cropping a $32 \times 32$ portion of the image and random horizontal flip; besides the data augmentation, all other hyperparameters and settings are kept the same. We repeat this experiment 15 times and plot the average test accuracy for each number of training epochs. The shaded regions represent 1 standard deviation around the average test accuracy. We train both on the full CIFAR10 training data, as well as a randomly chosen half of the training data. We do this to evaluate the behavior of data augmentation in the limited data regime, because there it may to lead to higher benefits. This experiment was done on a p3.2xlarge (GPU) instance on Amazon Web Services (AWS).

\vskip 0.2in
\bibliography{refs}

\end{document}